\documentclass{article}

\usepackage[main,final]{neurips_2026}

\usepackage[utf8]{inputenc} 
\usepackage[T1]{fontenc}    
\usepackage{hyperref}       
\usepackage{url}            
\usepackage{booktabs}       
\usepackage{amsfonts}       
\usepackage{nicefrac}       
\usepackage{microtype}      
\usepackage{xcolor}         
\usepackage{amsmath}
\usepackage{amsthm}
\usepackage{amssymb}
\usepackage{natbib}
\usepackage{graphicx}
\usepackage{url}
\usepackage{booktabs}
\usepackage{array}
\usepackage{tabularx}
\usepackage{tikz}

\usetikzlibrary{positioning, fit, arrows.meta}

\usepackage{xspace}
\usepackage{amsmath}
\usepackage{graphicx}
\usepackage{framed}
\usepackage{bbm}
\usepackage{algorithm}
\usepackage{natbib}
\usepackage{mathrsfs}
\usepackage{amssymb}
\usepackage{bm}
\usepackage{xcolor}
\usepackage{hyperref}
\hypersetup{
    colorlinks=true,
    linkcolor=blue,
    citecolor=blue,
    breaklinks=true
}

\usepackage{prettyref}
\usepackage{mathtools}
\usepackage{amsmath}
\newcommand{\cw}[1]{{\color{magenta}CW: #1}}

\usepackage{multirow}
\usepackage{nicefrac}
\usepackage{amssymb}
\usepackage{pifont}
\definecolor{Green}{rgb}{0.13, 0.65, 0.3}
\usepackage{colortbl}
\newcommand{\KL}{\textup{KL}}

\DeclareMathOperator*{\argmin}{argmin} 
\DeclareMathOperator*{\argmax}{argmax} 
\usepackage{enumitem}
\usepackage{graphicx}

\newcommand{\breg}{\textup{Breg}}

\newcommand{\piref}{\pi_{\rm ref}}
\usepackage{subcaption}
\newcommand{\algoff}{\textsc{E2D.OR}\xspace}
\newcommand{\reff}{\textup{ref}}

\newcommand{\odecoff}{\mathsf{Ordec}\text{-}\mathsf{O}}
\newcommand{\odecratio}{\mathsf{Ordec}\text{-}\mathsf{R}}
\newcommand{\odec}{\mathsf{Ordec}}

\newcommand{\pesratio}{\mathsf{Gdec}}
\newcommand{\algpes}{\textsc{GDE}\xspace}
\newcommand{\av}{\mathsf{av}}

\newcommand{\bc}{\mathsf{bc}}
\newcommand{\dr}{\mathsf{wr}}
\newcommand{\brr}{\mathsf{br}}
\newcommand{\Dav}{D_{\mathsf{av}}}

\newcommand{\gapcomp}{\mathsf{ER}}

\newcommand{\hatcalT}{\hat{\calT}}
\newcommand{\err}{\varepsilon}

\newcommand{\rhohat}{\hat{\rho}}

\newcommand{\calT}{\mathcal{T}}
\newcommand{\calG}{\mathcal{G}}

\newcommand{\calA}{\mathcal{A}}

\newcommand{\calS}{\mathcal{S}}
\newcommand{\calM}{\mathcal{M}}
\newcommand{\calW}{\mathcal{W}}

\newcommand{\calD}{\mathcal{D}}
\newcommand{\calF}{\mathcal{F}}

\DeclareMathOperator*{\poly}{poly} 
\newcommand{\E}{\mathbb{E}}

\newcommand{\order}{O}

\newcommand{\inner}[1]{\left\langle#1\right\rangle}

\newcommand{\ind}{\mathbf{1}}

\newcommand{\term}{\textbf{term}}

\newcommand{\avg}{\mathsf{avg}}

\newtheorem*{theorem*}{Theorem}
\newtheorem{theorem}{Theorem}
\newtheorem{example}{Example}
\newtheorem{assumption}{Assumption}

\newtheorem{lemma}[theorem]{Lemma}

\newtheorem{definition}{Definition}

\usepackage{tikz}

\newcommand{\nonl}{\renewcommand{\nl}{\let\nl}}

\newcommand{\numberthis}{\refstepcounter{equation}\tag{\theequation}}

\usepackage{prettyref}
\newcommand{\pref}[1]{\prettyref{#1}}

\newcommand{\savehyperref}[2]{\texorpdfstring{\hyperref[#1]{#2}}{#2}}
\newrefformat{eq}{\savehyperref{#1}{\textup{(\ref*{#1})}}}
\newrefformat{eqn}{\savehyperref{#1}{Equation~\ref*{#1}}}
\newrefformat{lem}{\savehyperref{#1}{Lemma~\ref*{#1}}}
\newrefformat{lemma}{\savehyperref{#1}{Lemma~\ref*{#1}}}
\newrefformat{def}{\savehyperref{#1}{Definition~\ref*{#1}}}
\newrefformat{line}{\savehyperref{#1}{Line~\ref*{#1}}}
\newrefformat{thm}{\savehyperref{#1}{Theorem~\ref*{#1}}}
\newrefformat{corr}{\savehyperref{#1}{Corollary~\ref*{#1}}}
\newrefformat{cor}{\savehyperref{#1}{Corollary~\ref*{#1}}}
\newrefformat{sec}{\savehyperref{#1}{Section~\ref*{#1}}}
\newrefformat{subsec}{\savehyperref{#1}{Section~\ref*{#1}}}
\newrefformat{app}{\savehyperref{#1}{Appendix~\ref*{#1}}}
\newrefformat{assum}{\savehyperref{#1}{Assumption~\ref*{#1}}}
\newrefformat{ex}{\savehyperref{#1}{Example~\ref*{#1}}}
\newrefformat{fig}{\savehyperref{#1}{Figure~\ref*{#1}}}
\newrefformat{alg}{\savehyperref{#1}{Algorithm~\ref*{#1}}}
\newrefformat{rem}{\savehyperref{#1}{Remark~\ref*{#1}}}
\newrefformat{conj}{\savehyperref{#1}{Conjecture~\ref*{#1}}}
\newrefformat{prop}{\savehyperref{#1}{Proposition~\ref*{#1}}}
\newrefformat{proto}{\savehyperref{#1}{Protocol~\ref*{#1}}}
\newrefformat{prob}{\savehyperref{#1}{Problem~\ref*{#1}}}
\newrefformat{claim}{\savehyperref{#1}{Claim~\ref*{#1}}}
\newrefformat{que}{\savehyperref{#1}{Question~\ref*{#1}}}
\newrefformat{op}{\savehyperref{#1}{Open Problem~\ref*{#1}}}
\newrefformat{fn}{\savehyperref{#1}{Footnote~\ref*{#1}}}
\newrefformat{tab}{\savehyperref{#1}{Table~\ref*{#1}}}
\newrefformat{fig}{\savehyperref{#1}{Figure~\ref*{#1}}}
\newrefformat{proc}{\savehyperref{#1}{Procedure~\ref*{#1}}}

\DeclareMathOperator*{\unif}{\mathsf{Unif}}

\newcommand{\PP}{\mathbb{P}}
\newcommand{\Rep}{\mathsf{Rep}}
\newcommand{\sA}{s_{\mathsf{A}}}
\newcommand{\sB}{s_{\mathsf{B}}}
\newcommand{\Ber}{\mathsf{Ber}}
\newcommand{\WA}{W_{\mathsf{A}}}
\newcommand{\WB}{W_{\mathsf{B}}}

\newcommand{\Cbar}{\underline{C}}
\newcommand{\mubar}{{\underline{\mu}}}

\usepackage{graphicx}
\newcommand{\oo}{\circ}
\newcommand{\xx}{{\mathbin{\scalebox{0.65}{\ensuremath{\times}}}}}

\usepackage{titletoc}

\newcommand{\odc}{\mathsf{ODC}}

\newcommand{\ode}{\mathsf{ODE}}

\newcommand{\oir}{\mathsf{OIR}}
\newcommand{\odr}{\mathsf{ODR}}

\newcommand{\picomp}{\mathring{\pi}}
\newcommand{\ri}{\text{ri}}
\newcommand{\conf}{\mathrm{conf}}
\newcommand{\stat}{\mathrm{stat}}
\newcommand{\opdc}{\mathsf{OPDC}}
\newcommand{\opdr}{\mathsf{OPDR}}
\newcommand{\dom}{\mathrm{dom}}
\newcommand{\Breg}{\text{Breg}}
\newcommand{\df}{\mathsf{\Phi}}

\usepackage{titlesec}

\hypersetup{pdftitle={On the Complexity of Offline Reinforcement Learning with Q*-Approximation and Partial Coverage},pdfauthor={Haolin Liu, Braham Snyder, Chen-Yu Wei}}
\AtBeginDocument{\DeclareMathAlphabet{\mathsf}{OT1}{cmss}{m}{n}}
\title{ On the Complexity of Offline Reinforcement Learning with $Q^\star$-Approximation and Partial Coverage}

\author{%
  Haolin Liu$^*$ \qquad Braham Snyder$^*$ \qquad Chen-Yu Wei\thanks{Authors are listed in alphabetical order.} \\
  University of Virginia\\
  \texttt{\{srs8rh, dqr2ye, kfw6en\}@virginia.edu} \\
}

\begin{document}

\maketitle

\begin{abstract}
We study offline reinforcement learning under $Q^\star$-approximation and partial coverage, a setting that motivates practical algorithms such as Conservative $Q$-Learning (CQL) \citep{kumar2020conservative} but has received limited theoretical attention. Our work is inspired by the following open question: \emph{Are $Q^\star$-realizability and Bellman completeness sufficient for sample-efficient offline RL under partial coverage?}
We answer this question in the negative through an information-theoretic lower bound. To identify additional structure that enables sample-efficient offline RL under partial coverage, we introduce a general decision-estimation framework, inspired by model-free decision-estimation coefficients (DEC) for online RL \citep{foster2024model, liu2025improved}. Our framework decomposes the complexity of offline RL into two parts: the \emph{decision complexity} and the \emph{value estimation error}.  This decomposition allows us to study the two sub-problems in a modular way. Our result not only unifies existing results in the $Q^\star$-approximation and partial coverage regime \citep{chen2022offline, uehara2023offline}, but further improves and generalizes them. On the decision complexity side, our improvement includes: the first $\epsilon^{-2}$ sample complexity bound for soft $Q$-learning under partial coverage that improves \cite{uehara2023offline}'s $\epsilon^{-4}$ bound, the removal of the need for additional online interaction in the value-gap setting of \cite{chen2022offline}, and new learnable settings beyond the above two cases. On the value estimation side, we provide a new characterization of the role of Bellman completeness under partial coverage, and the first characterization of offline learnability for general low-Bellman-rank MDPs \cite{jiang2017contextual, du2021bilinear, jin2021bellman}. The latter is a canonical online RL setting that has remained unexplored in offline RL except for special cases. As a side contribution, our techniques give the first analysis of CQL in the function approximation setting.
\end{abstract}

\section{Introduction}
Offline Reinforcement Learning (RL) studies policy learning from a fixed dataset, without interaction with the environment. This paradigm is appealing because it enables learning from data that may not have been generated by the learner itself. In practice, however, it is rarely realistic to assume that the offline data covers the entire state-action space. In such \emph{partial coverage} settings, algorithms typically rely on \emph{pessimism} to ensure safety against worst possible environments. 

Among value-based methods for offline RL with pessimism, two classes of algorithms are widely used: $Q$-learning-based methods such as Conservative $Q$-Learning (CQL) \citep{kumar2020conservative}, and actor-critic methods such as Behavior-Regularized Actor-Critic (BRAC)  \citep{wu2019behavior}. From a theoretical perspective, these methods are commonly analyzed under the $Q^\star$-realizability and $Q^\pi$-realizability assumptions, respectively. Under $Q^\pi$-realizability, \cite{xie2021bellman} gives a relatively complete sample-complexity characterization. Under $Q^\star$-realizability, however, existing results remain less complete. On the lower bound side,  it is known that $Q^\star$-realizability and partial coverage alone are not sufficient for sample-efficient learning \cite{foster2022offline, jia2024offline}.  On the upper bound side, existing works \citep{chen2022offline, uehara2023offline} rely on estimating additional quantities such as density ratios or Lagrange multipliers, and require extra conditions such as value gaps or behavior regularization. 
The fundamental roles of these additional assumptions remain unclear. More related works can be found in \pref{app:related works}.

In this work, we aim to provide a more comprehensive understanding of the potential and limitations of offline RL under 
$Q^\star$-approximation and partial coverage. Our contributions 
are twofold. First, we establish an information-theoretic lower 
bound showing that $Q^\star$-realizability, partial coverage, and Bellman completeness are \emph{not} sufficient for sample-efficient 
learning, indicating that the difficulty of offline RL under partial coverage does not merely come from the estimation of $Q^\star$. Motivated by this observation, we introduce a decomposition of the performance gap into \emph{decision complexity} and \emph{estimation error}, which together capture the overall difficulty of the problem. This decomposition yields a 
modular framework that recovers, improves, and extends prior 
results. A detailed overview of our results is given in 
\pref{sec:overview}.





\section{Results Overview}
\label{sec:overview}


\paragraph{Lower bound (\pref{sec: it lower bound}).}

We show that $Q^\star$-realizability, partial coverage, and Bellman completeness are \emph{not} sufficient for sample-efficient offline RL. This is somewhat surprising because these
assumptions allow the learner to construct a small confidence set containing the true $Q^\star$ function~\citep{jin2021bellman,xie2023role,xie2021bellman}. This indicates that even when the true value function is well localized, the learner may still be unable to decide which policy is safe to deploy. This motivates the need for an additional complexity measure that captures the difficulty of making robust decisions from a confidence set.


\paragraph{Upper bound: a decision-estimation decomposition (\pref{sec:unify}).}
To capture this missing decision difficulty, we introduce the offline
robust decision-estimation coefficient, $\odec$. Under $Q^\star$-realizability, the offline dataset induces a confidence set $\calF_\conf$ of plausible optimal $Q^\star$-functions. This in turn induces a set 
    $\calM_\conf = \{M : Q^\star_M \in \calF_\conf\}$
of plausible models, where $Q^\star_M$ denotes the optimal $Q$-function induced by model $M$. The learner must choose a policy that performs well in the true model, but the true model is only known to lie within the ambiguity set represented by $\calM_\conf$.

$\odec$ measures the value of a game between a policy player and a model player. The \emph{policy player} chooses a policy distribution $\rho$ to minimize the suboptimality gap $J_M(\pi_M) - \mathbb{E}_{\pi \sim \rho}[J_M(\pi)]$, where $J_M(\pi)$ denotes the expected payoff of $\pi$ under model $M$ and $\pi_M$ is the optimal policy under $M$. The \emph{model player} chooses a model $M \in \mathcal{M}_{\mathrm{conf}}$ that makes this gap large. However, the payoff of the adversary is offset by a discrepancy term that penalizes models that are statistically inconsistent with the functions in the confidence set. 

Formally, for parameter $\gamma>0$ and discrepancy measure
$D^\pi(f\|M)$ between function $f$ and model $M$ under the state-action distribution induced by $\pi$, we define
\begin{align*}
\odecoff^{D}_{\gamma}(\mathcal F_{\mathrm{conf}})
:=
\min_{\rho\in\Delta(\Pi)}
\max_{M\in\mathcal M_{\mathrm{conf}}}
\mathbb E_{\pi\sim\rho}
\!\left[
    J_M(\pi_M)-J_M(\pi)
    -\gamma\max_{f\in\mathcal F_{\mathrm{conf}}}
    D^{\pi_M}(f\|M)
\right]   \numberthis \label{eq: minmax}
\end{align*}
All remaining notation is introduced in \pref{sec:pre}. Our algorithm outputs a policy distribution $\hat{\rho}$ that is guaranteed to have (see \pref{thm: main})
\begin{align*}
  \underbrace{\E_{\hat{\pi}\sim \hat{\rho}}[J_{M^\star}(\pi^\star) - J_{M^\star}(\hat{\pi})]}_{\textbf{suboptimality}}
\leq
\underbrace{
\odecoff^{D}_{\gamma}(\mathcal F_{\mathrm{conf}})
}_{\textbf{decision complexity}}
+
\gamma
\underbrace{
    \max_{f\in\mathcal F_{\mathrm{conf}}}
    D^{\pi^\star}(f\|M^\star)
}_{\textbf{estimation error}}  \numberthis \label{eq: decompose bound2}
\end{align*}
where $M^\star$ is the underlying true model and $\pi^\star$ is the true optimal policy. 

This is the main takeaway of the paper. Offline RL under partial
coverage has two challenges: the learner must estimate the value function well
enough to make the estimation error small, and the remaining ambiguity must
still allow for a robust decision to make the decision complexity small. Conditions controlling the decision complexity and
controlling the estimation error can be combined modularly. This modular
view is summarized in \pref{fig:overview}. For each term, we provide a list of conditions under which they can be bounded. This not only subsume all upper bound results we are aware of for offline RL with $Q^\star$-realizability and partial coverage under general function approximation, but also extend to new settings.

\begin{figure}[t]
\vspace{-3mm}
    \centering
    \includegraphics[width=\textwidth]{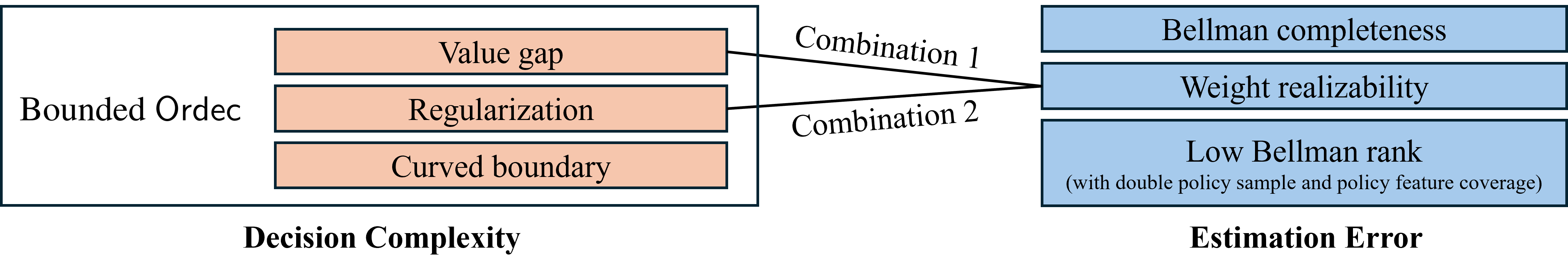}
    \vspace{-2mm}
    \caption{Overview of the decomposition under $Q^\star$ realizability and partial coverage. Assumptions in the decision complexity part can be combined with any assumptions in the estimation error part, leading to a learnable setting. Three decision complexity bounds are established in \pref{thm: gap bound}, \pref{thm: curved boundary}, \pref{thm:regular-dec}; three estimation error bounds are in \pref{lem: BC case}, \pref{lem: weight case}, \pref{lem: Bellman rank case}. Combinations~1 and~2 are studied by \cite{chen2022offline} and \cite{uehara2023offline}, respectively; our results improve over theirs in multiple aspects (see \pref{sec: comparison overview}). Besides them, the other seven combinations are new. }
    \label{fig:overview}
    \vspace{-2mm}
\end{figure}

\paragraph{Relation to online DEC.}
Our decomposition resembles the decision-estimation coefficient framework for online RL~\citep{foster2021statistical, foster2023tight, 
foster2024model, liu2025improved}. The key difference is \emph{under which state-action distribution} the
discrepancy is measured. In online DEC, the discrepancy is measured under the learner's policy, because the learner can collect data using that policy. In offline RL, the data distribution is fixed, so the learner's policy cannot help test whether a model is plausible. Instead, $\odec$ measures the discrepancy under the candidate model's own optimal policy $\pi_M$. This is what makes the
sub-optimality bound adaptive to partial coverage: models whose optimal policies are poorly supported by the offline data are penalized through the discrepancy term. We elaborate more in \pref{sec: offdec}.

\subsection{Comparison with prior work in offline RL with $Q^\star$-realizability and partial coverage} \label{sec: comparison overview}
Despite the natural appeal of this setting, positive results for $Q^\star$-realizability and partial coverage are scarce. To our knowledge, only \cite{chen2022offline} and \cite{uehara2023offline} have identified learnability conditions for this setting under general function approximation. As discussed above, bounded decision complexity and the ability to accurately estimate value functions are both key to establishing learnability. While not explicitly stated, the results of \cite{chen2022offline} and \cite{uehara2023offline} each rely on a combination of conditions that ensure both terms are bounded, as indicated in \pref{fig:overview}. In \pref{fig:overview}, we list three sufficient conditions that allow us to control decision complexity and estimation error, respectively, resulting in nine combinations in total. Besides the two identified by \cite{chen2022offline} and \cite{uehara2023offline}, the learnability of the remaining seven combinations has not been established before.

\paragraph{Improvement over \cite{chen2022offline} and \cite{uehara2023offline} (\pref{sec:decision-error}).} Even under the same set of assumptions, our algorithm and the analysis framework give improved guarantees over \cite{chen2022offline} and \cite{uehara2023offline}. Specifically, in the unregularized value-gap setting (the first condition for decision complexity in \pref{fig:overview}), our suboptimality guarantee automatically adapts to the gap of the true $Q^\star$ function without knowing the gap value, while \cite{chen2022offline} requires either knowing the gap value or additional online interaction with the environment to achieve similar guarantee as ours. This improvement comes from our new algorithm design based on $\odec$.
In the regularized setting (the second condition for decision complexity in \pref{fig:overview}), we obtain the first $\epsilon^{-2}$ sample complexity bound, improving the $\epsilon^{-4}$ bound of \cite{uehara2023offline}. This improvement comes from a novel performance difference lemma we develop for regularized MDP with partial coverage. Besides, we relax the requirement of \cite{uehara2023offline} that the offline data is generated from the reference policy in the regularization, and remove the polynomial dependence on the number of actions in their sample complexity bound. 

\paragraph{New settings with bounded decision complexity (\pref{sec:decision-error}).}
$\odec$ not only recovers the value-gap and the regularization settings studied in \cite{chen2022offline} and \cite{uehara2023offline}, but also extends to new settings. We provide a third condition that leads to bounded $\odec$ (left side of \pref{fig:overview}): curved decision boundary. This addresses the concerns raised by \cite{song2022hybrid, uehara2023offline} on the practicality of the value-gap assumption in \cite{chen2022offline} for continuous action spaces.

\paragraph{New settings with bounded estimation error (\pref{sec:estimation-error}).} On the estimation error side, \cite{chen2022offline} and \cite{uehara2023offline} both rely on the weight realizability assumption to control the error, leaving other possibilities open. We provide two additional conditions: Bellman completeness and low Bellman rank (the right-hand side of \pref{fig:overview}). While both are extensively studied under online RL and Bellman completeness has been studied under offline RL with full coverage \cite{chen2019information}, they are unexplored in offline RL with partial coverage. Specifically, partial coverage with Bellman completeness is an ideal theoretical framework for the practical CQL algorithm \cite{kumar2020conservative}. To our knowledge, no prior work has provided a sample complexity analysis for CQL beyond the tabular case, and our work achieves it for the first time.
On the other hand, low Bellman rank is a canonical setting for online RL \citep{jiang2017contextual,du2021bilinear,jin2021bellman} that remains unexplored in offline RL (even with full coverage) except for special cases. In fact, we show that under standard data assumptions and coverage definitions, low Bellman rank is  insufficient to control the estimation error, thus remaining not learnable even with bounded decision complexity. We propose to replace the standard data and coverage assumptions with \emph{double policy sampling} and \emph{policy feature coverage}, which allow for efficient estimation error control. We also show that relaxing either assumption back to the standard definition breaks the learnability.

\section{Notation and Setting}
\label{sec:pre}
We deal with both unregularized and regularized MDPs. For simplicity, we formulate regularized MDPs, while unregularized MDPs are just special cases when the regularizer is zero.


\textbf{Regularized Markov Decision Processes. } We consider a finite-horizon \emph{regularized} MDP defined by $M = (\calS, \calA, P_M, R_M, s_1, H, \psi)$, where $\calS$ is the state space, $\calA$ is the action space, $P_M: \calS \times \calA \rightarrow \Delta(\calS)$ is the transition function, $R_M: \calS \times \calA \rightarrow [0,1]$ is the reward function, and $\psi:\Delta(\calA) \times \calS \to \mathbb{R}_{\geq 0}$ is a state-dependent convex regularizer for action distributions.   
  We assume a \emph{layered} state space $\calS$, i.e., $\mathcal{S} = \mathcal{S}_1 \cup \cdots \cup \mathcal{S}_H$ where
$\mathcal{S}_h \cap \mathcal{S}_{h'} = \emptyset$ for any $h \neq h'$. Transitions only occur between adjacent layers: $P_M(s' \mid s, a) > 0$ only if $s \in \mathcal{S}_h$ and $s' \in \mathcal{S}_{h+1}$.  
Without loss of generality, let  $\mathcal{S}_1 = \{s_1\}$.  A policy $\pi$: $\calS\to\Delta(\calA)$ maps states to action distributions. Let $\Pi$ be the set of policies. For any policy~$\pi$, the regularized state value function is $V_M^\pi(s) \triangleq 
\mathbb{E}_M^\pi\big[
\sum_{h'=h}^{H} \big(
R_M(s_{h'},a_{h'}) - \psi (\pi(\cdot | s_{h'}); s_{h'})
\big)\big|
s_h=s \big]$ for $s\in\calS_h$, where $\mathbb{E}_M^\pi[\cdot]$ denotes expectation over trajectories induced
by $\pi$ and $P_M$. Furthermore, denote $J_M(\pi)\triangleq V^\pi_M(s_1)$. By induction, the Bellman equations hold: $Q_M^\pi(s,a) = R_M(s,a) + \mathbb{I}[s \notin \mathcal{S}_H]\cdot \mathbb{E}_{s' \sim P_M(\cdot \mid s,a)}\!\left[ V_M^\pi(s') \right]$ and $V_M^\pi(s) =
\mathbb{E}_{a\sim \pi(\cdot\mid s)}\!\left[Q_M^\pi(s,a)\right]
-
\psi(\pi(\cdot | s); s).$




The optimal regularized value functions are $Q_{M}^\star(s,a)=\max_\pi Q_{M}^\pi(s,a)$ and $V_{M}^\star(s) = \max_\pi V_{M}^\pi(s)$, and the optimal policy is $\pi_{M} \in\argmax_\pi J_{M}(\pi)$. 
We have  $Q_{M}^\star(s,a) =
R_M(s,a) + \mathbb{I}[s\notin \calS_H]\cdot \mathbb{E}_{s'\sim P_M(\cdot\mid s,a)}\!\left[V_M^{\star}(s')\right]$, 
$V_M^\star(s) = 
\textstyle\max_{p \in \Delta(\calA)}
\big\{
\mathbb{E}_{a \sim p}\left[Q_M^\star(s,a)\right] - \psi(p; s)
\big\}$ and 
$\pi_M(\cdot\mid s) \in \textstyle\argmax_{p\in\Delta(\calA)}\big\{\mathbb{E}_{a\sim p}\left[Q_M^\star(s,a)\right] - \psi(p; s)\big\}.$



For a policy $\pi$, denote the occupancy measure as $d_M^\pi(s) \triangleq \mathbb{E}_M^\pi[\mathbb{I}\{s_h = s\}]$ for $s \in\calS_h$, and $d_M^\pi(s,a) \triangleq d_M^\pi(s)\pi(a | s)$. Though $d^\pi_M$ is not a distribution as $\sum_{s\in\calS} d^\pi_M(s)=H$, we denote $\E_{s \sim d^\pi}[g(s)]\triangleq \sum_s d^\pi(s)g(s)$ for any function $g$. Assume the true environment follows MDP $M^\star = (\calS, \calA, P, R, s_1, H, \psi)$, and for simplicity, we write $V^\pi(s) \triangleq V^\pi_{M^\star}(s)$, $Q^{\pi}(s,a) \triangleq Q^\pi_{M^\star}(s,a), V^\star(s) \triangleq V^\star_{M^\star}(s)$, $Q^\star(s,a) \triangleq Q^\star_{M^\star}(s,a)$, $J(\pi)\triangleq J_{M^\star}(\pi)$, $\pi^\star \triangleq \pi_{M^\star}$, and $d^\pi \triangleq d^\pi_{M^\star}$.

\paragraph{Offline Reinforcement Learning. }  In offline RL, the learner only has access to a pre-collected dataset $\calD$ from MDP $M^\star$ and cannot interact with the environment. The dataset $\mathcal{D}$ consists of $n$ i.i.d. tuples $(s, a, r, s')$, where $(s, a) \sim \mu$ for some unknown distribution $\mu \in \Delta(\mathcal{S} \times \mathcal{A})$, $\mathbb{E}[r | s, a] = R(s, a)$, and $s' \sim P(\cdot | s, a)$. 
The goal is to learn a policy $\hat{\pi}$ that minimizes the sub-optimality $J(\pi^\star) - J(\hat{\pi})$. 
We seek \emph{coverage-adaptive} performance, where the sub-optimality adapts to the degree to which $\mu$ covers $\pi^\star$ (\pref{def: coverage}). 

\begin{definition}[Coverage]\label{def: coverage} For any policy $\pi$, define $C^{\pi}=\max_{(s,a)\in\calS\times \calA} \frac{d^{\pi}(s,a)}{H \mu(s,a)}$.\footnote{The division by $H$ is to correct the mismatch between $\sum_{s,a} d^\pi(s,a)=H$ and $\sum_{s,a}\mu(s,a)=1$. }
\end{definition}

In value-based offline RL, function approximation is used to approximate value functions. In this paper, we focus on the $Q^\star$-approximation scheme and assume $Q^\star$-realizability:  
\begin{definition}[$Q^\star$-realizability]\label{def: Q realiability}
The learner has access to a function class $\calF\subset \{f: \calS\times\calA\to \mathbb{R}\}$ such that $Q^\star\in \calF$. For any $f\in\calF$, denote $f(s) = \max_{p\in\Delta(\calA)}
\big\{
\mathbb{E}_{a\sim p}\left[f(s,a)\right]-\psi(p;s)
\big\}$ and denote $\pi_f(\cdot\mid s)
=
\argmax_{p\in\Delta(\calA)}
\big\{
\mathbb{E}_{a\sim p}\left[f(s,a)\right]-\psi(p;s)
\big\}$ with arbitrary tie-breaking.  
\end{definition}
\noindent\textbf{Bellman Operator.}\ \ \ \ 
Define the Bellman operator $\calT$ as $\calT f(s,a) = R(s,a) + \mathbb{E}_{s' \sim P(\cdot \mid s,a)}[f(s')]$.  

\noindent\textbf{Regularizer.}\ \ \ 
For convex function $\df: \Delta(\calA) \rightarrow \mathbb{R}$, let $\text{Breg}_{\df}(x,y) \triangleq \df(x) - \df(y) - \left\langle \nabla \df(y), x-y \right\rangle$ be the Bregman divergence defined with $\df$. We write $\text{Breg}_{\psi}(x,y;s) \triangleq \text{Breg}_{\psi(\cdot; s)}(x,y)$ for any $x,y \in \Delta(\calA)$ and  $\text{Breg}_{\psi}(\pi ,\pi';s) \triangleq \text{Breg}_{\psi(\cdot; s)}(\pi(\cdot|s), \pi'(\cdot|s))$ for any policies $\pi, \pi'$. We focus on regularizer $\psi(p;s) = \alpha \text{Breg}_{\df}(p; \pi_{\rm ref}(\cdot|s))$ where $\df$ is a convex function, $\pi_{\rm ref}$ is a given reference policy, and $\alpha \geq 0$ is the weight of regularization. We also write $\psi(\pi(\cdot|s);s) \triangleq \psi(\pi; s)$.

\section{An Information-Theoretic Lower Bound} \label{sec: it lower bound}

We begin with an information-theoretic lower bound showing that 
$Q^\star$-realizability, partial coverage, and Bellman completeness 
(defined below) are insufficient for sample-efficient offline RL.


\begin{assumption}[Bellman completeness]\label{assum: Bellman complete}
Along with the $\calF$ in \pref{def: Q realiability}, the learner has access to a function class $\mathcal{G}\subset \{g: \calS\times\calA\to \mathbb{R}\}$ such that $\mathcal{T}f \in \mathcal{G}$ for any $f \in \calF$. 
\label{assum:BC}
\end{assumption}

\begin{theorem}\label{thm: eps dep lower bound}
    There exists a family of MDPs $\calM$, a function class $\calF$ with $|\calF|=4$, and an offline data distribution~$\mu$ such that under any true model $M^\star\in \calM$, $Q^\star$-realizability and Bellman completeness hold (with $\calG=\calF$), and has coverage $C^{\pi^\star}=\Theta(1)$. However, to achieve \mbox{$\E[J(\pi^\star) - J(\hat{\pi})]\leq \epsilon$} with $M^\star\sim \unif(\calM)$, the learner must access $n\geq \Omega\big(\frac{1}{\epsilon} \min\big\{\sqrt{|\calS|}, \frac{1}{\Delta^2}\big\}\big)$ offline samples, where $\Delta=\min_{s} (V^\star(s) - \max_{a\neq \pi^\star(s)} Q^\star(s,a))$ is a parameter which can be chosen arbitrarily in $[0,\frac{1}{4}]$, and the number of states $|\calS|$ can be arbitrarily large. 
    
    Even if the learner can access trajectory data $(s_1, a_1, r_1, \cdots, s_H, a_H, r_H)$ from an offline policy $\pi_b$ such that $\max_{s,a}\frac{d^{\pi^\star}(s,a)}{d^{\pi_b}(s,a)} = \Theta(1)$, $n\geq \Omega\big(\frac{1}{\epsilon \poly(H)} \min\{2^H, \sqrt{|\calS|}, \frac{1}{\Delta^2}\}\big)$ trajectories are required. 
\end{theorem}
\vspace{-2mm}


\begin{figure}[t]
    \centering
    \begin{minipage}[b]{0.42\textwidth}
        \centering
        \includegraphics[width=\textwidth]{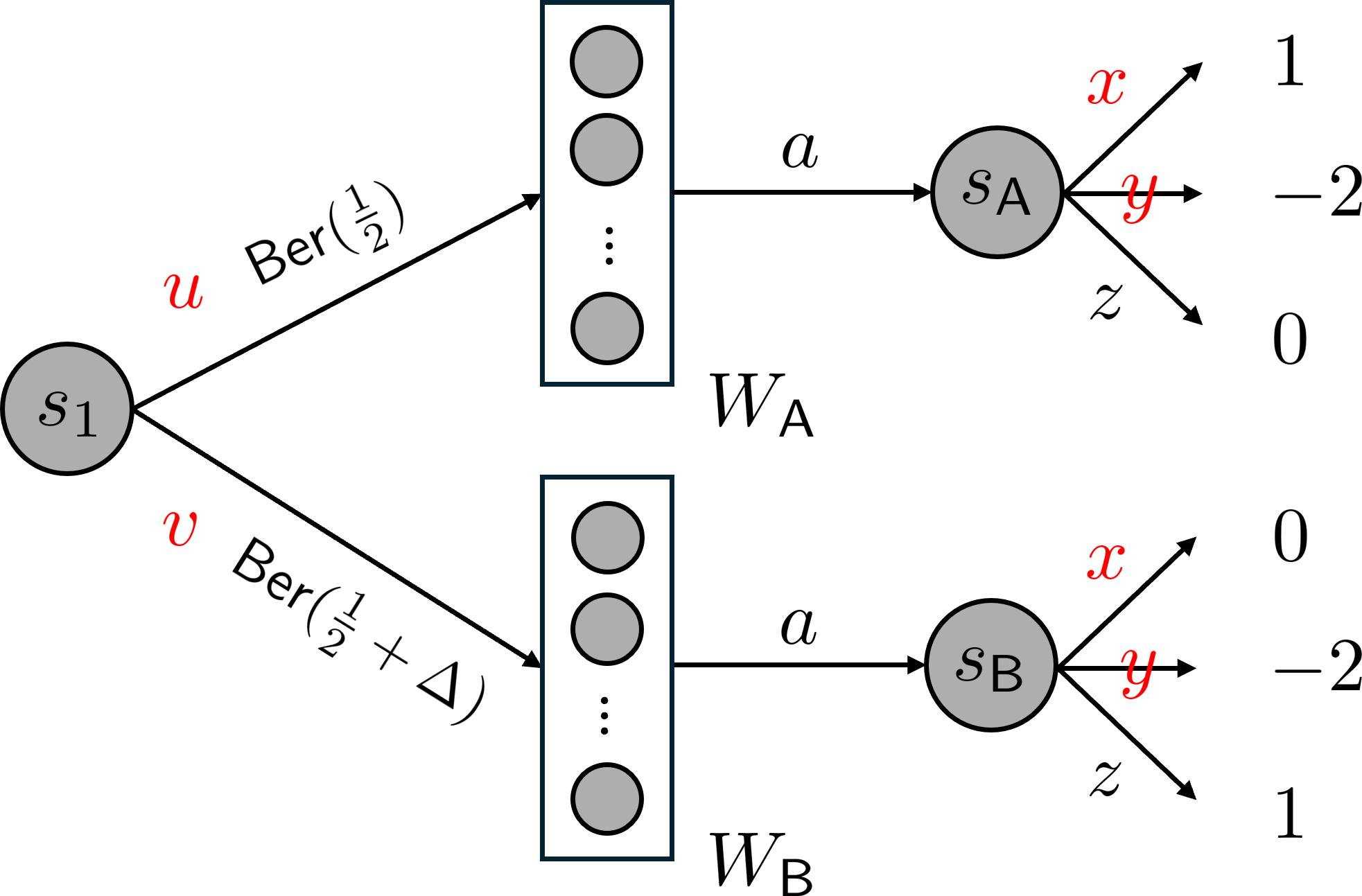}
        \par\medskip
        (a) $\calM_{u,x}$
    \end{minipage}
    \hfill
    \begin{minipage}[b]{0.42\textwidth}
        \centering
        \includegraphics[width=\textwidth]{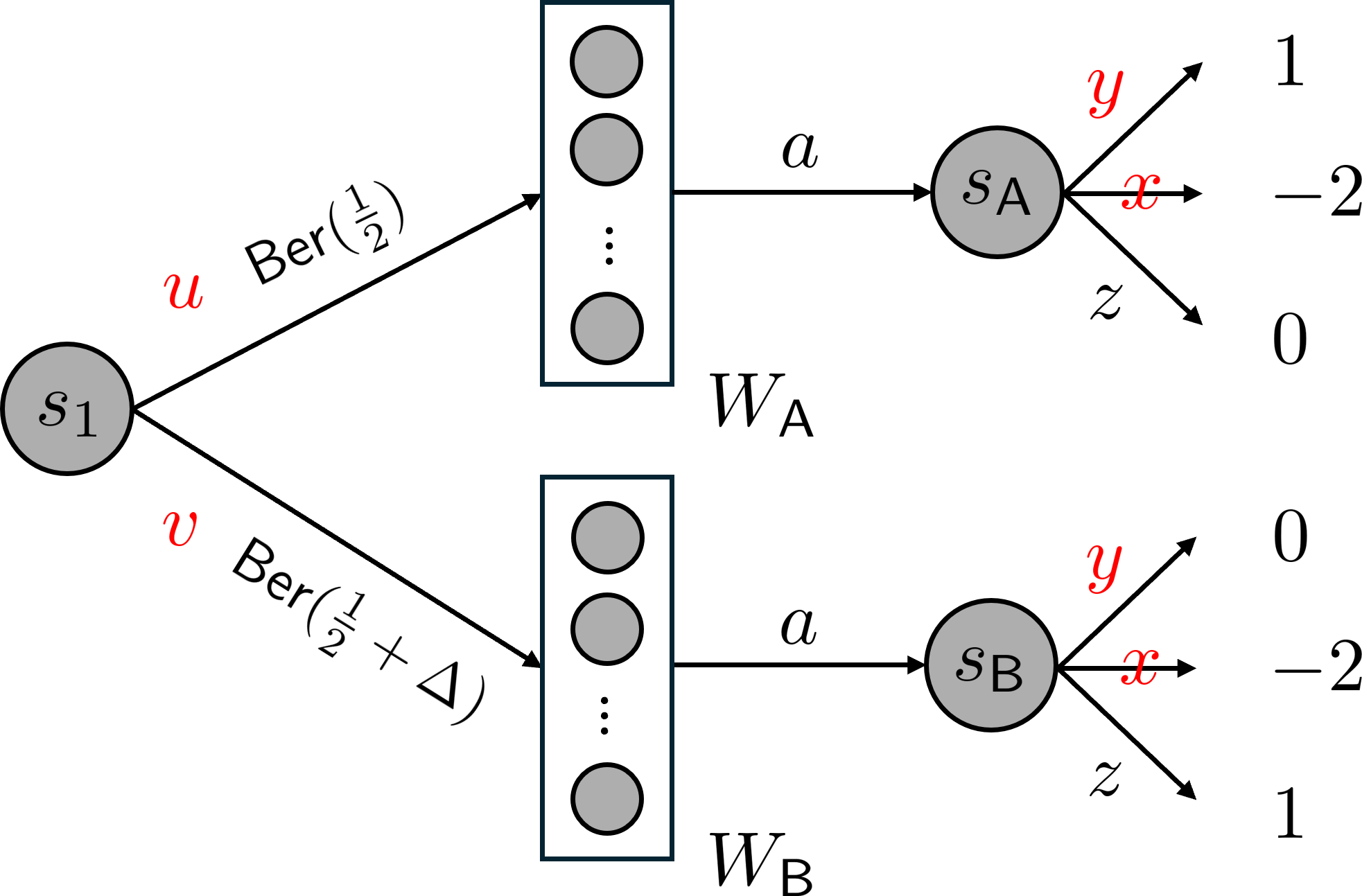}
        \par\medskip
        (b) $\calM_{u,y}$
    \end{minipage}
    
    \vspace{1em}
    
    \begin{minipage}[b]{0.42\textwidth}
        \centering
        \includegraphics[width=\textwidth]{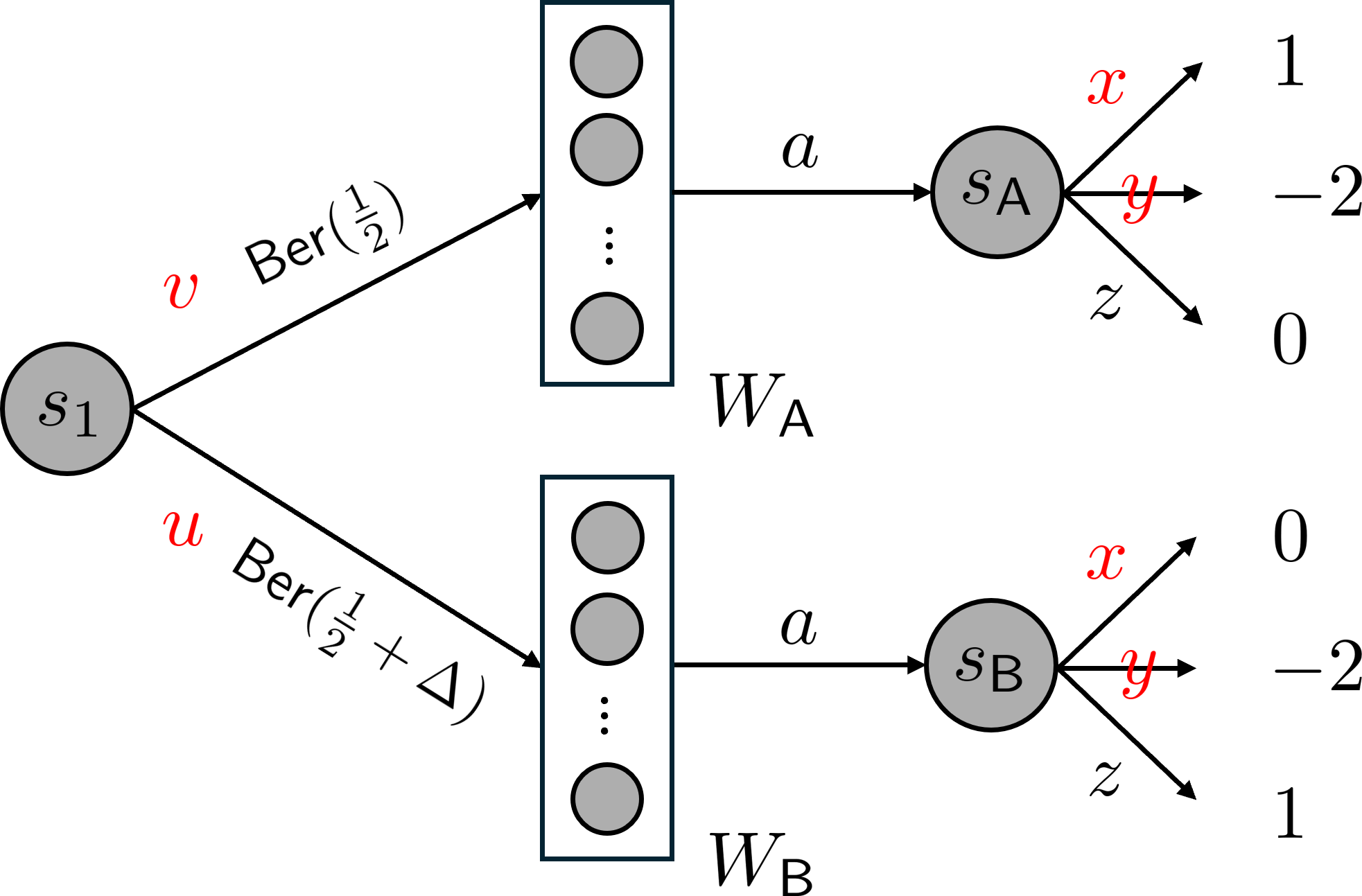}
        \par\medskip
        (c) $\calM_{v,x}$ 
    \end{minipage}
    \hfill
    \begin{minipage}[b]{0.42\textwidth}
        \centering
        \includegraphics[width=\textwidth]{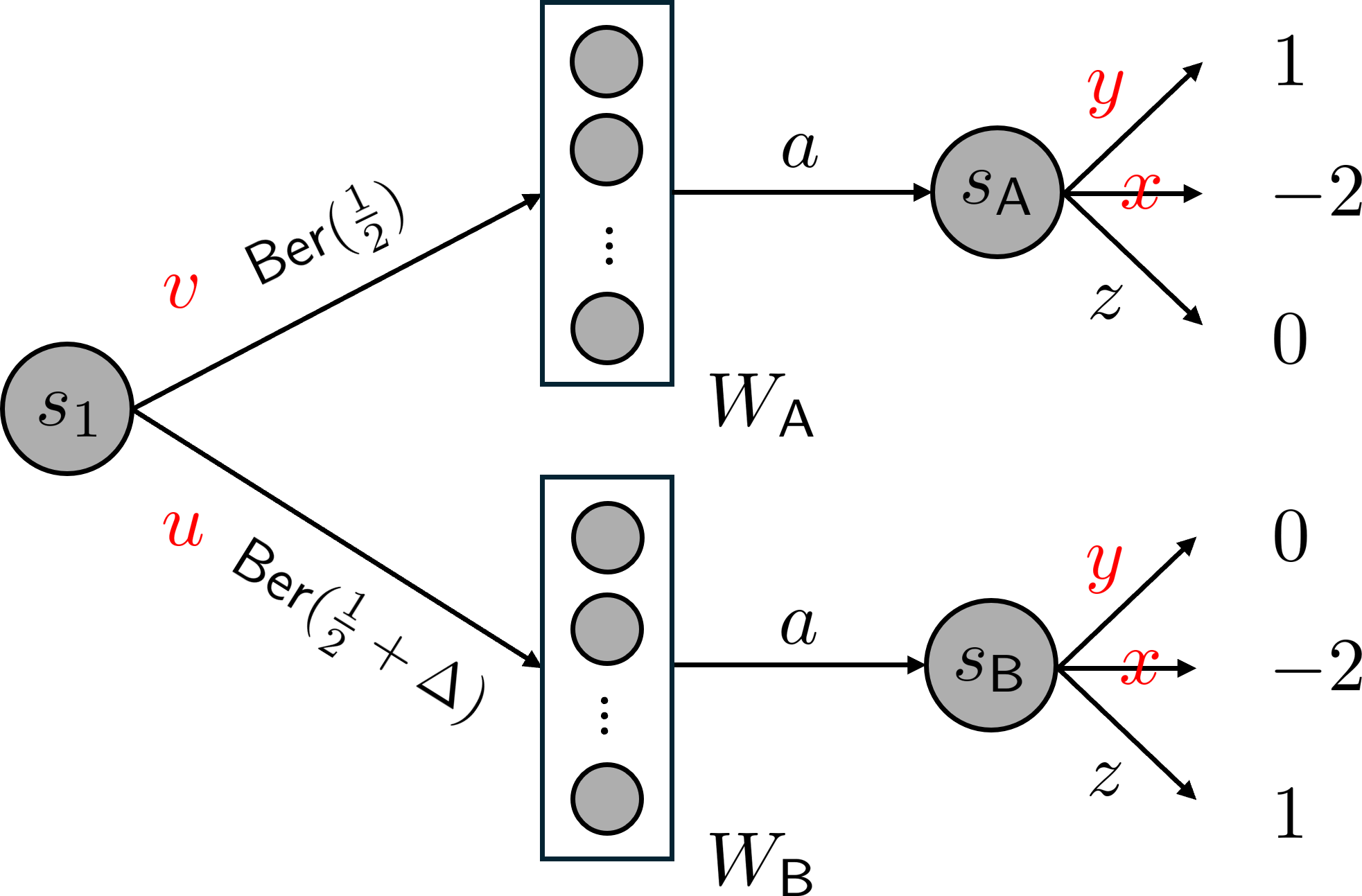}
        \par\medskip
        (d) $\calM_{v,y}$
    \end{minipage}
    \caption{The four MDP families used in the lower-bound construction.
The first subscript denotes the action in \(\{u,v\}\) that leads to
\(\WA\), and the other action leads to \(\WB\).
The second subscript denotes the action in \(\{x,y\}\) that receives
reward \(1\) at \(\sA\).
Red labels mark the components that vary across the four families.
}
    \label{fig:main}
\end{figure}

\begin{proof}[Proof sketch]
We first prove a lower bound of
\(\min\{\sqrt{|\calS|},1/\Delta^2\}\) without the \(\epsilon^{-1}\)
factor (\pref{lem: eps indep lower bound}), under the $(s,a,r,s')$ data format. The construction uses the
four MDP families in \pref{fig:main}, with \(M^\star\) drawn uniformly.
At the initial state \(s_1\), actions \(u\) and \(v\) yield rewards
\(\Ber(1/2)\) and \(\Ber(1/2+\Delta)\), with the higher-reward action
leading uniformly to the group \(\WB\), and the other to \(\WA\). From
\(\WA\) and \(\WB\), the unique action~\(a\) transitions to \(\sA\) and
\(\sB\), respectively, with zero reward. At \(\sA,\sB\), the three
terminal actions have deterministic rewards as shown in \pref{fig:main}.
Thus, \(\pi^\star\) chooses the \(\Ber(1/2+\Delta)\) action at \(s_1\)
and then chooses \(z\) on \(\sB\). The function class \(\calF\), with
\(|\calF|=4\), satisfies both \(Q^\star\)-realizability and Bellman
completeness as shown in \pref{app: function construction}.
The offline distribution \(\mu\) is induced by a behavior policy that
chooses \(u,v\) uniformly at \(s_1\) and chooses \(z\) at \(\sA,\sB\),
so \(C^{\pi^\star}\le 2\). 

Since the learner observes only transition
tuples \((s,a,r,s')\), not trajectories, the data consist of three types:
\textbf{Type~1} samples \((s_1,a_1,r_1,w_1)\), where
\(a_1\sim\unif(\{u,v\})\) and \(w_1\in\WA\cup\WB\);
\textbf{Type~2} samples \((w_2,a,0,s_2)\), where
\(w_2\sim\unif(\WA\cup\WB)\) and \(s_2\in\{\sA,\sB\}\); and
\textbf{Type~3} samples \((s_3,z,r_3)\), where
\(s_3\sim\unif(\{\sA,\sB\})\). To act optimally, the learner must decide
whether \(u\) or \(v\) leads to \(\WB\). Type~1 rewards can reveal this,
but distinguishing means \(1/2\) from \(1/2+\Delta\) requires
\(\Omega(1/\Delta^2)\) samples. The learner could instead try to infer
the group of \(w_1\), but \(\calF\) assigns the same value
\(f(w,a)=1\) to every \(w\in\WA\cup\WB\), so the function class itself
does not reveal the hidden partition. Type~2 samples do reveal whether
a sampled \(w_2\) lies in \(\WA\) or \(\WB\), but when the groups are
large, the learner is unlikely to see the same state in $\WA\cup \WB$ twice in Type~1 and
Type~2 before \(\Omega(\sqrt{|\calS|})\) samples. Finally, the learner should avoid \(x\) and \(y\) at \(\sA,\sB\), since
they are never observed offline and one of them has reward \(-2\). Overall, it is hard for the learner to determine between $u$ and $v$, and the learner should only choose $z$ on $\sA$ and $\sB$. A random guess at
\(s_1\) obtains value \(1/2+\Delta/2+1/2\), which is worse than
\(\pi^\star\) by \(1/2+\Delta/2=\Theta(1)\). In
\pref{thm: eps dep lower bound}, we extend this construction to obtain
the \(\epsilon\)-dependent and trajectory-feedback lower bounds; see
\pref{app: lower bound construct}.
\end{proof}

We note that our gap-dependent lower bound is quite different from those in online RL or simpler offline RL settings \citep{wang2022gap, nguyen2023instance}. In those settings, if the learner cannot tell apart two actions with gap $\Delta$, the sub-optimality suffered is only of order $\Delta$. In our case, before the learner tells apart two actions of gap $\Delta$, the learner suffers $\Theta(1)$ sub-optimality. This difference also shows in the sample complexity bounds: in prior settings, the learner can achieve $O(\frac{1}{\epsilon^2})$ sample complexity regardless of the value gap. In contrast, the sample complexity in \pref{thm: eps dep lower bound} scales with $\frac{1}{\epsilon\Delta^2}$, which can become arbitrarily large even under constant $\epsilon$.

\vspace{-1mm}
\section{Making Robust Decisions from a Confidence Set}
\label{sec:unify}
\vspace{-1mm}


In this section, we isolate the decision-making part of offline RL, assuming that the offline data has already been used to construct a confidence set $\calF_\conf$ of plausible $Q^\star$ functions. The key question is how
the learner should choose a policy from this ambiguity. We show that the usual
value-centric pessimism used in $Q^\star$-based offline RL can be suboptimal,
and we derive a more robust decision rule using $\odec$. The construction of $\calF_\conf$ is an estimation problem and is
discussed in \pref{sec:estimation-error}.



\subsection{Two Pessimism Principles}\label{sec: tale}


We identify two distinct pessimism principles in prior work to derive a policy from the confidence set.  The first, which we refer to as \textbf{policy-centric pessimism}, has been adopted in model-approximation \citep{uehara2021pessimistic} and $Q^\pi$-approximation schemes \citep{xie2021bellman}. This approach aligns with the \emph{robust optimization} principle \citep{gorissen2015practical}, where the learner outputs a policy $\hat{\pi}$ that maximizes the worst-case performance over the uncertainty set of models $\calM_\conf$:
\begin{align}
 \textstyle \hat{\pi} = \argmax_\pi \min_{M\in \calM_\conf}  J_M(\pi)  \label{eq: policy-centric}
\end{align}

However, prior work on $Q^\star$-approximation often follows a different \textbf{value-centric pessimism} principle, as seen in \cite{chen2022offline}. Under this principle, the order between $\min_M$ and $\max_\pi$ are swapped: 
\begin{align}
 \textstyle \hat{M} = \argmin_{M\in \calM_\conf} \max_{\pi}  J_M(\pi), \qquad \text{then let } \hat{\pi} = \pi_{\hat{M}}.   \label{eq: model-centric} 
\end{align}
In fact, this principle is more commonly written as: 
\begin{align*}
   \textstyle \hat{f} \in \argmin_{f \in \calF_{\mathrm{conf}}}
    f(s_1), \qquad \text{then let } \hat{\pi} = \pi_{\hat{f}}.    \numberthis \label{eq: value centr}
\end{align*}
The equivalence between \eqref{eq: model-centric} and \eqref{eq: value centr} can be seen as the following: 
with a confidence set $\calF_\conf$ for $Q^\star$, we can define a confidence set for $M^\star$:  $\calM_\conf = \{M:~ Q_M^\star \in \calF_\conf\}$. Thus, $\min_{f\in\calF_\conf} f(s_1)=$ $\min_{M\in\calM_\conf}V_M^\star(s_1) =$ $ \min_{M\in\calM_\conf}\max_\pi J_M(\pi)$. 

This deviation from robust optimization is not accidental. In model- or $Q^\pi$-approximation, the learner can evaluate $J_M(\pi)$ for an arbitrary policy $\pi$. Under $Q^\star$-approximation, the confidence set only represents optimal value functions. Therefore, for a fixed policy $\pi$, $J_M(\pi)$ is generally not available from $\calF_\conf$ unless $\pi=\pi_M$. This forces standard methods to treat $J_M(\pi_M)=\max_\pi J_M(\pi)$ as the basic object. 
However, value-centric pessimism can produce less desirable decision, as in \pref{ex: bad}. 

\begin{example}\label{ex: bad}
    Assume a single state $s_1$ and two actions $x,y$. Let $\calF_\conf=\{f_x, f_y\}$ with \mbox{$f_x(s_1, x)=1$,} $f_x(s_1, y)=0$, $f_y(s_1, x) = \frac{1}{2}-\Delta$, $f_y(s_1, y)=\frac{1}{2}+\Delta$ for some $\Delta\in(0, 0.01]$. Policy-centric pessimism \pref{eq: policy-centric} chooses action $\argmax_{a}\min_{f\in \{f_x,f_y\}}f(s_1,a) = x$. \mbox{Value-centric pessimism \pref{eq: model-centric} chooses} function $\argmin_{f\in\{f_x, f_y\}} \max_{a}f(s_1,a) = f_y$ and then chooses action $\argmax_{a}f_y(s_1,a) = y$. If the true environment is described by $f_x$, then the sub-optimality of policy-centric pessimism and value-centric pessimism is $0$ and $1$, respectively; if the true environment is described by $f_y$, then their sub-optimality is $2\Delta\leq 0.02$ and $0$, respectively. 
\end{example}
In \pref{ex: bad}, policy-centric pessimism sacrifices a negligible $2\Delta$ in one world to avoid a catastrophic error (gap of $1$) in the other---a robust behavior we usually favor in offline RL.


The question is therefore whether one can recover policy-centric robustness while still working only with a $Q^\star$ confidence set. In the next subsection, we answer this affirmatively. The framework we develop draws inspiration from the \emph{model-free decision-estimation coefficient} (DEC) \citep{foster2024model, liu2025decision, liu2025improved}. 

\subsection{A Decision-Estimation Coefficient for Offline RL under Partial Coverage}\label{sec: offdec}

We first review ideas in the existing online DEC framework. In model-based DEC \citep{foster2021statistical, foster2023tight}, the learner seeks a policy that minimizes regret against an adversarially chosen model. However, this model must remain close---measured under the state-action distribution induced by the learner’s current policy---to the posterior of $M^\star$ estimated from collected data. To achieve this, every model is penalized according to its discrepancy from the model posterior. This corresponds to the first decision rule in \pref{fig:dec-comparison}. Model-free DEC \citep{foster2024model, liu2025improved} follows a similar principle, but instead of maintaining a posterior of $M^\star$, it maintains a posterior of $Q^\star$. The learner's policy is optimized against the \emph{induced model set} $\calM=\{Q_M^\star\in\calF\}$. 
This gives the second decision rule in \pref{fig:dec-comparison}. 

\begin{figure}[H]
\vspace{-2.0mm}
\centering
\fbox{%
\begin{minipage}{0.98\linewidth}
     \vspace{-6pt}
     \begin{flalign*}
         &\ \ \  \text{Online model-based DEC:} &&\scalebox{0.85}{$\displaystyle\min_{\rho\in\Delta(\Pi)} \max_{M\in\calM} \E_{\pi\sim \rho}\Big[J_M(\pi_M) - J_M(\pi) - \gamma  \E_{\overline{M}\sim \mathsf{posterior}} [D^{\pi}(\overline{M}\| M)]\Big]$} &&& \hspace{-20pt} \\
         &\ \ \  \text{Online model-free DEC:}  &&\scalebox{0.85}{$\displaystyle\min_{\rho\in\Delta(\Pi)} \max_{M\in\calM} \E_{\pi\sim \rho} \Big[J_M(\pi_M) - J_M(\pi) - \gamma \E_{f\sim \mathsf{posterior}}[ D^{\pi}(f\| M)]\Big]$} &&& \hspace{-20pt}\\
         &\ \ \  \text{Offline model-free DEC:} &&\scalebox{0.85}{$\displaystyle\min_{\rho\in\Delta(\Pi)} \max_{M\in\calM_\conf} \E_{\pi\sim \rho} \Big[J_M(\pi_M) - J_M(\pi) - \gamma \max_{f\in\calF_\conf} D^{\pi_M}(f\| M)\Big]$} &&& \hspace{-20pt}
     \end{flalign*}
\end{minipage}
}
\vspace{-3pt}
\caption{Comparison of DEC objectives in online and offline settings for discrepancy measure $D$. }
\label{fig:dec-comparison}
\vspace{-4.3mm}
\end{figure}

In offline RL, the $Q^\star$ confidence set plays a similar role as the $Q^\star$ posterior in the online setting. Similarly, the learner seeks a policy that minimizes the sub-optimality against an adversarial model that remains in the confidence set. A key difference is that now this closeness is measured under the state-action distribution induced by each model’s \emph{own optimal policy}. First, unlike in online RL, the learner's output policy cannot generate new data and therefore cannot be used to test whether a candidate model is plausible. Thus, the discrepancy term should not be measured under the
learner's policy. Second, measuring discrepancy under $\pi_M$ penalizes candidate models whose own optimal policies are poorly supported by the offline data, which is exactly what is needed for coverage-adaptive guarantees under partial coverage. 


To measure the closeness between a model $M$ and a confidence function set under policy~$\pi_M$, we use $\max_{f\in\calF_\conf} D^{\pi_M}(f \| M)$ with standard discrepancy functions $D$ such as the Bellman error (will be clarified in \pref{sec:decision-error}). The intuition is that for the true model $M^\star$, $\max_{f\in\mathcal{F}_{\text{conf}}} D^{\pi^\star}(f\|M^\star)$ is expected to be small (usually of order $C^{\pi^\star}/n$), as any function $f$ that deviates significantly from $M^\star$ should have already been excluded from the confidence set $\calF_\conf$. Taking the maximum over $f\in\calF_\conf$ gives the strongest test of whether the candidate model is inconsistent with the confidence set. 
This gives to the third decision rule in \pref{fig:dec-comparison}.



%

This motivates the design of 
\algoff (\pref{alg:offset}).  We provide two versions of it, different in the way of trading the sub-optimality term and the discrepancy term. The offset version aligns with online DEC, and could give a better bound when the discrepancy term may approach zero. The ratio version, however, can obtain better dependence on $C^{\pi^\star}$ when it is unknown. We define the following complexity (Offline Robust DEC) that naturally arise from the algorithm: 
\begin{align}
   \odecoff_{\gamma}^D(\calF_\conf) &:= \min_{\rho\in\Delta(\Pi)} \max_{M\in\calM_\conf} \E_{\pi\sim \rho}\Big[J_M(\pi_M) - J_M(\pi) - \gamma \max_{f\in\calF_\conf} D^{\pi_M}(f\|M)\Big],    \label{eq: offdec 1}\\
   \odecratio^D(\calF_\conf) &:= \min_{\rho\in\Delta(\Pi)} \max_{M\in\calM_\conf}\frac{J_M(\pi_M) - \E_{\pi\sim \rho} [J_M(\pi)]}{\left(\max_{f\in\calF_\conf}D^{\pi_M}(f\|M)\right)^{1/2}}.   \label{eq: offdec 2}
\end{align}
It holds that $\odecoff_{\gamma}^D(\calF_\conf) \le \frac{4}{\gamma}\big(\odecratio^D(\calF_\conf)\big)^2$ (\pref{lem:offset-ratio conn}). With them, we have:

\begin{algorithm}[t]
    \caption{Offline Robust Estimation-to-Decision (\algoff)}
    \label{alg:offset}
    \textbf{Input:} Confidence set $\calF_{\mathrm{conf}}$, divergence measure $D^\pi(f\|M)$, parameter $\gamma$ for the offset version.   \\
    Define $\calM_\conf = \{M:~ Q_M^\star\in\calF_\conf\}$ and compute
    \begin{align*}
        \rhohat = 
        \begin{cases}
        \displaystyle\argmin_{\rho\in\Delta(\Pi)} \max_{M\in\calM_\conf} \E_{\pi\sim \rho}\left[J_M(\pi_M) - J_M(\pi) - \gamma \max_{f\in\calF_\conf} D^{\pi_M}(f\|M)\right]  &\text{(offset version)}\\[5pt]
        \displaystyle\argmin_{\rho\in\Delta(\Pi)} \max_{M\in\calM_\conf} \frac{J_M(\pi_M) - \E_{\pi\sim \rho} [J_M(\pi)]}{\left(\max_{f\in\calF_\conf}D^{\pi_M}(f\|M)\right)^{1/2}}.  &\text{(ratio version)} 
        \end{cases}
    \end{align*}
    \textbf{Output:} mixture policy $\hat{\pi}\sim \rhohat$. 
\end{algorithm}

\begin{theorem}\label{thm: main}
If $Q^\star\in\calF_\conf$, 
    then \algoff (\pref{alg:offset}) ensures 
    \begin{alignat*}{3}
\makebox[2.8cm][l]{\textup{(offset version)}}\;&
J(\pi^\star) - \E[J(\hat{\pi})]\;&\leq\;&  \textstyle 
\odecoff_\gamma^D(\calF_\conf)
+ \gamma \max_{f\in\calF_\conf} D^{\pi^\star}(f\|M^\star), \\
\makebox[2.8cm][l]{\textup{(ratio version)}}\;&
J(\pi^\star) - \E[J(\hat{\pi})]\;&\leq\;& \textstyle 
\odecratio^D(\calF_\conf)
\big(\max_{f\in\calF_\conf} D^{\pi^\star}(f\|M^\star)\big)^{1/2}. 
\end{alignat*}
\end{theorem}

\paragraph{Comparison with value-centric pessimism.} Following a similar analysis as in \cite{chen2022offline}, we can also express the suboptimality bound achieved by 
value-centric pessimism (either \eqref{eq: model-centric} or \eqref{eq: value centr}) as \mbox{$J(\pi^\star) - J(\hat{\pi})\leq \pesratio^D(\calF_\conf) (\max_{f\in\calF_\conf}D^{\pi^\star}(f\|M^\star))^{1/2}$, with } 
\begin{align}
   \pesratio^D(\calF_\conf) &:= \max_{M\in\calM_\conf}\frac{J_M(\pi_M) - J_M(\pi_{\hat{f}})}{\big(D^{\pi_M}(\hat{f}\|M)\big)^{1/2}}, \quad \text{where\ } \hat{f}=\argmin_{f\in\calF_\conf}f(s_1) \label{eq: pesdec 2}
\end{align}
See \pref{thm: pessimism} in \pref{app: section 4 proof} for a derivation. Comparing \eqref{eq: offdec 2} and \eqref{eq: pesdec 2}, it is clear that
$\odecratio^D\leq \pesratio^D$, and the key improvement is that $\odecratio^D$ goes through additional optimization over $\rho$ and $f$. 
The following is a concrete and intuitive example for the gap between $\odecratio^D $ and $\pesratio^D$: 
\begin{example}\label{ex: naive2 }
Consider the case with a single state $s_1$ and three actions $x,y,z$. Let $\calF_\conf=\{f_x, f_y\}$ with $f_x(s_1,x)=1, f_x(s_1,y)=0, f_x(s_1, z) = 1-\Delta$ and $f_y(s_1,x)=0, f_y(s_1,y)=1, f_y(s_1, z) = 1-\Delta$ for some $0<\Delta\leq 0.01$. It can be verified that $\odecoff^D_\gamma(\calF_\conf)\leq \Delta-\gamma$, $\odecratio^D(\calF_\conf)\leq \Delta$, and $\pesratio^D(\calF_\conf)=1$ with $D^a(f\|M) = (f(s_1,a) - R_M(s_1,a))^2$.  
\end{example}
In this example, clearly, choosing $x$ or $y$ is risky;  choosing~$z$, while always being sub-optimal, is the safest choice, and that is exactly the decision that will be made by \algoff.

\section{Bounding the Decision Complexity $\odec$}
\label{sec:decision-error}
In this section, we provide examples where $\odec$ is bounded. As the lower bound in \pref{sec: it lower bound} shows, $Q^\star$-approximation with partial coverage is a challenging or even impossible regime,  especially when the optimal value function has a very small or zero value gap. Therefore, even with the more refined decision principle developed in \pref{sec:unify}, some form of gap conditions or assumptions on the uniqueness of the optimal action is unavoidable. 
Throughout this section, we consider the average Bellman error: $ \Dav^{\pi}(f\| M) = \big(\E_{(s,a) \sim d_M^{\pi}}\left[f(s,a) - R_M(s,a) - \E_{s' \sim P_M(\cdot|s,a)}\left[f(s')\right]\right]\big)^2$.
We denote $\odecoff^{\av}_\gamma$ for $\odecoff^{\Dav}_\gamma$ (similar for others). 
See \pref{app:decision} for proofs in this section.

\subsection{Unregularized Case ($\psi\equiv 0$)}
\label{sec:gap}

When $\psi\equiv 0$, we recover and improve the gap bound in \cite{chen2022offline}. Define 
\begin{definition}
    $\Delta_f \triangleq \min_{s} \big(f(s, \pi_f(s)) - \max_{a\neq \pi_f(s)} f(s, a)\big)$. 
\end{definition}
\begin{theorem}[Value gap]\label{thm: gap bound} Assume $Q^\star\in\calF_\conf$. We have 
\begin{itemize}[leftmargin=11pt, topsep=0pt]
   \setlength{\itemsep}{0pt}
  \setlength{\parskip}{0pt}
  \setlength{\parsep}{0pt}
  \item $ \odecoff_{\gamma}^\av(\calF_\conf) \le \frac{H^2}{2\gamma \Delta_{Q^\star}^2} + \gamma \max_{f\in\calF_\conf}\Dav^{\pi^\star}(f \| M^\star)$  
  \item $\odecratio^\av(\calF_\conf), \pesratio^\av(\calF_\conf)\le \frac{H}{\Delta_{\hat{f}}}$ where $\hat{f}=\argmin_{f\in\calF_\conf} f(s_1)$.
\end{itemize}
\vspace{-4pt}
\end{theorem}

A key advantage of $\odecoff_{\gamma}^\av(\calF_\conf)$ is that it adapts to the true value gap $\Delta_{Q^\star}$. For fixed $\gamma$, combining with \pref{thm: main} gives an algorithm whose suboptimality guarantee scales with $\Delta_{Q^\star}^{-2}$. In contrast, the suboptimality of algorithm in \cite{chen2022offline} (value-centric pessimism) scales with $\pesratio^\av(\calF_\conf)$, which further scales with $\Delta_{\hat{f}}^{-1}$. This factor is not directly controllable and may be much larger than $\Delta_{Q^\star}^{-2}$.




We notice that \cite{song2022hybrid} and \cite{uehara2023offline} raise concerns on the applicability of the gap assumption \citep{chen2022offline} in continuous action space. Our complexities poses no requirement for discrete actions or strict value gaps. In \pref{thm: curved boundary}, we provide a continuous-action example with bounded $\odec$. 

\begin{theorem}[Curved decision boundary]\label{thm: curved boundary}
    Consider the linear setting with given feature mapping $\phi: \calS\times \calA\to \mathbb{R}^d$. Assume 1) $\exists\theta^\star\in\mathbb{R}^d$ such that $Q^\star(s,a) = \phi(s,a)^\top \theta^\star$, 2) for any policy~$\pi$ and any action pair $a,b\in\calA$, $|Q^\pi(s,a)-Q^\pi(s,b)|\leq L\|\phi(s,a)-\phi(s,b)\|$, and 3) for all $s$, $\phi(s,\cdot)\subset\mathbb{R}^d$ is a set with the following property:  for some $\kappa>0, \beta\geq 1$ (denote $\pi^\star(s) = \argmax_a \phi(s,a)^\top \theta^\star$):  
    \begin{align}
        \forall a\qquad \quad (\phi(s,\pi^\star(s)) - \phi(s,a))^\top \theta^\star \geq \kappa \|\phi(s,\pi^\star(s))-\phi(s,a)\|^\beta.    \label{eq: curved} 
    \end{align}
    Then $\odecoff_{\gamma}^\av(\calF_\conf)\leq \order\big(\left(\gamma^{-1}\kappa^{-2}L^{2\beta}\right)^{\frac{1}{2\beta-1}}\big) + \gamma \max_{f\in\calF_\conf}\Dav^{\pi^\star}(f \| M^\star)$. 
\end{theorem}

\subsection{Regularized Case ($\psi\not\equiv 0$)}\label{sec: regularized MDP}
Besides value gaps (\pref{thm: gap bound}) and decision sets with curvature
(\pref{thm: curved boundary}), regularization provides another way to bound
the decision complexity. We consider regularizers satisfying
\pref{assum:psi_proper}.

\begin{assumption}\label{assum:psi_proper}
For every $s\in\calS$, $\psi(\cdot;s)$ is Legendre (\pref{def:legendre}).
There exist constants $C_1^\psi,C_2^\psi>0$ such that, for all
$s\in\calS$ and $f_1,f_2\in\calF$, writing $\pi_i=\pi_{f_i}$, $\breg_\psi(\pi_2,\pi_1;s)
\le C_1^\psi\breg_\psi(\pi_1,\pi_2;s)$ and $\KL\left(\pi_1(\cdot|s)\|\pi_2(\cdot|s)\right) \le C_2^\psi\breg_\psi(\pi_1,\pi_2;s)$.
\end{assumption}

For example, KL regularization
$\psi(p;s)=\alpha\KL\left(p\|\pi_\reff(\cdot|s)\right)$,
as studied by \cite{uehara2023offline}, satisfies \pref{assum:psi_proper}
with $C_1^\psi=1+4H/\alpha$ and $C_2^\psi=1/\alpha$.
Tsallis and log-barrier regularizers are also covered;
their constants and verification are provided in \pref{lem:psi-example}
in \pref{app:aux-lemma}.

\begin{theorem}[Regularization]\label{thm:regular-dec}
Let regularizer $\psi$ satisfy \pref{assum:psi_proper}.  Then $\odecoff_\gamma^\av(\calF_\conf) \le 36\gamma^{-1}(C_1^{\psi}(1 + H^3C^{\psi}_2))^2$ and  $\odecratio^\av(\calF_\conf), \pesratio^\av(\calF_\conf) \le  3C_1^{\psi}(1 + H^3C^{\psi}_2)$.
\end{theorem}

\pref{thm:regular-dec} together with $\max_{f\in\calF_\conf} D^{\pi^\star}_\av(f\|M^\star)\lesssim C^{\pi^\star}/n$ (\pref{lem: BC case} or \pref{lem: weight case}) implies $\sqrt{C^{\pi^\star}/n}$ suboptimality gap, which provides the first $1/\epsilon^2$ sample complexity for this setting. To achieve this, the key technique in proving \pref{thm:regular-dec} is a new \emph{second-order performance difference lemma} for regularized RL (\pref{lem: SO PDL}). In the unregularized case, performance gaps between $\pi^\star$ and $\pi_f$ are usually controlled by \emph{first-order} divergence such as $D_{\mathrm{TV}}(\pi^\star\| \pi_f)$. In contrast, regularization introduces curvature that allows a \emph{second-order} divergence. Specifically, we show $J(\pi^\star) - J(\pi_f) \lesssim \E_{s \sim d^{\pi^\star}}\left[\Breg_{\psi}(\pi_f, \pi^\star; s)\right]$ where the right-hand side scales quadratically with the policy difference. While \cite{uehara2023offline} also studies regularized MDP, they bound the performance gap under the unregularized reward, and only achieves $1/\epsilon^4$ sample complexity. In addition, we remove their $\poly(|\calA|)$ dependence and accommodate more general offline distributions. A detailed comparison with \cite{uehara2023offline} is provided in \pref{app:regularized-proof}.

\section{Bounding the Estimation Error}
\label{sec:estimation-error}
In this section, we present several constructions of the confidence set
$\calF_\conf$ that ensure $Q^\star\in\calF_\conf$ while controlling the
estimation error
$\max_{f\in\calF_\conf}\Dav^{\pi^\star}(f\|M^\star)$ in \pref{thm: main}.
For a dataset $\calD$ of $n$ tuples $(s,a,r,s')$, we write
$\E_{\calD}[g]
:=
\frac{1}{n}\sum_{(s,a,r,s')\in\calD} g(s,a,r,s')$
for any function $g$. See \pref{app:estimation-error} for proofs in this section.

\subsection{Bellman Completeness}
\label{sec:regular}
Bellman completeness (\pref{assum: Bellman complete}) is widely adopted in online RL \citep{jin2021bellman, xie2023role} and offline RL with full coverage \citep{chen2019information}, but has not been studied under partial coverage except for the linear case \citep{golowichrole}. The following guarantee is, however, standard in all these settings. 


\begin{lemma}\label{lem: BC case}
   Assume $\calF$ and $\calG$ satisfies \pref{assum: Bellman complete}. Let 
   \begin{itemize}[leftmargin=11pt, topsep=0pt]
   \setlength{\itemsep}{0pt}
  \setlength{\parskip}{0pt}
  \setlength{\parsep}{0pt}
      \item $L_\bc(g, f) = \E_{(s,a,r,s')\sim\calD}\big[\left(g(s,a) - r - f(s')\right)^2\big]$ 
      \item $ \calF_\conf= \big\{f\in\calF:~ L_\bc(f,f) - \min_{g\in\calG} L_\bc(g,f)\leq \epsilon_\stat\big\}$, where $\epsilon_\stat=\frac{2H^2\log\left(|\calF||\calG|/\delta\right)}{n}$  
   \end{itemize}
    Then with probability at least $1-\delta$, $Q^\star\in\calF_\conf$ and $\max_{f\in\calF_\conf} \Dav^{\pi^\star}(f\|M^\star)\leq \order(HC^{\pi^\star}\epsilon_\stat)$. 
\end{lemma}
We note that $Q^\star$-realizability under regularization, partial coverage, and Bellman completeness exactly matches the setting of Conservative $Q$-Learning (CQL). However, existing analyses of CQL are restricted to the tabular setting \citep{kumar2020conservative}. In \pref{app: CQL}, we provide the first sample-complexity guarantee for CQL under general function approximation.

\subsection{Weight Realizability}
\label{sec:density-confidence}

Weight realizability (or density-ratio realizability) has been adopted in several prior studies in offline RL with partial coverage \citep{zhan2022offline, chen2022offline, rashidinejadoptimal, ozdaglar2023revisiting}. Its guarantees are below. 
\begin{assumption}[Weight realizability]\label{assum: weight real}
Define $w^{\pi}(s,a) = \frac{d^{\pi}(s,a)}{H\mu(s,a)}$. Assume $w^{\pi^\star} \in \mathcal{W}$ for a given function set $\calW$. Additionally, for any $w \in \mathcal{W}$, we have $\|w\|_{\infty} \le B_{\mathcal{W}}$. 
\label{assum:weight-real}
\end{assumption}
\begin{lemma}\label{lem: weight case}
   Assume $\calF$ and $\calW$ satisfies \pref{assum: weight real}. Let 
   \begin{itemize}[leftmargin=11pt, topsep=0pt]
   \setlength{\itemsep}{0pt}
  \setlength{\parskip}{0pt}
  \setlength{\parsep}{0pt}
      \item $L_\dr(w,f) = \left|\E_{(s,a,r,s')\sim\calD}\big[w(s,a)(f(s,a)-r-f(s'))\big]\right|$.  
      \item $ \calF_\conf= \big\{f\in\calF:~ \max_{w\in\calW} L_\dr(w,f) \leq \epsilon_\stat\big\}$, where $\epsilon_\stat=B_{\mathcal{W}}H\sqrt{\frac{2\log\left(|\calF||\calW|/\delta\right)}{n}}$.  
   \end{itemize}
    Then with probability at least $1-\delta$, $Q^\star\in\calF_\conf$ and $\max_{f\in\calF_\conf} \Dav^{\pi^\star}(f\|M^\star)\leq \order(H^2\epsilon_\stat^2)$. 
\end{lemma}

\subsection{Low $Q$-Bellman Rank}
\label{sec:low-q}
Bellman rank is established for online RL \citep{jiang2017contextual, du2021bilinear, jin2021bellman} but remains unexplored in offline RL except for special cases. Its absence is not for no reason---as we will show, under standard coverage and data format, efficient offline learning in low-Bellman-rank MDPs is impossible. To bypass this barrier, we introduce \emph{double policy sampling} (\pref{assum: double sample}) and \emph{policy feature coverage} (\pref{def: policy cover}). Then, in \pref{lem: first claim} and  \pref{lem: second claim} (\pref{app: jia lower }) we show  that dropping either assumption makes polynomial sample complexity impossible. 
\begin{assumption}[$Q$-Bellman rank]\label{assum: bellman rank} There exist mappings $X: \Pi\to \mathbb{R}^d$ and $W: \calF\to \mathbb{R}^d$ such that for any $\pi\in\Pi$ and $f\in\calF$,   $\E_{(s,a)\sim d^{\pi}}\left[f(s,a) - R(s,a) - \E_{s'\sim P(\cdot|s,a)}[f(s')]\right] = \inner{X(\pi), W(f)}$. 
\end{assumption}
\begin{assumption}[Double policy sampling]\label{assum: double sample}
    Assume that each offline data sample is generated as the following. First, sample a policy $\pi\sim \mubar\in\Delta(\Pi)$. Then generate \emph{two independent $(s,a,r,s')$ tuples} from this policy: $(s_\oo,a_\oo)\sim d^{\pi}$, $r_\oo\sim R(s_\oo, a_\oo)$, $s'_\oo\sim P(\cdot|s_\oo,a_\oo)$, and $(s_\xx,a_\xx)\sim d^{\pi}$, $r_\xx\sim R(s_\xx, a_\xx)$, $s'_\xx\sim P(\cdot|s_\xx,a_\xx)$.   They constitute a data sample: \mbox{$(s_\oo, a_\oo, r_\oo, s'_\oo, s_{\xx}, a_{\xx}, r_{\xx}, s'_{\xx})\in\calD$.} 
\end{assumption}
\begin{definition}[Policy feature coverage]\label{def: policy cover}
    Assume $\calF$ has low Bellman rank (\pref{assum: bellman rank}), and let $\mubar$ be defined in \pref{assum: double sample}. Define $ \Cbar^{\pi^\star} = X(\pi^\star)^\top \Sigma_{\mubar}^{\dagger}X(\pi^\star)$, where $\Sigma_{\mubar} = \E_{\pi\sim \mubar}[X(\pi)X(\pi)^\top]$. 
\end{definition}

\begin{lemma}\label{lem: Bellman rank case}
   Let \pref{assum: bellman rank} and \pref{assum: double sample} hold. Also, let  
   \begin{itemize}[leftmargin=11pt, topsep=0pt]
   \setlength{\itemsep}{0pt}
  \setlength{\parskip}{0pt}
  \setlength{\parsep}{0pt}
      \item $L_\brr(f) = \E_{(s_\oo,a_\oo,r_\oo,s'_\oo,s_\xx,a_\xx,r_\xx,s_\xx')\sim\calD}\big[\left(f(s_\oo, a_\oo) - r_\oo - f(s'_\oo)\right)\left(f(s_\xx, a_\xx) - r_\xx - f(s'_\xx)\right)\big]$.  
      \item $ \calF_\conf= \big\{f\in\calF:~  L_\brr(f) \leq \epsilon_\stat\big\}$, where $\epsilon_\stat=H^2\sqrt{\frac{\log(2|\calF|/\delta)}{2n}}$.
   \end{itemize}
   Then with probability at least $1-\delta$, $Q^\star\in\calF_\conf$ and $\max_{f\in\calF_\conf} \Dav^{\pi^\star}(f\|M^\star)\leq \order(\Cbar^{\pi^\star}\epsilon_\stat)$. 
\end{lemma}

\clearpage

\bibliography{ref}
\bibliographystyle{alpha}

\clearpage
\appendix
\noindent{\huge\textbf{Appendices}}\\[10pt]
{
\startcontents[section]
\printcontents[section]{l}{1}{\setcounter{tocdepth}{2}}
}

\clearpage

\section{Related Work}
\label{app:related works}
\paragraph{Offline RL Theory.} The theoretical study of offline reinforcement learning has largely focused on identifying the structural properties of the function class and the coverage conditions on offline data that enable sample-efficient learning. These two aspects exhibit a trade-off: stronger assumptions on the offline data can often compensate for weaker assumptions on the function class, and vice versa. For model-based offline RL, \cite{uehara2021pessimistic} provide a nearly complete characterization of learnability, showing that a realizable model class together with partial coverage is sufficient for sample-efficient learning. For model-free offline RL, however, the theoretical picture remains fragmented. Under the weakest assumption of $Q^\star$ realizability, the offline data must satisfy a strong push-forward coverage condition to enable sample-efficient learning \citep{xie2021batch}. When this data assumption is relaxed to all-policy coverage, it has been shown that even $Q^\pi$ realizability for every policy $\pi$ does not guarantee sample-efficient learning \citep{wang2020statistical, foster2022offline, jia2024offline}. This motivates the need for additional structural assumptions on the function class. A common such assumption is Bellman completeness. Under Bellman completeness, \cite{chen2019information} show that all-policy coverage together with $Q^\star$ realizability is sufficient for learnability, while \cite{xie2021bellman} show that single-policy coverage together with $Q^\pi$ realizability is sufficient. In the special case of linear function classes, \cite{golowichrole} show that $Q^\star$ realizability, Bellman completeness, and partial coverage together suffice for learning. However, the fundamental role of Bellman completeness beyond the linear setting remains elusive. Our work contributes to this direction by showing that, for general function classes, $Q^\star$ realizability, Bellman completeness, and partial coverage are not sufficient for sample-efficient learning. While the work of \cite{song2022hybrid} makes a similar claim, their lower bound is algorithm-specific and can be handled by existing algorithms. In contrast, our lower bound is information-theoretic. This provides a sharp separation from previous positive results. 

Beyond Bellman completeness, another line of work assumes additional realizability of a Lagrange multiplier or density-ratio-type function in addition to $Q^\star$ realizability \citep{zhan2022offline, chen2022offline, rashidinejadoptimal, ozdaglar2023revisiting, uehara2023offline}. Under such assumptions, sample-efficient learning can be obtained when the optimal value has a sufficiently large gap over suboptimal policies \citep{chen2022offline}, or when the value objective includes an additional regularizer \citep{zhan2022offline, rashidinejadoptimal, ozdaglar2023revisiting, uehara2023offline}. As discussed in \pref{sec:overview}, our analytical framework also recovers and improves several results in this line of work. Results discussed above are summarized in \pref{tab:offline-rl-theory}. For broader discussions of offline RL algorithms and theory, we refer the reader to \cite{che2025tutorial} for a comprehensive review and comparison.

\begin{table*}[t]
\centering
\scriptsize
\setlength{\tabcolsep}{4pt}
\renewcommand{\arraystretch}{1.18}
\begin{tabular}{
@{}
>{\raggedright\arraybackslash}p{0.12\textwidth}
>{\raggedright\arraybackslash}p{0.16\textwidth}
>{\raggedright\arraybackslash}p{0.18\textwidth}
>{\raggedright\arraybackslash}p{0.32\textwidth}
>{\centering\arraybackslash}p{0.1\textwidth}
@{}
}
\toprule
\textbf{Reference}
& \textbf{Value realizability}
& \textbf{Coverage}
& \textbf{Additional structure}
& \textbf{Learnability} \\
\midrule

\multicolumn{5}{@{}l}{\textbf{\textit{Only with value  realizability and coverage}}} \\
\midrule

\citep{xie2021batch}
& $Q^\star$ realizability
& Push-forward coverage
& None
& \textcolor{green!50!black}{Yes} \\

\citep{foster2022offline}
& $Q^\pi$ realizability
& All-policy coverage
& None
& \textcolor{red}{No} \\

\addlinespace[2pt]
\midrule
\multicolumn{5}{@{}l}{\textbf{\textit{Bellman-completeness-based results}}} \\
\midrule

\citep{chen2019information}
& $Q^\star$ realizability
& All-policy coverage
& Bellman completeness
& \textcolor{green!50!black}{Yes} \\

\citep{xie2021bellman}
& $Q^\pi$ realizability
& Partial coverage
& Bellman completeness
& \textcolor{green!50!black}{Yes} \\

\citep{golowichrole}
& $Q^\star$ realizability
& Partial coverage
& Bellman completeness; linear function class
& \textcolor{green!50!black}{Yes} \\

\textbf{Ours}
& $Q^\star$ realizability
& Partial coverage
& Bellman completeness; general function class
& \textcolor{red}{No} \\

\addlinespace[2pt]
\midrule
\multicolumn{5}{@{}l}{\textbf{\textit{Weighted-realizability-based results}}} \\
\midrule

\citep{chen2022offline}\textsuperscript{$\dagger$}
& $Q^\star$ realizability
& Partial coverage
& Weighted realizability; large value gap
& \textcolor{green!50!black}{Yes} \\

\citep{uehara2023offline}\textsuperscript{$\ddagger$}
& $Q^\star$ realizability
& Partial coverage
& Weighted realizability; regularized value
& \textcolor{green!50!black}{Yes} \\

\bottomrule
\end{tabular}

\vspace{1mm}
\begin{minipage}{0.96\textwidth}
\scriptsize
\textsuperscript{$\dagger$} Our method can handle unknown value gaps without online access.
\quad
\textsuperscript{$\ddagger$} Our method improves the sample-complexity bound.
\end{minipage}

\caption{Summary of model-free offline RL learnability with $Q^\star$ or $Q^\pi$ realizability under different structural assumptions on function class and offline data.}
\label{tab:offline-rl-theory}
\end{table*}

\paragraph{Bridge Online and Offline RL. } Online and offline RL have traditionally been studied as separate areas, but recent works suggest that their structural assumptions can be viewed through a unified lens \citep{xie2023role, amortila2024harnessing, mhammedi2024power, krishnamurthy2025role}. For example, coverage assumptions in offline RL have natural online counterparts,  referred to as coverability conditions \citep{xie2023role}, which measure whether the MDP can be effectively covered by some exploration distribution. Thus, coverability captures the intrinsic difficulty of exploration in online RL \citep{xie2023role}. Under small coverability and $Q^\star$ realizability, \cite{xie2023role} and \cite{amortila2024harnessing} show that sample-efficient online RL is possible under Bellman completeness and density-ratio realizability, respectively. These results illustrate how ideas from offline RL can inform online RL.

In this work, we pursue the complementary direction: using insights from online RL to better understand offline RL. In online RL, although optimism is a broadly useful principle, it can be arbitrarily suboptimal compared with information-directed methods on certain instances \citep{lattimore2017end}. We establish a mirrored phenomenon in offline RL, showing that value-based pessimism, despite its general applicability, can likewise be highly suboptimal in certain problems (\pref{sec:unify}).

We further provide, to our knowledge, the first characterization of offline learnability for low-Bellman-rank MDPs without Bellman completeness \citep{jiang2017contextual, du2021bilinear, jin2021bellman}. Low Bellman rank is a canonical structural condition in online RL, but its offline counterpart has been largely unexplored beyond special cases. Our analysis identifies \emph{double policy sampling} and \emph{policy feature coverage} as crucial requirements for offline learning in low-Bellman-rank MDPs (\pref{sec:low-q}), highlighting structural restrictions that are less visible in the online setting.


\newpage

\section{Proofs in \pref{sec: it lower bound}} \label{app: lower bound construct}

\subsection{MDP construction}\label{app: MDP consturction}
We construct four families of finite-horizon MDPs as illustrated in \pref{fig:main}.
In all families, the episode starts at state $s_1$, where the agent chooses between two actions ${u,v}$. Taking either action yields a Bernoulli reward (specified on the corresponding edge in the figure) and transitions to a middle-layer group of states, either $\WA$ or $\WB$. The group sizes satisfy $|\WA|=|\WB|=m$, and conditional on entering a group, the next state is drawn uniformly from that group. 

From any state $w\in \WA\cup\WB$, there is a single available action $a$, which yields zero reward and transitions deterministically to $\sA$ if $w\in\WA$ and to $\sB$ if $w\in\WB$. Finally, from $\sA$ or $\sB$, the agent chooses among three actions ${x,y,z}$, receives the deterministic terminal reward shown in \pref{fig:main}, and the episode ends. One terminal reward equals $-2$, which lies outside the reward range specified in \pref{sec:pre}; this can be removed by a simple affine rescaling and does not affect the argument.

Each family consists of a collection of MDPs indexed by hidden assignments. Let $W$ be a set of $2m$ abstract middle-layer states. Each MDP instance corresponds to an assignment
$\phi: W\to\{\WA,\WB\}$ where exactly 
$m$ states are assigned to $\WA$ and the remaining $m$ states are assigned to $\WB$. 
There are $\binom{2m}{m}$ such assignments. For each family $\calM_{u,x}$, $\calM_{u,y}$, $\calM_{v,x}$, and $\calM_{v,y}$, the transition and reward structure is fixed as in \pref{fig:main}, while the assignment $\phi$ varies over all such possibilities. Consequently,  $|\calM_{u,x}|=|\calM_{u,y}|=|\calM_{v,x}|=|\calM_{v,y}|=\binom{2m}{m}$. We write $M^\phi_{u,x}$ for the MDP in $\calM_{u,x}$ corresponding to assignment $\phi$, and similarly for the other families.

Define $\calM_u=\calM_{u,x}\cup\calM_{u,y}$, $\calM_v=\calM_{v,x}\cup\calM_{v,y}$, and $\calM=\calM_u\cup\calM_v$.
The ground-truth MDP $M^\star$ is drawn uniformly from $\calM$.

\subsection{Function set construction} \label{app: function construction}
We define the function class $\calF=\{(f_1, f_2, f_3), (f_1, f_2, g_3), (g_1, f_2, f_3), (g_1, f_2, g_3)\}$ where the components are defined as
\begin{gather*}
   \begin{cases}
       f_3(\sA, x)=1 \\
       f_3(\sA, y)=-2 \\
       f_3(\sA, z)=0 \\
       f_3(\sB, x)=0 \\
       f_3(\sB, y)=-2 \\
       f_3(\sB, z)=1 
   \end{cases} 
   \qquad 
   \begin{cases}
       g_3(\sA, x)=-2 \\
       g_3(\sA, y)=1 \\
       g_3(\sA, z)=0 \\
       g_3(\sB, x)=-2 \\
       g_3(\sB, y)=0 \\
       g_3(\sB, z)=1 
   \end{cases} 
   \qquad
       f_2(w, a) = 1 \ \ \forall w\in \WA\cup \WB
   \quad 
   \\[10pt]
   \begin{cases}
       f_1(s_1, u)=\tfrac{3}{2} \\
       f_1(s_1, v)=\tfrac{3}{2}+\Delta 
   \end{cases} 
   \qquad 
   \begin{cases}
       g_1(s_1, u)=\tfrac{3}{2}+\Delta \\
       g_1(s_1, v)=\tfrac{3}{2} 
   \end{cases} 
\end{gather*}
\noindent\underline{Checking Bellman completeness}\ \ \  
We verify that $\calF$ is Bellman complete.
Let $M\in\calM$ be arbitrary. We have for any $Q_3\in\{f_3, g_3\}$
\begin{align*}
    \forall w\in \WA\cup \WB, \quad  (\calT_M Q_3)(w,a) = R_M(w,a) + \E_{s'\sim P_M(\cdot|w,a)}\left[\max_{a'} Q_3(s', a')\right] = 0 + 1 = 1.  
\end{align*}
Thus, $(\calT_M Q_3) = f_2$. Next, 
\begin{align*}
    \forall a_1\in\{u,v\}, \quad (\calT_M f_2)(s_1, a_1) = R_M(s_1,a_1) + \E_{s'\sim P_M(\cdot|s_1,a_1)}\left[\max_{a'} f_2(s',a')\right] = R_M(s_1,a_1) + 1. 
\end{align*}
Depending on $M\in\calM_u$ or $M\in\calM_v$, $R_M+1$ equals either $f_1$ or $g_1$. Therefore, $(\calT_M f_2)\in \{f_1, g_1\}$. 
This verifies that $\calF$ is Bellman complete. 
\subsection{Dataset construction} \label{app: dataset constriction}
We assume the offline dataset $\calD=(\calD_1,\calD_2,\calD_3)$, where: 
\begin{itemize}[leftmargin=11pt, topsep=0pt]
    \item $\calD_1$ consists of $n$ tuples $(s_1, a_1^{(i)}, r_1^{(i)}, w_1^{(i)})_{i=1}^n$ from layer~1, where $a_1^{(i)}\sim \unif\{u,v\}$; then $r_1^{(i)}\in\{0,1\}$ and $w_1^{(i)}\in W$ are drawn according to $M^\star$ given $(s_1, a_1^{(i)})$.  
    \item $\calD_2$ consists of $n$ tuples $(w_2^{(i)}, a, 0, s_2^{(i)})_{i=1}^n$  from  layer~2, where  $w_2^{(i)}\sim \unif(W)$, and $s_2^{(i)}$ is drawn according to $M^\star$ given $(w_2^{(i)}, a)$. 
    \item $\calD_3$ consists of $n$ tuples $(s_3^{(i)}, z, r_3^{(i)})_{i=1}^n$ from layer~3, where $s_3^{(i)}\sim \unif\{\sA, \sB\}$. After taking action $z$, $r_3^{(i)}\in \{0,1\}$ is drawn according to $M^\star$ given $(s_3^{(i)}, z)$. 
\end{itemize}
\noindent\underline{Checking coverage over the optimal policy $\pi_{M}$}\ \ \  For any model $M\in\calM$ in \pref{fig:main}, the unique optimal policy $\pi_M$ always follows the lowest branch and achieves value $\frac{3}{2}+\Delta$. The above data distribution has coverage $C^{\pi_M}=2$ with respect to $\pi_M$.


\subsection{Lower bound proof}\label{app: lower bound proof}
For each $M\in\calM$, let \(\mathbb{P}_M\) denote the law of the dataset $\calD$ and $\E_M$ the expectation under $\PP_M$. Let \(\pi_M\) denote the optimal policy in $M$. 
Define the mixture laws of the dataset: 
\begin{alignat}{2}
\mathbb{P}_{u,x} &= \frac{1}{|\mathcal{M}_{u,x}|}\sum_{M\in\mathcal{M}_{u,x}}\mathbb{P}_M, 
&\quad \mathbb{P}_{u,y} &= \frac{1}{|\mathcal{M}_{u,y}|}\sum_{M\in\mathcal{M}_{u,y}}\mathbb{P}_M, \nonumber \\
\mathbb{P}_{v,x} &= \frac{1}{|\mathcal{M}_{v,x}|}\sum_{M\in\mathcal{M}_{v,x}}\mathbb{P}_M, 
&\quad \mathbb{P}_{v,y} &= \frac{1}{|\mathcal{M}_{v,y}|}\sum_{M\in\mathcal{M}_{v,y}}\mathbb{P}_M,\nonumber  \\
\mathbb{P}_u &= \frac{1}{|\mathcal{M}_u|}\sum_{M\in\mathcal{M}_u}\mathbb{P}_M,
&\qquad \mathbb{P}_v &= \frac{1}{|\mathcal{M}_v|}\sum_{M\in\mathcal{M}_v}\mathbb{P}_M, \nonumber \\
\PP &= \frac{1}{|\calM|}\sum_{M\in\calM} \PP_M,   \label{eq: laws def} 
\end{alignat}
and let $\E_{u,x}, \E_{u,y}, \E_{v,x}, \E_{v,y}, \E_u, \E_v, \E$ denote the corresponding expectation. Furthermore, let $J_M(\pi)$ denote the value of $\pi$ under model $M$. 
  By the construction in \pref{fig:main}, 
    \begin{align}
       &J_M(\hat{\pi}) = \label{eq: V_M(pi)} \\
       &\begin{cases}  \hat{\pi}(u|s_1)\big(\frac{1}{2} + \hat{\pi}(x|\sA) - 2\hat{\pi}(y|\sA)\big) + \hat{\pi}(v|s_1)\big(\tfrac{1}{2}+\Delta+\hat{\pi}(z|\sB) - 2\hat{\pi}(y|\sB)\big)\triangleq J_{u,x}(\hat{\pi})  &\text{if\ } M\in \mathcal{M}_{u,x}   \\
        \hat{\pi}(u|s_1)\big(\frac{1}{2} + \hat{\pi}(y|\sA) - 2\hat{\pi}(x|\sA)\big) + \hat{\pi}(v|s_1)\big(\tfrac{1}{2}+\Delta+\hat{\pi}(z|\sB) - 2\hat{\pi}(x|\sB)\big) \triangleq J_{u,y}(\hat{\pi})  &\text{if\ } M\in \mathcal{M}_{u,y}  \\
       \hat{\pi}(v|s_1)\big(\frac{1}{2} + \hat{\pi}(x|\sA) - 2\hat{\pi}(y|\sA)\big) + \hat{\pi}(u|s_1)\big(\tfrac{1}{2}+\Delta+\hat{\pi}(z|\sB) - 2\hat{\pi}(y|\sB)\big) \triangleq J_{v,x}(\hat{\pi})   &\text{if\ } M\in \mathcal{M}_{v,x} \\
       \hat{\pi}(v|s_1)\big(\frac{1}{2} + \hat{\pi}(y|\sA) - 2\hat{\pi}(x|\sA)\big) + \hat{\pi}(u|s_1)\big(\tfrac{1}{2}+\Delta+\hat{\pi}(z|\sB) - 2\hat{\pi}(x|\sB)\big)  \triangleq J_{v,y}(\hat{\pi})  &\text{if\ } M\in \mathcal{M}_{v,y}
       \end{cases}
       \nonumber 
    \end{align}

\begin{lemma}[Action $z$ is optimal] \label{lem: z is better}
Let an algorithm output \(\hat\pi\) (possibly stochastic). Define \(\hat\pi^{z}\) to match \(\hat\pi\) everywhere except \(\hat\pi^{z}(\sA)=\hat\pi^{z}(\sB)=z\) deterministically. Then $\mathbb E\!\left[J_M(\hat\pi)\right]
\leq 
\mathbb E\!\left[J_M(\hat\pi^{z})\right]$. 
\end{lemma}

\begin{proof} 
    For any assignment $\phi$, the models $M_{u,x}^\phi$ and $M_{u,y}^\phi$ induce identical data distributions, since they differ only in terminal rewards for actions $x$ and $y$, which never appear in the dataset. The same holds for $M_{v,x}^\phi$ and $M_{v,y}^\phi$. Consequently, 
    \begin{align}
        \PP_{u,x}=\PP_{u,y} \quad \text{and} \quad \PP_{v,x}=\PP_{v,y}. \label{eq: same marginal}  
    \end{align}
    \allowdisplaybreaks
    The expected value of policy $\hat{\pi}$ is 
    \begin{align*}
    &\E[J_M(\hat{\pi}) ] \\
        &= \frac{1}{4} \E_{M\sim \unif(\calM_{u,x})} \E_M[J_M(\hat{\pi})] + \frac{1}{4} \E_{M\sim \unif(\calM_{u,y})} \E_M[J_M(\hat{\pi})] \\
        &\qquad +\frac{1}{4} \E_{M\sim \unif(\calM_{v,x})} \E_M[J_M(\hat{\pi})] + \frac{1}{4} \E_{M\sim \unif(\calM_{v,y})} \E_M[J_M(\hat{\pi})] \tag{$M$ is chosen uniformly from $\calM$}\\
        &= \frac{1}{4} \E_{M\sim \unif(\calM_{u,x})} \E_M[J_{u,x}(\hat{\pi})] + \frac{1}{4} \E_{M\sim \unif(\calM_{u,y})} \E_M[J_{u,y}(\hat{\pi})] \\
        &\qquad +\frac{1}{4} \E_{M\sim \unif(\calM_{v,x})} \E_M[J_{v,x}(\hat{\pi})] + \frac{1}{4} \E_{M\sim \unif(\calM_{v,y})} \E_M[J_{v,y}(\hat{\pi})] \tag{by \pref{eq: V_M(pi)}}\\
        &= \frac{1}{4} \E_{u,x}[J_{u,x}(\hat{\pi})] + \frac{1}{4} \E_{u,y}[J_{u,y}(\hat{\pi})] + \frac{1}{4} \E_{v,x}[J_{v,x}(\hat{\pi})] + \frac{1}{4} \E_{v,y}[J_{v,y}(\hat{\pi})] \\
        &= \frac{1}{4} \E_{u,x}[J_{u,x}(\hat{\pi})+J_{u,y}(\hat{\pi})]  + \frac{1}{4} \E_{v,x}[J_{v,x}(\hat{\pi})+J_{v,y}(\hat{\pi})]  \tag{by \pref{eq: same marginal}} \\
        &= \frac{1}{4}\E_{u,x}\left[ \hat{\pi}(u|s_1)\big(1-\hat{\pi}(x|\sA)-\hat{\pi}(y|\sA)\big) +\hat{\pi}(v|s_1)\big(1+2\Delta + 2\hat{\pi}(z|\sB) - 2\hat{\pi}(x|\sB)-2\hat{\pi}(y|\sB)\big)  \right] \\
        &\qquad + \frac{1}{4}\E_{v,x}\left[ \hat{\pi}(v|s_1)\big(1-\hat{\pi}(x|\sA)-\hat{\pi}(y|\sA)\big) +\hat{\pi}(u|s_1)\big(1+2\Delta + 2\hat{\pi}(z|\sB) - 2\hat{\pi}(x|\sB)-2\hat{\pi}(y|\sB)\big)  \right] \tag{by \pref{eq: V_M(pi)}} \\
        &= \frac{1}{4}\E_{u,x}\left[ \hat{\pi}(u|s_1)\hat{\pi}(z|\sA) +\hat{\pi}(v|s_1)\big(-1+2\Delta + 4\hat{\pi}(z|\sB)\big)  \right] \\
        &\qquad + \frac{1}{4}\E_{v,x}\left[ \hat{\pi}(v|s_1)\hat{\pi}(z|\sA) +\hat{\pi}(u|s_1)\big(-1+2\Delta + 4\hat{\pi}(z|\sB)\big)  \right]. \tag{$\sum_{a\in\{x,y,z\}} \hat{\pi}(a|s)=1$ for $s=\sA$ and $\sB$}
    \end{align*}
    As $\hat{\pi}^z(\cdot|s_1) = \hat{\pi}(\cdot|s_1)$ and $\hat{\pi}^z(z|\sA) \geq \hat{\pi}(z|\sA)$ and $\hat{\pi}^z(z|\sB) \geq \hat{\pi}(z|\sB)$, by the last expression, we have 
    \begin{align*}
        \E[J_M(\hat{\pi}) ] \leq \E[J_M(\hat{\pi}^z) ]. 
    \end{align*}
\end{proof}

\begin{lemma}[Reduction to total variation]\label{lem: reduction to TV}
For any offline algorithm,
\[
\mathbb E\!\left[J_M(\pi_M)-J_M(\hat\pi)\right]
\;\ge\;
\frac12(1+\Delta)\left(1-D_{\mathrm{TV}}(\mathbb{P}_u,\mathbb{P}_v)\right).
\]
\end{lemma}

\begin{proof}
By \pref{lem: z is better} we may assume \(\hat\pi(\sA)=\hat\pi(\sB)=z\). Let $\beta=\hat\pi(u\mid s_1)\in[0,1]$. A direct calculation from \pref{eq: V_M(pi)} yields 
\begin{align*}
   J_M(\hat{\pi})=
   \begin{cases}
   \tfrac{1}{2}\beta + (\tfrac{3}{2}+\Delta)(1-\beta)  & \text{if } M\in\calM_u, \\
   \tfrac{1}{2}(1-\beta) + (\tfrac{3}{2}+\Delta)\beta & \text{if } M\in\calM_v.
   \end{cases}
\end{align*}
Also, $J_M(\pi_M) = \frac{3}{2}+\Delta$ for any $M$. 
Thus, 
\begin{align*}
&\E[J_M(\pi_M)-J_M(\hat\pi)] \\
&= \left(\frac{3}{2}+\Delta\right) - \frac{1}{2}\E_u \left[ \frac{1}{2}\beta + \left(\frac{3}{2}+\Delta\right)(1-\beta)\right]
 - \frac{1}{2}\E_v \left[ \frac{1}{2}(1-\beta) + \left(\frac{3}{2}+\Delta\right)\beta\right]\\
 &= \left(\frac{3}{2}+\Delta\right) - \frac{1}{2}\left(\frac{3}{2}+\Delta\right) - \frac{1}{2}\cdot\frac{1}{2} + \frac{1}{2}\left(\left(\frac{3}{2}+\Delta\right)-\frac{1}{2}\right) \left(\E_u[\beta] - \E_v[\beta]\right) \\
 &= \frac{1}{2}(1+\Delta) \left(1 + (\E_u[\beta] - \E_v[\beta]) \right) \\ 
 &\geq \frac{1}{2}(1+\Delta) \left(1 - D_{\mathrm{TV}}(\PP_u, \PP_v)\right), 
\end{align*}
where we use that \(\beta\in[0,1]\), \(\big|\E_u[\beta]-\E_v[\beta]\big|\le D_{\mathrm{TV}}(\mathbb{P}_u,\mathbb{P}_v)\) in the last inequality.
\end{proof}

\begin{lemma}[TV bound]\label{lem: TV bound}
If
$m\geq 4n^2$, then
$
D_{\mathrm{TV}}(\mathbb{P}_u,\mathbb{P}_v)\;\le\;\frac{6n^2}{m} + 4\sqrt{n}\Delta 
$. 
\end{lemma}

\begin{proof}
\ \\[-20pt]
\paragraph{Step 1: remove the $\Delta$-shift via a comparison family.}  
Define an auxiliary collection of families $\calM'_{u,x},\calM'_{u,y},\calM'_{v,x},\calM'_{v,y}$ that are identical to those in \pref{fig:main}, except that the rewards on actions $u$ and $v$ at $s_1$ are both $\Ber(\tfrac12)$ (i.e., we set $\Delta=0$ at the first layer). Define $\calM'_u,\calM'_v,\calM'$ and the corresponding mixture laws $\PP'_u,\PP'_v,\PP'$ analogously to \pref{eq: laws def}. 

By the triangle inequality and Pinsker’s inequality,
\begin{align*}
    D_{\mathrm{TV}}(\mathbb{P}_u,\mathbb{P}_v) &\leq D_{\mathrm{TV}}(\mathbb{P}_u',\mathbb{P}_v') + D_{\mathrm{TV}}(\mathbb{P}_u,\mathbb{P}_u') + D_{\mathrm{TV}}(\mathbb{P}_v,\mathbb{P}_v') \\
    &\leq D_{\mathrm{TV}}(\mathbb{P}_u',\mathbb{P}_v') + \sqrt{\frac{1}{2}\KL(\PP_u, \PP_u')} + \sqrt{\frac{1}{2}\KL(\PP_v, \PP_v')} 
\end{align*}
The only difference between $\PP_u$ and $\PP'_u$ (and similarly between $\PP_v$ and $\PP'_v$) is the $n$ first-layer rewards in $\calD_1$: under $\PP_u$ these rewards are $\Ber(\tfrac12)$ or $\Ber(\tfrac12+\Delta)$ depending on the action, whereas under $\PP'_u$ they are always $\Ber(\tfrac12)$. Hence, 
\begin{align*}
   \KL(\PP_u, \PP_u') \leq n\KL\left( \Ber\left(\tfrac{1}{2}+\Delta\right), \Ber\left(\tfrac{1}{2}\right)\right), \quad \KL(\PP_v, \PP_v') \leq n\KL\left( \Ber\left(\tfrac{1}{2}+\Delta\right), \Ber\left(\tfrac{1}{2}\right)\right).  
\end{align*}
Using the bound
$\KL(\Ber(\tfrac12+\Delta), \Ber(\tfrac12))\le 8\Delta^2$ for $\Delta\le \tfrac14$, we obtain
\begin{align*}
    D_{\mathrm{TV}}(\mathbb{P}_u,\mathbb{P}_v) 
    &\leq D_{\mathrm{TV}}(\mathbb{P}_u',\mathbb{P}_v') + 2\sqrt{\frac{n}{2} \times 8\Delta^2}   \\
    &= D_{\mathrm{TV}}(\mathbb{P}_u',\mathbb{P}_v') + 4\sqrt{n}\Delta. 
\end{align*}

\paragraph{Step 2: control repeats in the visited middle-layer states. }
By definition,
\[
D_{\mathrm{TV}}(\mathbb{P}_u',\mathbb{P}_v')
\;=\;\frac12\sum_{\calD}\big|
\mathbb{P}'_u(\calD)-\mathbb{P}'_v(\calD)
\big| 
\]
where the summation runs over all possible datasets. 
For any dataset $\calD=(\calD_1,\calD_2,\calD_3)$ of the form in \pref{app: dataset constriction}, define
\begin{align*}
   \mathsf{Rep}(\calD) = \ind\{\exists i\neq j, \ \ w_1^{(i)}=w_1^{(j)}\text{\ or\ } w_2^{(i)}=w_2^{(j)}\text{\ or\ } w_1^{(i)}=w_2^{(j)}\}. 
\end{align*}
That is, $\Rep(\calD)=1$ if any state in $W$ appears more than once in the dataset $\calD$.  Note that
under any $M\in\calM'$, the middle-layer states $\{w_1^{(i)}\}_{i=1}^n$ are i.i.d.\ and uniformly distributed over $W$. Indeed, $a_1^{(i)}\sim\unif\{u,v\}$, and conditional on $a_1^{(i)}$ the next state is uniform over the corresponding size-$m$ group; hence, marginally, $\Pr(w_1^{(i)}=w)
=\frac{1}{2m}$ for all $w\in W$. 
Moreover, $\{w_2^{(i)}\}_{i=1}^n$ are i.i.d.\ $\unif(W)$ by construction. Therefore, the $2n$ middle-layer states
$\{w_1^{(i)}\}_{i=1}^n \cup \{w_2^{(i)}\}_{i=1}^n$
are i.i.d.\ samples from $\unif(W)$.

Thus, the probability that any state in $W$ is sampled more than once is 
\begin{align*}
    \PP_u'(\Rep(\calD)=1) 
    &= 1- \PP_u'(\Rep(\calD)=0) \\
    &= 1 - 1\times \frac{2m-1}{2m} \times \frac{2m-2}{2m} \times \cdots \times \frac{2m-(2n-1)}{2m} \\
    &\leq 1-\left(1-\frac{n}{m}\right)^{2n} 
    \leq \frac{2n^2}{m}. 
\end{align*}
The same bound holds for $\PP_v'(\Rep(\calD)=1)$.  Decomposing the TV sum according to $\mathsf{Rep}(\calD)$ yields 
\begin{align}
    \sum_\calD  |\PP_u'(\calD) - \PP_v'(\calD)| &= \sum_{\calD:~ \Rep(\calD)=0} |\PP_u'(\calD) - \PP_v'(\calD)| + \sum_{\calD:~\Rep(\calD)=1} |\PP_u'(\calD) - \PP_v'(\calD)| \nonumber \\
    &\leq \sum_{\calD:~\Rep(\calD)=0} |\PP_u'(\calD) - \PP_v'(\calD)| +  \PP_u'(\Rep(\calD)=1) +   \PP_v'(\Rep(\calD)=1) \nonumber \\
    &\leq \sum_{\calD:~\Rep(\calD)=0} |\PP_u'(\calD) - \PP_v'(\calD)| + \frac{4n^2}{m}.   \label{eq: inter step 2}
\end{align}
%
%
%
%
%
%
%
\paragraph{Step 3: evaluate $\PP'_u(\calD)$ on datasets without repeats.} 
Fix a dataset $\calD$ with $\Rep(\calD)=0$. 
For a fixed model $M\in\calM_u'$, the likelihood factors over layers: 

\begin{itemize}[leftmargin=11pt, topsep=0pt]
\item \textbf{Samples in $\calD_1$:} \ \   Each sample in $\calD_1$ starts with $s_1$, choosing $a_1^{(i)}\sim \unif\{u,v\}$. Under $M\in\calM_u'$,  $r_1^{(i)}\sim \Ber(\frac{1}{2})$, and $w_1^{(i)}$ is uniform on the appropriate \mbox{size-$m$} set ($\WA$ or $\WB$) given $a_1^{(i)}$. For example, if $a_1^{(i)}=u$ then $w_1^{(i)}\sim \unif(\WA)$ according to \pref{fig:main}(a)(b). Hence, for a realized tuple $(s_1,a_1^{(i)}, r_1^{(i)}, w_1^{(i)})$, 
\begin{align*}
\PP_M\big((s_1, a_1^{(i)}, r_1^{(i)}, w_1^{(i)})\big)
&=\PP'((a_1^{(i)},r_1^{(i)})) \cdot \frac{1}{m}\ind\{\text{$M$'s transition is consistent with } (s_1,a_1^{(i)},w_1^{(i)})\}
\end{align*}
where $\PP'$ is mixture law $\PP'=\E_{M\sim \unif(\calM')}\PP_M$. It can be used here because all $M\in\calM'$ share the same law of $(a_1^{(i)},r_1^{(i)})$, and so does their mixture. The indicator enforces that the assignment of $M$ places $w_1^{(i)}$
in the correct group ($\WA$ or $\WB$) given $a_1^{(i)}$, as the example given above. 
\item \textbf{Samples in $\calD_2$.  }  Each sample in $\calD_2$ first draws $w_2^{(i)}\sim \unif(W)$ and chooses action $a$. Then the next-state $s_2^{(i)}$ is deterministic given $w_2^{(i)}$: if $w_2^{(i)}\in W_{\mathsf{A}}$ under $M$, then $s_2^{(i)}=\sA$; otherwise, $s_2^{(i)}=\sB$. Thus for a realized tuple $(w_2^{(i)},a,0,s_2^{(i)})$,
\[
\begin{aligned}
\PP_M\big((w_2^{(i)},a,0,s_2^{(i)})\big)
&=\frac{1}{2m}\ind\{\text{$M$'s transition is consistent with } (w_2^{(i)},a,s_2^{(i)})\}.
\end{aligned}
\]
\item \textbf{Samples in $\calD_3$. }  Each sample in $\calD_3$ first draws $s_3^{(i)}\sim \unif\{\sA,\sB\}$, chooses $z$, and receives reward $r_3^{(i)}$. By construction, $r_3^{(i)}=0$ if $s_3^{(i)}=\sA$, and $r_3^{(i)}=1$ if $s_3^{(i)}=\sB$. This procedure is the same for all models $M$. Thus, 
\begin{align*}
    \PP_M\big( (s_3^{(i)}, z, r_3^{(i)}\big) = \PP'((s_3^{(i)}, r_3^{(i)}))
\end{align*}
\end{itemize}
Multiplying across all samples in all $\calD_1, \calD_2, \calD_3$, we get 
\begin{align*}
\PP_M(\calD)
&= \PP'\big((a_1^{(i)}, r_1^{(i)})_{i=1}^n\big)\PP'\big((s_3^{(i)}, r_3^{(i)})_{i=1}^n\big) \cdot \frac{1}{m^n}\cdot \frac{1}{(2m)^n}\cdot \\
&\qquad  \ind\{\text{$M$'s transition is consistent with } (s_1, a_1^{(i)}, w_1^{(i)})_{i=1}^n \text{\ and\ }(w_2^{(i)}, a, s_2^{(i)})_{i=1}^n\}
\end{align*}
Averaging over $M\sim\unif(\calM_u')$ yields
\begin{align}
\PP_u'(\calD)
&= \PP'\big((a_1^{(i)}, r_1^{(i)})_{i=1}^n\big)\PP'\big((s_3^{(i)}, r_3^{(i)})_{i=1}^n\big)  \cdot \frac{1}{m^n}\cdot \frac{1}{(2m)^n}\cdot   \label{eq: tmpe step}\\
&\quad  \underbrace{\E_{M\sim \unif(\calM_u')}\big[\ind\{\text{$M$'s transition is consistent with } (s_1, a_1^{(i)}, w_1^{(i)})_{i=1}^n \text{\ and\ }(w_2^{(i)}, a, s_2^{(i)})_{i=1}^n\} \big]}_{(\star)}  \nonumber 
\end{align}



\paragraph{Step 4: count consistent assignments.}
Next, we calculate $(\star)$ for a given $\calD$ with $\Rep(\calD)=0$. 
Each transition sample $(s_1, a_1^{(i)}, w_1^{(i)})$ in $\calD_1$ specifies a constraint for $M$. For example, assume $(s_1, a_1^{(i)}, w_1^{(i)})=(s_1, u, w)$ for some $w\in W$. Then to make $M\in\calM_u'$ consistent with $(s_1, u, w)$, $M$ must assign $w$ to $W_{\mathsf{A}}$. Similarly, each transition tuple $(w_2^{(i)}, a, s_2^{(i)})$ sets a constraint for $M$: Assume $(w_2^{(i)}, a, s_2^{(i)}) = (w, a, \sB)$ for some $w\in W$. Then to make $M\in\calM_u'$ consistent with this tuple, $M$ must assign $w$ to $W_{\mathsf{B}}$. As there are no repeating states from $W$ in the dataset, there are $2n$ non-overlapping constraints on how $M$ should assign each $w\in W$ to $\{\WA, \WB\}$.   

The total number of possible assignments from $W$ to $\{\WA, \WB\}$ is $\binom{2m}{m}$. With $2n$ non-overlapping constraints, there are $N=2n$ elements in $W$ whose assignments are \emph{not free}. The total number of assignments under the constraints becomes $\binom{2m-N}{m-l}$, where $l$ is the number of constraints restricting $w$ to be in $W_{\mathsf{A}}$, and $N-l$ is the number of constraints restricting $w$ to $W_{\mathsf{B}}$. 

Then, $(\star)$ is the fraction of the assignments satisfying the constraints specified by $(s_1, a_1^{(i)}, w_1^{(i)})_{i=1}^n$ and $(w_2^{(i)}, a, s_2^{(i)})_{i=1}^n$, which is given by $\binom{2m-N}{m-l}\Big/\binom{2m}{m}$ and can be expanded as 
\begin{align*}
    \frac{\frac{(2m-N)!}{(m-l)!\,(m-N+l)!}}{\frac{(2m)!}{m!\,m!}} = \frac{\prod_{i=0}^{l-1}\left(m-i\right)\times\prod_{i=0}^{N-l-1}\left(m-i\right) }{\prod_{i=0}^{N-1} \left(2m-i\right)} = \frac{m^N}{(2m)^N}\cdot\frac{\prod_{i=0}^{l-1}\left(1-\frac{i}{m}\right)\times\prod_{i=0}^{N-l-1}\left(1-\frac{i}{m}\right) }{\prod_{i=0}^{N-1} \left(1-\frac{i}{2m}\right)}. 
\end{align*}
From this expression, we have
\begin{align*}
    (\star) &\geq 2^{-N} \left(1-\frac{l}{m}\right)^l \left(1-\frac{N-l}{m}\right)^{N-l}  \tag{lower bound the numerator and upper bound the denominator}\\
    &\geq 2^{-N} \left(1-\frac{l^2}{m}\right) \left(1-\frac{(N-l)^2}{m}\right)   \tag{$(1-\zeta)^k\geq 1-k\zeta$ for $\zeta\in(0,1)$ and $k\geq 1$}\\
    &\geq 2^{-N} \left(1-\frac{N^2}{m}\right) \tag{$l^2 + (N-l)^2\leq N^2$}
\end{align*}
and similarly 
\begin{align*}
    (\star) &\leq 2^{-N} \cdot\frac{1}{\left(1-\frac{N}{2m}\right)^N} \leq 2^{-N}\cdot \frac{1}{1-\frac{N^2}{2m}}\leq 2^{-N} \left(1+\frac{N^2}{m}\right). \tag{$\frac{1}{1-\zeta}\leq 1+2\zeta$ for $\zeta\in(0,\frac{1}{2})$}  
\end{align*}
Therefore, by \pref{eq: tmpe step} and that $N=2n$, 
\begin{align*}
    \PP_u'(\calD) \in \PP'\big((a_1^{(i)}, r_1^{(i)})_{i=1}^n\big)\PP'\big((s_3^{(i)}, r_3^{(i)})_{i=1}^n\big) \cdot \frac{1}{m^n}\cdot \frac{1}{(2m)^n} \cdot 2^{-2n} \left[1-\frac{4n^2}{m}, 1+\frac{4n^2}{m}\right]. 
\end{align*}
The same holds for $\PP_v'(\calD)$. Plugging them into the first term on the right-hand side of \pref{eq: inter step 2}, we get 
\begin{align*}
    &\sum_{\calD:~\Rep(\calD)=0} |\PP_u'(\calD) - \PP_v'(\calD)| \\
    &\leq \sum_{(a_1^{(i)}, r_1^{(i)}, w_1^{(i)}, w_2^{(i)}, s_2^{(i)}, s_3^{(i)}, r_3^{(i)})_{i=1}^n} \PP'((a_1^{(i)}, r_1^{(i)}))\PP'((s_3^{(i)}, r_3^{(i)}))  \cdot \frac{1}{m^n}\cdot \frac{1}{(2m)^n}\cdot 2^{-2n}\cdot \frac{8n^2}{m} \\
    &= \sum_{(w_1^{(i)}, w_2^{(i)}, s_2^{(i)})_{i=1}^n}  \frac{1}{m^n}\cdot \frac{1}{(2m)^n}\cdot 2^{-2n}\cdot \frac{8n^2}{m} \\
    &= (2m\times 2m\times 2)^n \times \frac{1}{(m\times 2m\times 4)^{n}}\cdot \frac{8n^2}{m}   \tag{each $w_1^{(i)}$ and $w_2^{(i)}$ has $2m$ possible values, and $s_2^{(i)}$ has $2$ possible values}\\
    &= \frac{8n^2}{m}. 
\end{align*}
Combining everything above, we get 
\begin{align*}
    D_{\mathrm{TV}}(\mathbb{P}_u,\mathbb{P}_v) \leq \frac{1}{2}\left(\frac{8n^2}{m} + \frac{4n^2}{m}\right) + 4\sqrt{n}\Delta = \frac{6n^2}{m} + 4\sqrt{n}\Delta. 
\end{align*}
\end{proof}



\begin{lemma}[An $\epsilon$-independent lower bound]\label{lem: eps indep lower bound}
    For any $\epsilon\leq \frac{1}{4}$, to achieve \mbox{$\E[J(\pi^\star) - J(\hat{\pi})]\leq \epsilon$} in the instance described in \pref{app: MDP consturction}--\pref{app: dataset constriction}, any offline algorithm must access at least $n\geq \Omega\big( \min\big\{\sqrt{m}, \frac{1}{\Delta^2}\big\}\big)$ samples. 
\end{lemma}
\begin{proof}
   By \pref{lem: reduction to TV} and \pref{lem: TV bound}, we have 
   \begin{align*}
       \E[J(\pi^\star) - J(\hat{\pi})] \geq \frac{1}{2}\left(1 - D_{\mathrm{TV}}(\PP_u, \PP_v)\right)\geq \frac{1}{2}\left(1 - \frac{6n^2}{m} - 4\sqrt{n}\Delta\right). 
   \end{align*}
   To make the left-hand side smaller than $\epsilon\leq \frac{1}{4}$, we need $\frac{6n^2}{m} + 4\sqrt{n}\Delta \geq \frac{1}{2}$, 
   implying that either $\frac{6n^2}{m}\geq \frac{1}{4}$ or $4\sqrt{n}\Delta\geq \frac{1}{4}$. Hence, $n\geq \Omega\big(\min\{\sqrt{m}, \frac{1}{\Delta^2}\}\big)$. 
\end{proof}

\begin{proof}[Proof of \pref{thm: eps dep lower bound}]
~\\
\textbf{$\epsilon$-dependent lower bound for $(s,a,r,s')$ data. }
    We start with proving the first part of the theorem, where the learner can only access $(s,a,r,s')$ data. We leverage the lower bound in \pref{lem: eps indep lower bound} (with the instances illustrated in \pref{fig:main}). Notice \pref{lem: eps indep lower bound} already proves this theorem for the special case $\epsilon=\frac{1}{4}$. 

    To lift that lower bound construction to an $\epsilon$-dependent bound for an arbitrary $\epsilon\in(0,\frac{1}{4}]$, we extend it in the way illustrated in \pref{fig: eps-dependent}. That is, we add an initial state $s_0$, on which the learner can only take one action $a_0$ that has an instantaneous reward of zero.  After taking $a_0$, with probability $p=4\epsilon$, the state transitions to $s_1$ (the initial state in the MDP instance of \pref{lem: eps indep lower bound}), and with probability $1-p=1-4\epsilon$, the state transitions to a chain of states $(z_1, z_2, z_3)$ on which there is only one action $a_0$ to choose and the reward is always zero.  In the previous construction, $\mu$ is induced by a behavior policy $\pi_b$. We will keep this behavior policy, as there is no decision to make on all other states we add. It is straightforward to check that $C^{\pi^\star}$ remains the same after this extension. For the function set $\calF$ defined in the previous construction (\pref{app: function construction}), we extend each $f\in\calF$ with $f(s_0, a_0) = p f(s_1)$ and $f(z_1,a_0) = f(z_2,a_0) = f(z_3,a_0)=0$. This keeps the function set to satisfy $Q^\star$-realizability and Bellman completeness. 

    With all the conditions satisfied, we check how many samples the learner needs to ensure $J(\pi^\star) - J(\hat{\pi}) \leq \epsilon$. Observe that 
    \begin{align*}
        J(\pi^\star) - J(\hat{\pi}) = p(V^\star(s_1) - V^{\hat{\pi}}(s_1)) = 4\epsilon (V^\star(s_1) - V^{\hat{\pi}}(s_1)). 
    \end{align*}
    In order to achieve $J(\pi^\star) - J(\hat{\pi}) \leq \epsilon$, the learner must ensure $V^\star(s_1) - V^{\hat{\pi}}(s_1)\leq \frac{1}{4}$. As already established in \pref{lem: eps indep lower bound}, this requires at least $\Omega\big(\min\{\sqrt{|\calS|}, \frac{1}{\Delta^2}\}\big)$ samples in the original MDP starting from $s_1$. However, in the extended MDP, $1-p$ portion of the offline data are trivial, and with only probability $p$ can the learner get any samples from the original MDP. Therefore, the total number of samples the learner needs is of order $\Omega\big(\frac{1}{p}\min\{\sqrt{|\calS|}, \frac{1}{\Delta^2}\}\big) = \Omega\big(\frac{1}{\epsilon}\min\{\sqrt{|\calS|}, \frac{1}{\Delta^2}\}\big)$. This proves the first claim of the theorem.

\textbf{$\epsilon$-dependent lower bound for trajectory data. } To prove a lower bound under trajectory feedback, we further extend the construction with a reduction established in \cite{jia2024offline}. We note that the reason why trajectory breaks the lower bound in the instances in \pref{fig:main} or \pref{fig: eps-dependent} is because with trajectory data, the learner can easily tell whether it is $u$ or $v$ that leads to $\sB$ by just checking which state the trajectory leads to after two steps. This is contrary to the case of $(s,a,r,s')$ data where the state is essentially ``re-sampled'' from $\mu$ after the state transitions to $\WA\cup \WB$.

\begin{figure}[t]
    \centering
    \includegraphics[width=0.65\textwidth]{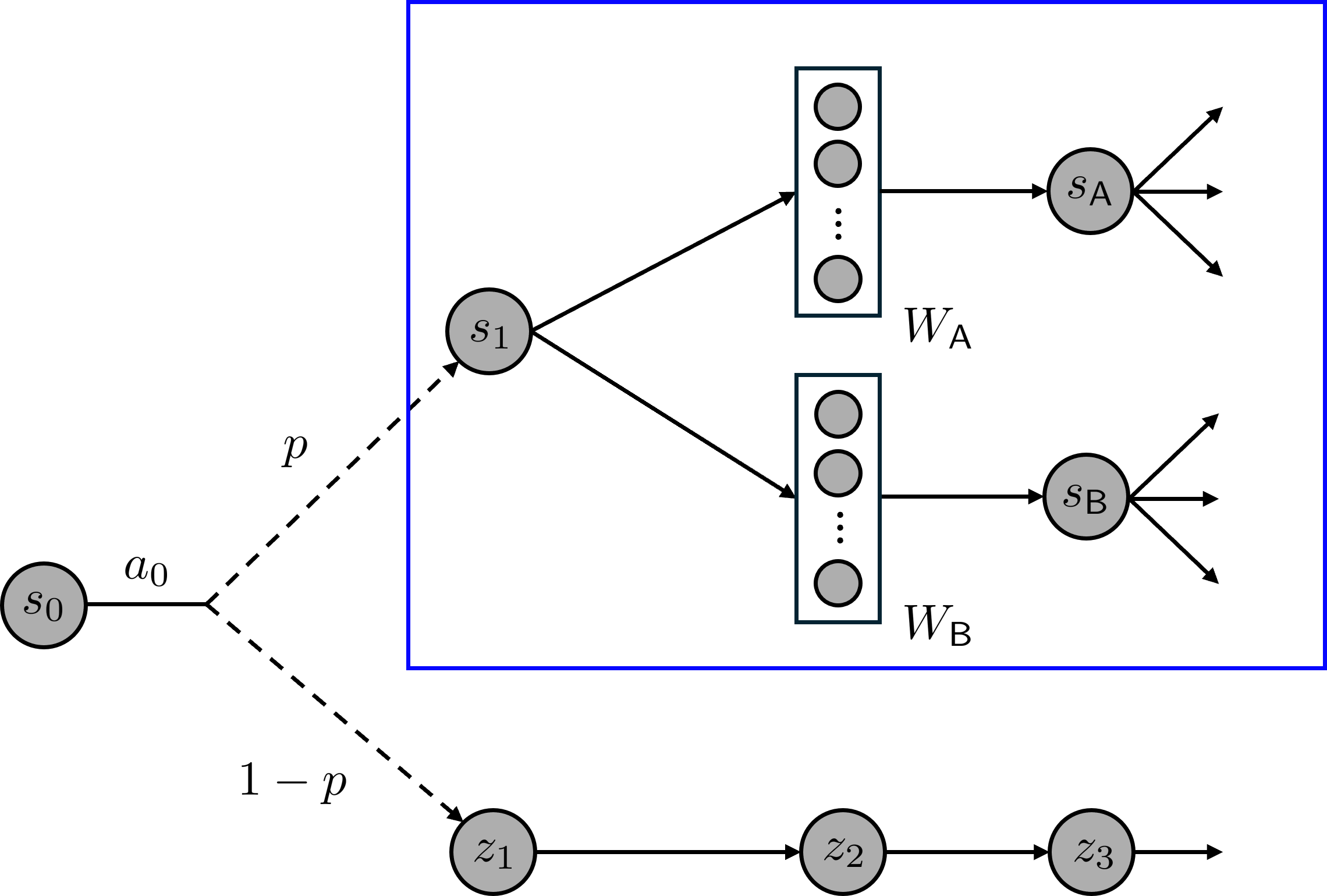}
    \caption{Construction for $\epsilon$-dependent lower bound with non-trajectory data. The blue box is the hard instance in \pref{lem: eps indep lower bound} (\pref{fig:main}). } \label{fig: eps-dependent}
\end{figure}

    The idea of \cite{jia2024offline} (in their Section~4.2) is to mimic such re-sampling under trajectory feedback. In our case, we extend the construction in \pref{fig: eps-dependent} by repeating the middle layers for $\Theta(H)$ times, resulting in \pref{fig: trajecto}. Denote $W_h=W_{h,\mathsf{A}}\cup W_{h,\mathsf{B}}$ for $h=3,4,\ldots, H-1$. There are two actions on states in $W_3, W_4, \ldots, W_{H-2}$.  If taking action~1 (indicated by the blue arrows in \pref{fig: trajecto}), the next state is drawn uniformly from \emph{the same group} on the next layer; if taking action~2 (indicated by the red arrows in \pref{fig: trajecto}), the next state is drawn uniformly from the \emph{union of the two groups.} We let action~1 have an instantaneous reward of $0$,  and action~2 have an instantaneous reward of $-2$. This implies that the optimal action in $W_3, W_4, \ldots, W_{H-2}$ is always action~1. Also, we extend any function $f$ in the function set (defined in \pref{app: function construction}) by letting $f(w, 1)=1$ and $f(w,2) = -2+1=-1$ for $w\in W_3 \cup \cdots \cup W_{H-2}$. For the behavior policy, we let $\pi_b(1|w) = \pi_b(2|w) = \frac{1}{2}$ for $w\in W_3 \cup \cdots \cup W_{H-2}$. 

    This construction makes $C^{\pi^\star}\leq 4$. This is because the state distribution coverage $\frac{d^{\pi^\star}(s)}{H\mu(s)}\leq 2$ remains, and thus $\frac{d^{\pi^\star}(s,a)}{H\mu(s,a)}= \frac{d^{\pi^\star}(s)\pi^\star(a|s)}{H\mu(s)\pi_b(a|s)}\leq 2\times 2=4$ due to $\pi_b(1|w)=\frac{1}{2}$ for $w\in W_{3}\cup \cdots W_{H-2}$. Furthermore, it keeps realizability and Bellman completeness. 
    
    Because $\pi_b$ chooses both actions with equal probability, for each trajectory \emph{that passes through $s_1$}, with probability $1-2^{-(H-4)}$,  action~2 is chosen at least once during layers $h=3, \ldots, H-2$. Furthermore, if action~2 is chosen in any of these layers, the state distribution on layer $H-1$ will keep no information of the action chosen on~$s_1$. This exactly mimics the ``re-sampling'' scenario when the learner only has $(s,a,r,s')\sim \mu$ data. 

    If the number of trajectories passing through $s_1$ is smaller than $2^{H-4}$, then with a constant probability, action~2 is chosen at least once in all these trajectories. In this case, the previous hardness argument for $(s,a,r,s')$ applies, and the learner requires at least $\Omega\big(\frac{1}{\epsilon} \min\{\sqrt{m}, \frac{1}{\Delta^2}\}\big) = \Omega\big(\frac{1}{\epsilon} \min\{\sqrt{|\calS|/H}, \frac{1}{\Delta^2}\}\big)$ samples.  This concludes that in this new construction, the learner requires at least $\Omega\big(\frac{1}{\epsilon \poly(H)} \min\{2^H, \sqrt{|\calS|}, \frac{1}{\Delta^2}\}\big)$ trajectories to learn an $\epsilon$-optimal policy. 
\end{proof}

\begin{figure}[t]
    \centering
    \includegraphics[width=0.9\textwidth]{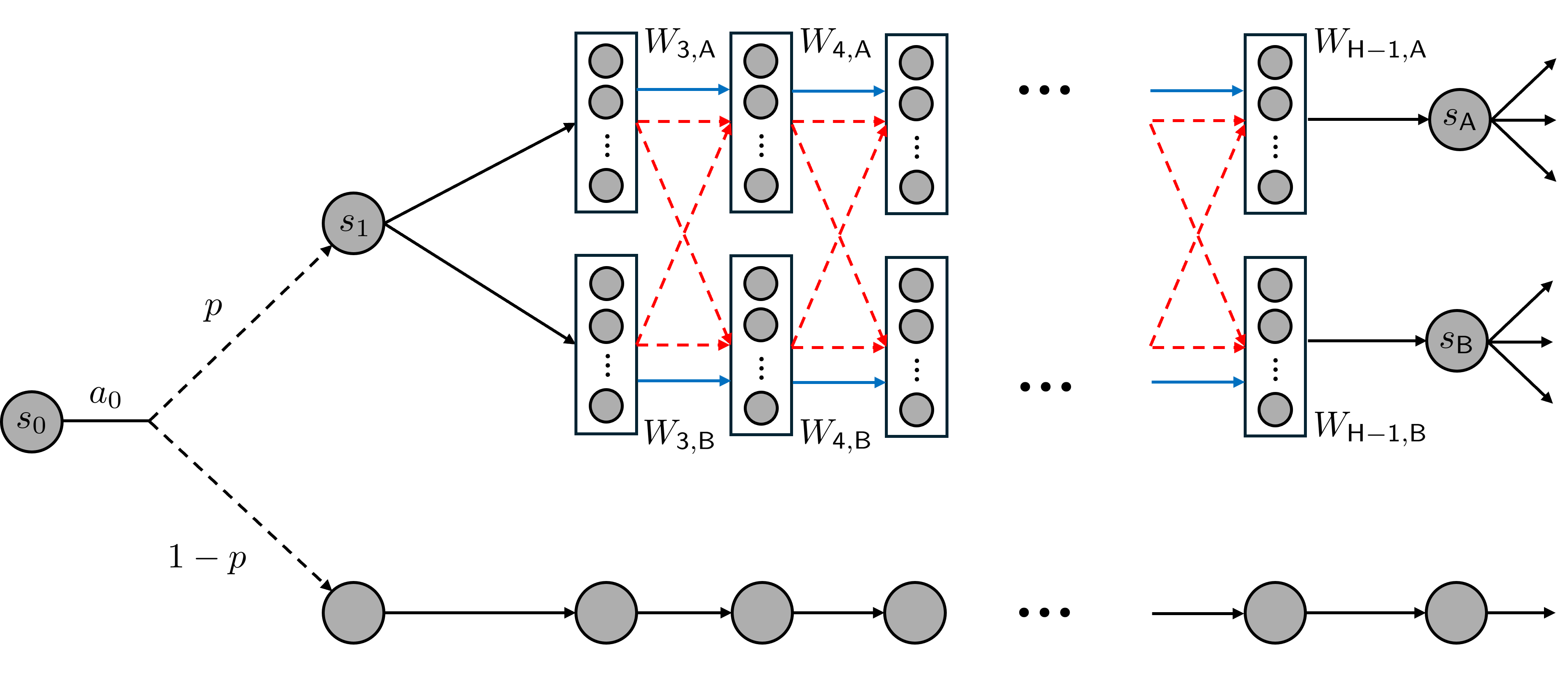}
    \caption{Construction for $\epsilon$-dependent lower bound with trajectory data (extended from \pref{fig: eps-dependent}). The blue arrows indicate the transition of action~1, which always leads to a uniform distribution \emph{over the same group} on the next layer. The red arrows indicate the transition of action~2, which always leads to a uniform distribution  \emph{over all states} on the next layer. }
    \label{fig: trajecto}
\end{figure}

\section{Competing with An Arbitrary Comparator} \label{app: arbitrary comparator}

A stronger and more desirable guarantee in offline RL under partial coverage is to compete with an arbitrary comparator policy $\picomp$, with the performance gap adapting to the coverage $C^{\picomp}$. Such guarantees have been established under model realizability \cite{uehara2021pessimistic} and $Q^\pi$-realizability \cite{xie2021bellman}. However, under $Q^\star$-realizability, this type of guarantee is fundamentally harder to obtain. Below, we present a pathological example showing that even in a seemingly simple setting, it is impossible to compete with an arbitrary $\picomp$.

Recall that the hardness of competing with $\pi^\star$ in \pref{thm: eps dep lower bound} relies on an extreme small or zero value gap $V^\star(s)-\max_{a\neq \pi^\star(s)}Q^\star(s,a)$. Below, we show that when the goal is to compete with an arbitrary comparator policy $\picomp$, the hardness holds even when this value gap is $\Theta(1)$. 

To show this, we use the same MDP structure in \pref{fig:main} with $r(s_1, u) \sim \Ber(\frac{1}{2} + 0.1)$, $r(s_1, v) \sim \Ber(\frac{1}{2})$ and other rewards remain the same. For every MDP, let the comparator policy be the one that always goes to the lowest branch and take action $z$ on $s_B$. Specifically, for $\calM_{u,x}$ and $\calM_{u,y}$, we consider $\picomp(s_1) = v$ and $\picomp(s_B)= z$, which is not the optimal policy now because the optimal policies for $\calM_{u,x}$ and $\calM_{u,y}$ take $u$ on $s_1$ and take $x$ and $y$ on $s_A$, respectively. For $\calM_{v,x}$ and $\calM_{v,y}$, we consider $\picomp(s_1) = u$ and $\picomp(s_B)= z$, which remains to be the optimal policy. We assume the offline data remain the same, so $C^{\picomp}$ is small for any MDP although the optimal policy may not be covered. In this case, when we try to learn a policy $\hat{\pi}$ such that $\E[J(\picomp) - J(\hat{\pi})]\leq \epsilon$, the constant gap on $u,v$ make it possible to distinguish them with constant samples. However, even if the learner can distinguish $u$ and $v$, they still do not know which action leads to the lowest branches unless $\text{poly}(m)$ samples are observed. Without such information, a random guess will lead to constant suboptimality gap.

However, we may still extend our algorithm and the associated decision complexity to accommodate an arbitrary comparator policy $\picomp$. As shown in \pref{alg:arb-dec}, given a policy class $\Pi$ that contains our target comparator policy, we solve a minimax problem against the joint worst case over both the environment model and the comparator policy. Given the above pathological example,  the associated decision complexity under $Q^\star$-realizability may not be bounded.  However, under $Q^\pi$-realizability and policy Bellman completeness assumptions, with $\calF_{\rm conf}$ chosen to match the confidence set in \cite{xie2021bellman}, \pref{alg:arb-dec} recovers their guarantees.
\begin{algorithm}[t]
    \caption{Offline Robust Estimation-to-Decision with Arbitrary Comparator}
    \label{alg:arb-dec}
    \textbf{Input:} Confidence set $\calF_{\mathrm{conf}}$, policy class $\Pi$ divergence measure $D^\pi(f\|M)$, parameter $\gamma$ for the offset version.   \\
    Define $\calM_\conf = \{M:~ Q_M^\star\in\calF_\conf\}$ and compute
    \begin{align*}
        \rhohat = 
        \begin{cases}
        \displaystyle\argmin_{\rho\in\Delta(\Pi)} \max_{(M, \picomp) \in\calM_\conf \times \Pi} \E_{\pi\sim \rho}\left[J_M(\picomp) - J_M(\pi) - \gamma \max_{f\in\calF_\conf} D^{\picomp}(f\|M)\right]  &\text{(offset version)}\\[15pt]
        \displaystyle\argmin_{\rho\in\Delta(\Pi)} \max_{(M, \picomp)\in \calM_\conf \times \Pi} \frac{J_M(\picomp) - \E_{\pi\sim \rho} [J_M(\pi)]}{\left(\max_{f\in\calF_\conf}D^{\picomp}(f\|M)\right)^{1/2}}.  &\text{(ratio version)} 
        \end{cases}
    \end{align*}
    \textbf{Output:} mixture policy $\hat{\pi}\sim \rhohat$.    
\end{algorithm}

\section{Proofs in \pref{sec:unify}}\label{app: section 4 proof}

\begin{lemma}
$\odecoff_{\gamma}^D(\calF_\conf) \le \frac{4}{\gamma}\left(\odecratio^D(\calF_\conf)\right)^2$
\label{lem:offset-ratio conn}.
\end{lemma}
\begin{proof}
For any $M \in \calM_\conf$, we have
\begin{align*}
     &\E_{\pi \sim \rho}\left[J_M(\pi_M) - J_M(\pi)\right] 
     \\&= \frac{J_M(\pi_M) - \E_{\pi \sim \rho}\left[J_M(\pi) \right]}{\sqrt{\max_{f \in \calF_{\rm conf}} D^{\pi_M}\left(f\| M\right)}} \sqrt{ \max_{f \in \calF_{\rm conf}} D^{\pi_M}\left(f\| M\right) }
     \\&\le \frac{4}{\gamma}\frac{\left(J_M(\pi_M) - \E_{\pi \sim \rho}\left[J_M(\pi) \right]\right)^2}{ \max_{f \in \calF_{\rm conf}}  D^{\pi_M}\left(f\| M\right)} + \gamma \max_{f \in \calF_{\rm conf}} D^{\pi_M}\left(f\| M\right) \tag{AM-GM}
\end{align*}
Since $\frac{J_M(\pi_M) - \E_{\pi\sim \rho} [J_M(\pi)]}{\left(\max_{f\in\calF_\conf}D^{\pi_M}(f\|M)\right)^{1/2}} > 0$, we have $  \odecoff_{\gamma}^D(\calF_\conf) \le \frac{4}{\gamma}\left(\odecratio^D(\calF_\conf)\right)^2$.
\end{proof}

\begin{proof}[Proof of \pref{thm: main}]
For the offset version, we have
\begin{align*}
    &J(\pi^\star) - \E_{\pi \sim \hat{\rho}}[J(\pi)] 
    \\&= J(\pi^\star) - \E_{\pi \sim \hat{\rho}}[J(\pi)] -  \gamma \max_{f\in\calF_\conf} D^{\pi^\star}(f\|M^\star) +  \gamma \max_{f\in\calF_\conf} D^{\pi^\star}(f\|M^\star)
    \\&\le \max_{M \in \calM_{\rm conf}}\left\{J_M(\pi_M) - \E_{\pi \sim \hat{\rho}}[J_M(\pi)] -  \gamma \max_{f\in\calF_\conf} D^{\pi_M}(f\|M)\right\} +  \gamma \max_{f\in\calF_\conf} D^{\pi^\star}(f\|M^\star)  \tag{$M^\star \in M_{\rm conf}$ from $Q^\star \in \calF_{\rm conf}$}
    \\&= \min_{\rho \in \Delta(\Pi)}\max_{M \in \calM_{\rm conf}}\left\{J_M(\pi_M) - \E_{\pi \sim \rho}[J_M(\pi)] -  \gamma \max_{f\in\calF_\conf} D^{\pi_M}(f\|M)\right\} +  \gamma \max_{f\in\calF_\conf} D^{\pi^\star}(f\|M^\star)  \tag{by the choice of $\hat{\rho}$}
    \\&= \odecoff_\gamma^D(\calF_\conf) + \gamma \max_{f\in\calF_\conf} D^{\pi^\star}(f\|M^\star).
\end{align*}
For the ratio version, we have
\begin{align*}
        &J(\pi^\star) - \E_{\pi \sim \hat{\rho}}[J(\pi)] 
        \\&= \frac{J(\pi^\star) - \E_{\pi \sim \hat{\rho}}[J(\pi)]}{\left(\max_{f\in\calF_\conf} D^{\pi^\star}(f\|M^\star)\right)^{\frac{1}{2}}} \left(\max_{f\in\calF_\conf} D^{\pi^\star}(f\|M^\star)\right)^{\frac{1}{2}}
        \\&\le \max_{M \in \calM_{\rm conf}}\frac{J_M(\pi_M) - \E_{\pi \sim \hat{\rho}}[J(\pi)]}{\left(\max_{f\in\calF_\conf} D^{\pi_M}(f\|M)\right)^{\frac{1}{2}}} \left(\max_{f\in\calF_\conf} D^{\pi^\star} (f\|M^\star)\right)^{\frac{1}{2}} \tag{$M^\star \in M_{\rm conf}$ from $Q^\star \in \calF_{\rm conf}$}
        \\&= \min_{\rho \in \Delta(\Pi)}\max_{M \in \calM_{\rm conf}}\frac{J_M(\pi_M) - \E_{\pi \sim \rho}[J(\pi)]}{\left(\max_{f\in\calF_\conf} D^{\pi_M}(f\|M)\right)^{\frac{1}{2}}} \left(\max_{f\in\calF_\conf} D^{\pi^\star}(f\|M^\star)\right)^{\frac{1}{2}}
        \\&= \odecratio^D(\calF_\conf)\left(\max_{f\in\calF_\conf} D^{\pi^\star}(f\|M^\star)\right)^{\frac{1}{2}}.
\end{align*}
\end{proof}

The guarantee of value-centered pessimism \pref{eq: value centr} is given by 
\begin{theorem}\label{thm: pessimism}
    If $Q^\star\in\calF_\conf$, then \algpes ensures $$ J(\pi^\star) - J(\hat{\pi}) \leq\pesratio^D(\calF_\conf) \big(D^{\pi^\star}(\hat{f}\|M^\star)\big)^{1/2}. $$
\end{theorem}

\begin{proof}[Proof of \pref{thm: pessimism}]
We have
\begin{align*}
    J(\pi^\star) - J(\pi_{\hat{f}}) &= \frac{J(\pi^\star) - J(\pi_{\hat{f}})}{\left(D^{\pi^\star}(\hat{f}\|M^\star)\right)^{\frac{1}{2}}}\left(D^{\pi^\star}(\hat{f}\|M^\star)\right)^{\frac{1}{2}}
    \\&\le \max_{M\in\calM_\conf}\frac{J_M(\pi_M) - J_M(\pi_{\hat{f}})}{\big(D^{\pi_M}(\hat{f}\|M)\big)^{1/2}} \big(D^{\pi^\star}(\hat{f}\|M^\star)\big)^{1/2} \tag{$M^\star \in M_{\rm conf}$ from $Q^\star \in \calF_{\rm conf}$}
    \\&= \pesratio^D(\calF_\conf) 
    \big(D^{\pi^\star}(\hat{f}\|M^\star)\big)^{1/2}.
\end{align*}
\end{proof}

\section{Proofs in \pref{sec:decision-error}}
\label{app:decision}



To connect $\odec$ with complexities established in prior work, we define the \emph{Exploitability Ratio}: 
\begin{definition}[Exploitability Ratio]\label{def: exploitability}
For $f\in\calF_\conf$, define
\begin{alignat*}{3}
\gapcomp(f; \calF_\conf)\,&=&
\max_{M\in\calM_\conf}
\frac{J_M(\pi_M) - J_M(\pi_f)}
{\E_{(s,a)\sim d^{\pi_M}_M}[f(s) - f(s,a) + \psi(\pi_M;s)]}
\end{alignat*}
with $\frac{0}{0}\triangleq 0$. If $\pi_f$ or $\pi_M$ is not unique, choose the one that makes this ratio the largest. 
\end{definition}
The ratio $\gapcomp(f;\calF_\conf)$ quantifies how much the greedy policy $\pi_f$ may be exploited by another model $M\in\calM_\conf$ (i.e., $J_M(\pi_M) - J_M(\pi_f)$ in the numerator), relative to how much $f$ thinks $\pi_f$ is better than $\pi_M$ (i.e., $f(s) - f(s,\pi_M(s))+ \psi(\pi_M; s) $ in the denominator). 
The next theorem shows performance bounds achieved by \algoff and \algpes in terms of $\gapcomp$: 
\begin{theorem}\label{thm: dec to ER}
Let $\hat{f}=\argmin_{f\in\calF_\conf} f(s_1)$ and assume $Q^\star\in\calF_\conf$. Then the following hold: 
\begin{align}
    \odecoff_{\gamma}^\av(\calF_\conf) & \textstyle \leq \frac{1}{2\gamma}\gapcomp(Q^\star;~\calF_\conf)^2 + \gamma \max_{f\in\calF_\conf}\Dav^{\pi^\star}(f \| M^\star),    \label{eq: bound 11}\\
    \odecratio^\av(\calF_\conf) &\leq \gapcomp(\hat{f}; \calF_\conf),  \label{eq: bound 22} \\
    \pesratio^\av(\calF_\conf) &\leq \gapcomp(\hat{f}; \calF_\conf).   \label{eq: bound 33}
\end{align}
\end{theorem}
Before proving \pref{thm: dec to ER}, we first provide auxiliary \pref{lem: symmetric} and \pref{lem: pessimismc solution gua}.

\begin{lemma}\label{lem: symmetric}
   Let $M, M'$ be two models and $f=Q_M^\star, f'=Q_{M'}^\star$. Then 
   \begin{align*}
       &\Dav^{\pi_M}(f'\|M) + \Dav^{\pi_{M'}}(f\|M') \\&\geq \frac{1}{2}\left(\E^{\pi_M, M}[f'(s)-f'(s,a) + \psi(\pi_M;s)] + \E^{\pi_{M'}, M'}[f(s)-f(s,a) + \psi(\pi_{M'};s)]\right)^2 
   \end{align*}
 where $\E^{\pi, M}[g(s,a)]$ denotes $\E_{(s,a)\sim d^\pi_M}[g(s,a)]$. 
\end{lemma}
\begin{proof}
\begin{align}
    &f(s_1) - f'(s_1) 
    \\&= J_M(\pi_M) - f'(s_1) \nonumber \\
    &=  \E^{\pi_M, M}\left[R_M(s,a)  - \psi(\pi_M; s)  \right] + \E^{\pi_M, M}\left[ -f'(s) + \E_{s'\sim P_M(\cdot|s,a)} [f'(s')] \right] \nonumber\\
    &= \E^{\pi_M, M}\left[ -f'(s,a) + R_M(s,a) + \E_{s'\sim P_M(\cdot|s,a)} [f'(s')]  \right] - \E^{\pi_M, M}\left[ f'(s) - f'(s,a) + \psi(\pi_M; s)\right]. \label{eq: same calc}
\end{align}
Similarly, 
\begin{align*}
    &f'(s_1) - f(s_1) 
    \\&= \E^{\pi_{M'}, M'}\left[ -f(s,a) + R_{M'}(s,a) + \E_{s'\sim P_{M'}(\cdot|s,a)} [f(s')]  \right] - \E^{\pi_{M'}, M'}\left[ f(s) - f(s,a) + \psi(\pi_{M'};s)\right]. 
\end{align*}
Summing up the two equalities, we get 
\begin{align*}
    &\E^{\pi_M, M}\left[ f'(s) - f'(s,a) + \psi(\pi_M;s)\right] + \E^{\pi_{M'}, M'}\left[ f(s) - f(s,a) + \psi({\pi_{M'}};s)\right]  \\
    &= \E^{\pi_M, M}\left[ -f'(s,a) + R_M(s,a) + \E_{s'\sim P_M(\cdot|s,a)} [f'(s')]  \right] \\
    &\qquad + \E^{\pi_{M'}, M'}\left[ -f(s,a) + R_{M'}(s,a) + \E_{s'\sim P_{M'}(\cdot|s,a)} [f(s')]  \right]. 
\end{align*}
Squaring both sides and using $(x+y)^2\leq 2(x^2+y^2)$: 
\begin{align*}
    &\left(\E^{\pi_M, M}\left[ f'(s) - f'(s,a) + \psi(\pi_M; s)\right] + \E^{\pi_{M'}, M'}\left[ f(s) - f(s,a) + \psi(\pi_{M'};s)\right]\right)^2 \\
    &\leq 2\left( \E^{\pi_M, M}\left[ -f'(s,a) + R_M(s,a) + \E_{s'\sim P_M(\cdot|s,a)} [f'(s')]  \right] \right)^2 \\
    &\qquad + 2\left(\E^{\pi_{M'}, M'}\left[ -f(s,a) + R_{M'}(s,a) + \E_{s'\sim P_{M'}(\cdot|s,a)} [f(s')]  \right]\right)^2 \\
    &= 2\Dav^{\pi_M}(f'\|M) +  2\Dav^{\pi_{M'}}(f\|M'). 
\end{align*}
\end{proof}


\begin{lemma}\label{lem: pessimismc solution gua}
    Let $\hat{f}= \argmin_{f\in\calF_\conf} f(s_1)$. Then for any $M\in\calM_\conf$, 
    \begin{align*}
        D^{\pi_M}_\av(\hat{f}\|M) \geq \left(\E^{\pi_M, M}[\hat{f}(s) - \hat{f}(s,a) + \psi(\pi_M;s)]\right)^2 
    \end{align*}
    where $\E^{\pi, M}[g(s,a)]$ denotes $\E_{(s,a)\sim d^\pi_M}[g(s,a)]$. 
\end{lemma}
\begin{proof}
    Let $f=Q_M^\star\in\calF_\conf$. By \pref{eq: same calc}, 
    \begin{align}
        &\E^{\pi_M, M}\left[ \hat{f}(s) - \hat{f}(s,a) + \psi(\pi_M; s)\right] 
        \\&= \E^{\pi_M, M}\left[ -\hat{f}(s,a) + R_M(s,a) + \E_{s'\sim P_M(\cdot|s,a)} [\hat{f}(s')]  \right] + \hat{f}(s_1) - f(s_1)  \label{eq: temptemp} \\
        &\leq \E^{\pi_M, M}\left[ -\hat{f}(s,a) + R_M(s,a) + \E_{s'\sim P_M(\cdot|s,a)} [\hat{f}(s')]  \right]. \tag{by the choice of $\hat{f}$}
    \end{align}
    Since the left-hand side is non-negative, after squaring both sides the inequality is still true. This proves the lemma. 
\end{proof}

\begin{proof}[Proof of \pref{thm: dec to ER}]
~\\
\textbf{Proving \pref{eq: bound 11}}
    By definition, 
    \begin{align*}
        \odecoff_{\gamma}^\av(\calF_\conf) = \min_{\rho\in\Delta(\Pi)} \max_{M\in\calM_\conf} \left\{J_M(\pi_M) - \E_{\pi\sim \rho}\left[J_M(\pi)\right] - \gamma\max_{f\in \calF_\conf} \Dav^{\pi_M}(f\|M)\right\}. 
    \end{align*}
    For the last term, we have 
\begin{align*}
    &\max_{f\in\calF_\conf} \Dav^{\pi_M} (f\|M) \\
    &\geq \Dav^{\pi_M} (Q^\star\|M) \tag{recall $Q^\star\triangleq Q^\star_{M^\star}$}\\
    &\geq \frac{1}{2}\Big(  \E^{\pi_M, M}[\underbrace{V^\star(s) - Q^\star(s,a) + \psi(\pi_M; s)}_{\geq  0}]+ \E^{\pi^\star, M^\star}[\underbrace{V^\star_M(s) - Q^\star_M(s,a) + \psi(\pi^\star;s)}_{\geq 0}]\Big)^2 - \Dav^{\pi^\star}(Q_M^\star \| M^\star)   \tag{\pref{lem: symmetric}} \\
    &\geq \frac{1}{2}\left(  \E^{\pi_M, M}[V^\star(s) - Q^\star(s,a) + \psi(\pi_M;s)]\right)^2 - \Dav^{\pi^\star}(Q_M^\star \| M^\star)    \\
    &\geq \frac{1}{2}\left(  \E^{\pi_M, M}[V^\star(s) - Q^\star(s,a) + \psi(\pi_M;s)]\right)^2 - \max_{f\in\calF_\conf}\Dav^{\pi^\star}(f \| M^\star). 
\end{align*}
Thus, 
 \begin{align*}
        &\odecoff_{\gamma}^\av(\calF_\conf) \\
        &= \min_{\rho\in\Delta(\Pi)} \max_{M\in\calM_\conf} \Bigg\{J_M(\pi_M) - \E_{\pi\sim \rho}\left[J_M(\pi)\right] \\
        &\qquad - \frac{1}{2}\gamma\left(  \E^{\pi_M, M}[V^\star(s) - Q^\star(s,a) + \psi(\pi_M;s)]\right)^2 + \gamma \max_{f\in\calF_\conf}\Dav^{\pi^\star}(f \| M^\star)\Bigg\} \\
        &\leq  \frac{1}{2\gamma}\min_{\rho\in\Delta(\Pi)}\max_{M\in\calM_\conf} \left(\frac{J_M(\pi_M) - \E_{\pi\sim \rho}[J_M(\pi)]}{\E^{\pi_M, M}[V^\star(s) - Q^\star(s,a) + \psi(\pi_M;s)]}\right)^2  + \gamma \max_{f\in\calF_\conf}\Dav^{\pi^\star}(f \| M^\star)
    \end{align*}
where in the last inequality we use the AM-GM inequality. Finally, using that $\min_x \max_y (F(x,y)^2) = \min_x  (\max_y F(x,y))^2 = (\min_x \max_y F(x,y))^2$ for non-negative $F$ proves \pref{eq: bound 11}.  

\textbf{Proving \pref{eq: bound 22}} 
   By definition, 
    \begin{align*}
       \odecratio^{\av}(\calF_\conf) &:= \min_{\rho\in\Delta(\Pi)} \max_{M\in\calM_\conf}\frac{J_M(\pi_M) - \E_{\pi\sim \rho} [J_M(\pi)]}{\sqrt{\max_{f\in\calF_\conf}\Dav^{\pi_M}(f\|M)}}. 
    \end{align*}
    For the denominator, we have 
    \begin{align}
        \sqrt{\max_{f\in\calF_\conf}\Dav^{\pi_M}(f\|M)} \geq \sqrt{\Dav^{\pi_M}(\hat{f}\|M)} \geq \E^{\pi_M, M}[\hat{f}(s) - \hat{f}(s,a) + \psi(\pi_M;s)] \label{eq: using}
    \end{align}
    by \pref{lem: pessimismc solution gua}. 
    Plugging this into the above definition shows  \pref{eq: bound 22}.

    \textbf{Proving \pref{eq: bound 33}} By definition, 
    \begin{align*}
        \pesratio^{\av}(\calF_\conf) &:= \max_{M\in\calM_\conf}\frac{J_M(\pi_M) - J_M(\pi_{\hat{f}})}{\sqrt{\Dav^{\pi_M}(\hat{f}\|M)}}. 
    \end{align*}
    Using \pref{eq: using} again in the denominator shows \pref{eq: bound 33}. 
\end{proof}

\begin{proof}[Proof of \pref{thm: gap bound}]
   By definition, when $\psi\equiv 0$, 
   \begin{align*}
       \gapcomp(f; \calF_\conf) &= \max_{M\in\calM_\conf}
\frac{J_M(\pi_M) - J_M(\pi_f)}
{\E_{(s,a)\sim d^{\pi_M}_M}[f(s) - f(s,a)]} \\
&\leq \max_{M\in\calM_\conf}
\frac{\E_{(s,a)\sim d^{\pi_M}_M} \left[Q_M^{\pi_f}(s,\pi_M) - Q_M^{\pi_f}(s,\pi_f)\right] }
{\E_{(s,a)\sim d^{\pi_M}_M}[\Delta_f \|\pi_M(\cdot|s) - \pi_f(\cdot|s)\|_1]} \tag{by the performance difference and the gap definition} \\
&\leq \max_{M\in\calM_\conf}
\frac{\E_{(s,a)\sim d^{\pi_M}_M} \left[H \|\pi_M(\cdot|s) - \pi_f(\cdot|s)\|_1\right] }
{\E_{(s,a)\sim d^{\pi_M}_M}[\Delta_f \|\pi_M(\cdot|s) - \pi_f(\cdot|s)\|_1]} \\
&= \frac{H}{\Delta_f}. 
   \end{align*}
Combining this bound with \pref{thm: dec to ER} finishes the proof.
\end{proof}

The first two assumptions in \pref{thm: curved boundary} are mild. The third \pref{eq: curved} requires some \emph{curvature} at the boundary of the feature set, making the optimal action unique. Notably, it only requires curvature at the greedy action of $\theta^\star$ but not other $\theta$'s, similar to \pref{thm: gap bound} that only depends on the value gap of~$Q^\star$.

\begin{proof}[Proof of \pref{thm: curved boundary}]
    Following the same steps as in the Proof of \pref{thm: dec to ER} and $\psi\equiv 0$, we arrive at the following bound: 
    \begin{align*}
        &\odecoff_{\gamma}^\av(\calF_\conf) \\
        &= \min_{\rho\in\Delta(\Pi)} \max_{M\in\calM_\conf} \Bigg\{J_M(\pi_M) - \E_{\pi\sim \rho}\left[J_M(\pi)\right] \\
        &\qquad - \frac{1}{2}\gamma\left(  \E^{\pi_M, M}[V^\star(s) - Q^\star(s,a)]\right)^2 + \gamma \max_{f\in\calF_\conf}\Dav^{\pi^\star}(f \| M^\star)\Bigg\} \\
        &\leq \max_{M\in\calM_\conf} \Bigg\{J_M(\pi_M) - J_M(\pi^\star) - \frac{1}{2}\gamma\left(  \E^{\pi_M, M}[V^\star(s) - Q^\star(s,a)]\right)^2 \Bigg\}+ \gamma \max_{f\in\calF_\conf}\Dav^{\pi^\star}(f \| M^\star). 
    \end{align*}
    By the performance difference lemma and the second assumption, we have 
    \begin{align*}
        J_M(\pi_M) - J_M(\pi^\star) 
        &= \E^{\pi_M, M}\left[ Q^{\pi^\star}_M(s,\pi_M) - Q^{\pi^\star}_M(s, \pi^\star)\right] \\
        &\leq L\E^{\pi_M, M}\left[ \|\phi(s,\pi_M) - \phi(s, \pi^\star)\|\right]
    \end{align*}
    where $\phi(s,\pi):=\phi(s,\pi(s))$. On the other hand, 
    \begin{align*}
        \E^{\pi_M, M} \left[V^\star(s) - Q^\star(s,a)\right] 
        &= \left(\E^{\pi_M, M} \left[\phi(s,\pi^\star) - \phi(s,\pi_M)\right]\right)^\top \theta^{\star} \\
        &\geq \kappa \E^{\pi_M, M} \left[\|\phi(s,\pi^\star) - \phi(s,\pi_M)\|^\beta\right]
    \end{align*}
    where the last inequality is by the assumption of the theorem. Therefore, 
    \begin{align*}
         &J_M(\pi_M) - J_M(\pi^\star) - \frac{1}{2}\gamma \left(\E^{\pi_M, M} \left[V^\star(s) - Q^\star(s,a)\right] \right)^2 \\
         &\leq L\E^{\pi_M, M}\left[ \|\phi(s,\pi_M) - \phi(s, \pi^\star)\|\right] - \frac{\kappa^2\gamma}{2} \left(\E^{\pi_M, M} \left[\|\phi(s,\pi^\star) - \phi(s,\pi_M)\|^\beta\right]\right)^2 
    \end{align*}
    For any random variable $X\geq 0$, and any $A,B>0$, we can bound 
    \begin{align*}
        &A\E[X] - B\left(\E[X^\beta]\right)^2 \\
        &\leq A \left(\E[X^\beta]\right)^{\frac{1}{\beta}} - B\left(\E[X^\beta]\right)^2 \\
        &= AY^{\frac{1}{\beta}} - BY^2  \tag{let $Y=\E[X^\beta]$} \\
        &\lesssim A^{\frac{2\beta}{2\beta-1}} B^{-\frac{1}{2\beta-1}}. 
    \end{align*}
    Plugging in $A=L$ and $B=\frac{\kappa^2\gamma}{2}$ gives $\left(\frac{L^{2\beta}}{\kappa^2\gamma}\right)^{\frac{1}{2\beta-1}}$. 
\end{proof}


\begin{lemma}[Second-Order Performance Difference Lemma]\label{lem: SO PDL}
If the regularizer $\psi$ satisfies \pref{assum:psi_proper}, then for any policy $\pi$ and model $M$ with $Q^\star_M(s,a) - Q_M^\pi(s,a) \le B, \,\, \forall s,a$, we have
\begin{align*}
    J_M(\pi_M)-  J_M(\pi) \le 3\left(1 + BHC^{\psi}_2\right) \E^{\pi_M, M}\left[\breg_{\psi}\left(\pi, \pi_M; s\right)\right], 
\end{align*}
where $\E^{\pi, M}[g(s)]$ denotes $\sum_{s} d^{\pi}_M(s) g(s)$. 
\label{lem:bregman-pdl}
\end{lemma}

\begin{proof}
Below, denote $g(s,\pi) = \sum_a \pi(a|s)g(s,a)$. 
For $s \in \calS_{H+1}$, define $V_M^\pi(s) = 0$. For any $s \in \calS_h$, we have
\begin{align*}
   &V_M^{\star}(s) - V_M^\pi(s) \\
   &= Q_M^\star(s,\pi_M) - \psi(\pi_M; s) - Q^\pi_M(s, \pi) + \psi(\pi; s) \\
   &= R_M(s,\pi_M) + \E_{s'\sim P_M(\cdot|s,\pi_M)}[V^\star_M(s')] - R_M(s,\pi) - \E_{s'\sim P_M(\cdot|s,\pi)}[V^\pi_M(s')] - \psi(\pi_M; s) + \psi(\pi;s) \\
   &= R_M(s,\pi_M) + \E_{s'\sim P_M(\cdot|s,\pi_M)}[V^\pi_M(s')] - R_M(s,\pi) - \E_{s'\sim P_M(\cdot|s,\pi)}[V^\pi_M(s')] - \psi(\pi_M; s) + \psi(\pi;s) \\
   &\qquad + \E_{s'\sim P_M(\cdot|s,\pi_M)}[V^\star_M(s')-V^\pi_M(s')] \\
   &= Q_M^\pi(s,\pi_M) - Q_M^\pi(s,\pi)  - \psi(\pi_M; s) + \psi(\pi;s) + \E_{s'\sim P_M(\cdot|s,\pi_M)}[V^\star_M(s')-V^\pi_M(s')] \\
   &= \sum_a (\pi_M(a|s) - \pi(a|s))Q_M^\pi(s,a)  - \psi(\pi_M; s) + \psi(\pi;s) + \E_{s'\sim P_M(\cdot|s,\pi_M)}[V^\star_M(s')-V^\pi_M(s')] \\
   &= \underbrace{\sum_a (\pi_M(a|s) - \pi(a|s))(Q_M^\pi(s,a) - Q_M^\star(s,a))}_{\term_1}\\
   &\qquad + \underbrace{\Big(Q^\star_M(s,\pi_M) - \psi(\pi_M; s)\Big) - \Big(Q^\star_M(s,\pi) - \psi(\pi;s) \Big)}_{\term_2} + \E_{s'\sim P_M(\cdot|s,\pi_M)}[V^\star_M(s')-V^\pi_M(s')]
\end{align*}
Since for any $s$, $\pi_M(\cdot|s) = \argmax_{p \in \Delta(\calA)}\left\{\sum_{a} p(a) Q_M^\star(s,a) - \psi(p)\right\}$, from \pref{assum:psi_proper}, $\psi$ is a  Legendre mirror map. Thus, from \pref{lem:breg}, we have
\begin{align*}
    \term_2 = \breg_{\psi}\left(\pi, \pi_M; s\right). 
\end{align*}
Furthermore, 
\begin{align*}
    \term_1 
    &=\sum_{a \in \calA} \left(\pi_M(a|s) - \pi(a|s)\right)\left(Q_M^{\pi}(s, a) - Q_M^\star(s, a)\right) 
    \\&\le\frac{1}{\eta}\KL\left(\pi(\cdot|s)\|\pi_M(\cdot|s)\right) + \eta \sum_{a}\pi_M(a|s)\left(Q_M^\star(s,a) - Q_M^\pi(s,a)\right)^2  \tag{\pref{lem:KL-stability}}
    \\&\le \frac{1}{\eta}\KL\left(\pi(\cdot|s)\|\pi_M(\cdot|s)\right) + \eta B \sum_{a}\pi_M(a|s)\left(Q_M^\star(s,a) - Q_M^\pi(s,a)\right)\tag{$0 \le Q^\star_M(s,a) - Q_M^\pi(s,a) \le B, \, \forall \pi$ }
    \\&\leq  \frac{1}{\eta}C^\psi_2 \breg_\psi(\pi, \pi_M; s) + \eta B \cdot \E_{s' \sim P_M(\cdot|s,\pi_M)}\left[V_M^\star(s') - V_M^\pi(s')\right]. 
\end{align*}
Overall, 
\begin{align*}
    V_M^{\star}(s) - V_M^\pi(s) \leq \left(1+\frac{1}{\eta}C^\psi_2\right) \breg_\psi(\pi, \pi_M; s) + \left(1+\eta B\right)  \E_{s' \sim P_M(\cdot|s,\pi_M)}\left[V_M^\star(s') - V_M^\pi(s')\right]. 
\end{align*}
Letting $\eta=\frac{1}{BH}$ and taking expectation over $s\sim d^{\pi_M}_M$ in $\calS_h$, we get 
\begin{align*}
   &\sum_{s\in\calS_h} d^{\pi_M}_M(s)\Big(V_M^{\star}(s) - V_M^\pi(s) \Big) \\
   &\leq \left(1+BHC^\psi_2\right) \sum_{s\in\calS_h} d^{\pi_M}_M(s) \breg_\psi(\pi, \pi_M; s) + \left(1+\frac{1}{H}\right)\sum_{s\in\calS_{h+1}} d^{\pi_M}_M(s)\Big(V_M^{\star}(s) - V_M^\pi(s) \Big). 
\end{align*}
Expand this recursively over layers and using that $(1+\frac{1}{H})^H\leq 3$ gives the desired inequality. 

\end{proof}

\begin{proof}[Proof of \pref{thm:regular-dec}] For any $f \in \calF_{\rm conf}$, we have
\begin{align*}
&\E_{(s,a)\sim d^{\pi_M}_M}[f(s) - f(s,a) + \psi(\pi_M;s)] 
\\&= \E_{(s,a)\sim d^{\pi_M}_M}[\E_{a \sim \pi_f(\cdot|s)}[f(s,a)] - \psi(\pi_f; s) - f(s,a) + \psi(\pi_M; s)]
\\&= \E_{(s,a)\sim d^{\pi_M}_M}\left[\text{Breg}_{\psi}\left(\pi_M, \pi_{f}; s\right)\right]  \tag{\pref{lem:breg}}
\\& \ge  \frac{1}{C_1^{\psi}}\E_{(s,a)\sim d^{\pi_M}_M}\left[\text{Breg}_{\psi}\left(\pi_f, \pi_{M}; s\right)\right].  \tag{\pref{assum:psi_proper}}
\end{align*}
Thus, from \pref{lem:bregman-pdl}, we have
\begin{align*}
&\frac{J_M(\pi_M) - J_M(\pi_f)}
{\E_{(s,a)\sim d^{\pi_M}_M}[f(s) - f(s,a) + \psi(\pi_M;s)]} 
\\&\le \frac{3C_1^\psi\left(1 + BHC^{\psi}_2\right) \E^{\pi_M, M}\left[\breg_{\psi}\left(\pi_f, \pi_M; s\right)\right]}{ \E^{\pi_M, M}\left[\breg_{\psi}\left(\pi_f, \pi_M; s\right)\right]} = 3C_1^\psi\left(1 + BHC^{\psi}_2\right)
\end{align*}
From \pref{lem:bound-of-Q*}, $B = H^2$, thus 
\begin{align}
\gapcomp(f; \calF_\conf) \le 3C_1^\psi\left(1 + H^3C^{\psi}_2\right) 
\label{eq:ER-regular}
\end{align}
Combining this bound with \pref{thm: dec to ER} together with  $\odecoff_{\gamma}^D(\calF_\conf) \le \frac{4}{\gamma}\big(\odecratio^D(\calF_\conf)\big)^2$ finishes the proof.
\end{proof}

\subsection{Comparison with \cite{uehara2023offline}}\label{app:regularized-proof}
\cite{uehara2023offline} considers KL regularized MDPs, which is equivalent to our formulation with $\psi(p; s) = \alpha \KL\left(p\|\pi_{\rm ref}(\cdot|s)\right)$. Besides $Q^\star \in \calF$, \cite{uehara2023offline} additionally assume access to a function class $\mathcal{L}$ such that a specific Lagrangian multiplier is realizable and every $\ell \in \mathcal{L}$ has $\|\ell\|_{\infty} \le B_{\mathcal{L}}$. Let $\hat{\pi}$ be the output policy of Algorithm 1 in \cite{uehara2023offline},  their Theorem 3 shows that with probability $1-\delta$,
\begin{align}
    \E_{(s,a) \sim d^{\pi^\star}}[R(s,a)] - \E_{(s,a) \sim d^{\hat{\pi}}}[R(s,a)] \le n^{-\frac{1}{4}} \text{Poly}\left(|\calA|, H, C^{\pi^\star}, R_{\max}, \log\left(\frac{|\calF||\mathcal{L}|}{\delta}\right), B_{\mathcal{L}}\right)
\label{eq:uehara-result}
\end{align}
where $R_{\max} = \max_{s,a}\left|\log \frac{\pi^\star(a|s)}{\pi_{\rm ref}(a|s)} \right|$. \cite{uehara2023offline} also assumes the offline action distribution is identical to $\pi_{\rm ref}$. In other words,  the offline dataset $\mathcal{D}$ consists of $n$ i.i.d. tuples $(s, a, r, s')$, where $s \sim \mu$ for some unknown distribution $\mu \in \Delta(\mathcal{S})$, $a \sim \pi_{\rm ref}(\cdot|s)$, $\mathbb{E}[r | s, a] = R(s, a)$, and $s' \sim P(\cdot | s, a)$. This assumption is restrictive and our methods do not need it.


Let $\hat{f}$ be the output of \pref{eq: value centr} with $\psi(p; s) = \alpha \KL\left(p\|\pi_{\rm ref}(\cdot|s)\right)$, we have
\begin{align*}
    &\left|\E^{\pi^\star, M^\star}\left[ -\hat{f}(s,a) + R(s,a) + \E_{s'\sim P(\cdot|s,a)} [\hat{f}(s')]  \right]\right|
    \\&\ge \E^{\pi^\star, M^\star}\left[ -\hat{f}(s,a) + R(s,a) + \E_{s'\sim P(\cdot|s,a)} [\hat{f}(s')]  \right] + \hat{f}(s_1) - V^\star(s_1) \tag{by the choice of $\hat{f}$}
    \\& = \E^{\pi^\star, M^\star}\left[ \hat{f}(s) - \hat{f}(s,a) + \psi(\pi^\star; s)\right] 
    \\&= \E^{\pi^\star, M^\star}[\E_{a \sim \pi_{\hat{f}}(\cdot|s)}[\hat{f}(s,a)] - \psi(\pi_{\hat{f}}; s) - \hat{f}(s,a) + \psi(\pi^\star; s)]
    \\&=  \alpha\E^{\pi^\star, M^\star}\left[\KL(\pi^\star, \pi_{\hat{f}}; s)\right]  \tag{\pref{lem:breg}}
\end{align*}
From \pref{lem:pdl} and using the confidence set defined in \pref{sec:density-confidence}, we have
\begin{align*}
     &\E_{(s,a) \sim d^{\pi^\star}}[R(s,a)] - \E_{(s,a) \sim d^{\pi_{\hat{f}}}}[R(s,a)]  
     \\&\le \E^{\pi^\star, M^\star}\left[\text{TV}\left(\pi^\star(\cdot|s), \pi_{\hat{f}}(\cdot|s)\right)\right] - \alpha \E^{\pi^\star, M^\star}\left[\KL\left(\pi^\star(\cdot|s), \pi_{\hat{f}}(\cdot|s)\right)\right] + \alpha \E^{\pi^\star, M^\star}\left[\KL\left(\pi^\star(\cdot|s), \pi_{\hat{f}}(\cdot|s)\right)\right]
     \\&\le \frac{1}{\alpha} + \alpha\left|\E^{\pi^\star, M^\star}\left[\hat{f}(s,a) - R(s,a) - \E_{s' \sim P(\cdot|s,a)}\left[\hat{f}(s')\right]\right]\right|
     \\&\le  \order\left(\frac{1}{\alpha} + \alpha B_{\mathcal{W}}H^2\sqrt{\frac{2\log\left(|\calF||\calW|/\delta\right)}{n}} \right)  \tag{\pref{lem: weight case}}
     \\&= \order\left(H\left(\frac{B_{\mathcal{W}}^2\log\left(|\calF||\calW|/\delta\right)}{n}\right)^{\frac{1}{4}}\right) \tag{Optimal choice of $\alpha$}
\end{align*}

Compared with \pref{eq:uehara-result}, our bound eliminates the polynomial dependence on $|\calA|$, $R_{\max}$ and remains valid for arbitrary offline action distributions, which need not coincide with $\pi_{\rm ref}$. Moreover, under the measure of regularized objective $J(\pi)$, Algorithm in \pref{eq: value centr} can achieve $1/\epsilon^2$ sample complexity bound as shown in \pref{eq: bound 11}, which cannot be obtained by the analysis in \cite{uehara2021pessimistic}. Our analysis here relies on \pref{assum: weight real}, and therefore introduces an additional function class $\mathcal{W}$ for density ratio realizability. This assumption is comparable with the 
Lagrangian multiplier realizability condition in \cite{uehara2023offline}, which posits that the corresponding dual variable lies in a function class $\mathcal{L}$.

\section{Proofs in \pref{sec:estimation-error}}
\label{app:estimation-error}

In this section, we assume $f \in [0,H]$ for every $f \in \calF$. This is valid because $Q^\star \in [0,H]$ from \pref{lem:bound-of-Q*}, and we can remove every $f$ not in $[0,H]$ out of $\calF$.

\begin{proof}[Proof of \pref{lem: BC case}] 
We first prove $Q^\star \in \calF_{\conf}$ with high probability. Define
\begin{align*}
    X_{g}^\star(s,a,r,s') &= \left(g(s,a) - r - V^\star(s')\right)^2 -   \left(Q^\star(s,a) - r - V^\star(s')\right)^2  
    \\&= \left(g(s,a) - Q^\star(s,a)\right)^2 - \left(g(s,a) - Q^\star(s,a)\right)\left(Q^\star(s,a) - r - V^\star(s')\right)
\end{align*}
Given $\E\left[ r + V^\star(s')\mid s,a\right] = (\mathcal{T}Q^\star)(s,a) = Q^\star(s,a)$, we have
\begin{align*}
    \E\left[X_g^\star(s,a,r,s') \mid s,a\right] = (g(s,a) - Q^\star(s,a))^2. 
\end{align*}
By Freedman inequality, with probability at least $1-\delta$, for any $g \in \mathcal{G}$, we have
\begin{align*}
    &\left|\E_{(s,a,r,s') \sim \calD}\left[X^\star_g(s,a,r,s')\right] - \E_{(s,a) \sim \calD}\left[(g(s,a) - Q^\star(s,a))^2\right]\right|
    \\& \le H\sqrt{\frac{\log(|\mathcal{G}|/\delta)}{n}\E_{(s,a) \sim \calD}\left[(g(s,a) - Q^\star(s,a))^2\right]} + \frac{H^2\log(|\mathcal{G}|/\delta)}{n}
    \\&\le  \frac{1}{2}\E_{(s,a) \sim \calD}\left[(g(s,a) - Q^\star(s,a))^2\right] + \frac{2H^2\log(|\mathcal{G}|/\delta)}{n}, \tag{AM-GM}
\end{align*}
where we assume $g \in [0,H]$ for every $g \in \mathcal{G}$. This is valid because $Q^\star \in [0,H]$ from \pref{lem:bound-of-Q*}. Thus, 
\begin{align*}
    -\E_{(s,a,r,s') \sim \calD}\left[X^\star_g(s,a,r,s')\right] &\le -\frac{1}{2}\E_{(s,a) \sim \calD}\left[(g(s,a) - Q^\star(s,a))^2\right] + \frac{2H^2\log(|\mathcal{G}|/\delta)}{n} \\
    &\le \frac{2H^2\log(|\mathcal{G}|/\delta)}{n}. 
\end{align*}
This implies
\begin{align*}
    &\E_{(s,a,r,s') \sim \calD}\left[\left(Q^\star(s,a) - r - V^\star(s')\right)^2\right] \\
    &\quad - \min_{g \in \mathcal{G}}\E_{(s,a,r,s') \sim \calD}\left[\left(g(s,a) - r - V^\star(s')\right)^2\right] \le  \frac{2H^2\log(|\mathcal{G}|/\delta)}{n} \le \epsilon_{\rm stat}.
\end{align*}
Now we prove the guarantee of $D_{\av}^{\pi^\star}$. Define
\begin{align*}
X_f(s,a,r,s') &=  \left(f(s,a) - r - f(s')\right)^2 -   \left((\mathcal{T}f)(s,a) - r - f(s')\right)^2 
\\& = (f(s,a) - (\mathcal{T}f)(s,a))^2 + 2\left(f(s,a) - (\mathcal{T}f)(s,a))\right)\left((\mathcal{T}f)(s,a) - r - f(s')\right).
\end{align*}
Given $\E\left[ r + f(s')\mid s,a\right] = (\mathcal{T}f)(s,a)$, we have
\begin{align*}
    \E\left[X_f(s,a,r,s') \mid s,a\right] = (f(s,a) - (\mathcal{T}f)(s,a))^2. 
\end{align*}
By Freedman inequality, with probability at least $1-\delta$, for any $f \in \calF$, we have
\begin{align*}
    &\left|\E_{(s,a,r,s') \sim \calD}\left[X_f(s,a,r,s')\right] - \E_{(s,a) \sim \calD}\left[(f(s,a) - (\mathcal{T}f)(s,a))^2\right]\right|
    \\& \le H\sqrt{\frac{\log(|\calF|/\delta)}{n}\E_{(s,a) \sim \calD}\left[(f(s,a) - (\mathcal{T}f)(s,a))^2\right]} + \frac{H^2\log(|\calF|/\delta)}{n}. 
\end{align*}
Thus, 
\begin{align*}
    &\E_{(s,a) \sim \calD}\left[(f(s,a) - (\mathcal{T}f)(s,a))^2\right] 
    \\&\le \E_{(s,a,r,s') \sim \calD}\left[X_f(s,a,r,s')\right] +  H\sqrt{\frac{\log(|\calF|/\delta)}{n}\E_{(s,a) \sim \calD}\left[(f(s,a) - (\mathcal{T}f)(s,a))^2\right]} + \frac{H^2\log(|\calF|/\delta)}{n}
    \\&\le \underbrace{\E_{(s,a,r,s') \sim \calD}\left[\left(f(s,a) - r - f(s')\right)^2\right] -   \min_{g \in \mathcal{G}}\E_{(s,a,r,s') \sim \calD}\left[\left(g(s,a) - r - f(s')\right)^2\right]}_{\le \epsilon_{\rm stat}}
    \\& \qquad \quad + H\sqrt{\frac{\log(|\calF|/\delta)}{n}\E_{(s,a) \sim \calD}\left[(f(s,a) - (\mathcal{T}f)(s,a))^2\right]} + \frac{H^2\log(|\calF|/\delta)}{n}. 
\end{align*}
Thus, we have
\begin{align*}
    &\E_{(s,a) \sim \calD}\left[(f(s,a) - (\mathcal{T}f)(s,a))^2\right] \\
    &\le H\sqrt{\frac{\log(|\calF|/\delta)}{n}\E_{(s,a) \sim \calD}\left[(f(s,a) - (\mathcal{T}f)(s,a))^2\right]} + \order\left(\frac{H^2\log(|\calF||\mathcal{G}|/\delta)}{n}\right). 
\end{align*}
Solving the equation gives
\begin{align}
     \E_{(s,a) \sim \calD}\left[(f(s,a) - (\mathcal{T}f)(s,a))^2\right] \le \order\left(\frac{H^2\log(|\calF||\mathcal{G}|/\delta)}{n}\right). 
\label{eq:empirical-square}
\end{align}
Finally, with probability $1-\delta$, for any $f \in \calF_{\rm conf}$, we have
\begin{align*}
      \Dav^{\pi^\star}(f\| M^\star) &= \left(\E_{(s,a) \sim d^{\pi^\star}}\left[f(s,a) - R(s,a) - \E_{s' \sim P(\cdot|s,a)}\left[f(s')\right]\right]\right)^2 
      \\&\le \E_{(s,a) \sim d^{\pi^\star}}\left[\left(f(s,a) - R(s,a) - \E_{s' \sim P(\cdot|s,a)}\left[f(s')\right]\right)^2 \right]
      \\&\le HC^{\pi^\star}\E_{(s,a) \sim \mu}\left[\left(f(s,a) - R(s,a) - \E_{s' \sim P(\cdot|s,a)}\left[f(s')\right]\right)^2 \right]
      \\&\le \order\left(\frac{C^{\pi^\star}  H^3\log(|\calF||\mathcal{G}|/\delta)}{n}\right).  \tag{Combine \pref{eq:empirical-square} and \pref{lem:freedman}}
\end{align*}
\end{proof}

\begin{proof}[Proof of \pref{lem: weight case}] We first prove $Q^\star \in \calF_{\rm conf}$. From Hoeffding's inequality, with probability $1-\delta$, for any $w \in \mathcal{W}$, we have 
\begin{align*}
&\left|\E_{(s,a,r,s') \sim \calD}\left[w(s,a)\left(Q^\star(s,a) -r - V^\star(s')\right)\right]\right|
\\&\le \left|\E_{(s,a) \sim \calD}\left[w(s,a)\left(Q^\star(s,a) - (\mathcal{T}Q^\star)(s,a)\right)\right]\right| + B_{\mathcal{W}}H \sqrt{\frac{2\log(|\mathcal{W}|/\delta)}{n}}
\\&= B_{\mathcal{W}}H \sqrt{\frac{2\log(|\mathcal{W}|/\delta)}{n}}  \\
&\leq \epsilon_\stat. 
\end{align*}
Thus, $Q^\star\in\calF_\conf$. 
Now we prove the guarantee of $\Dav^{\pi^\star}$. For any $f \in \calF_{\rm conf}$, we have
\begin{align*}
    &\left|\E_{(s,a) \sim d^{\pi^\star}}\left[f(s,a) - R(s,a) - \E_{s' \sim P(\cdot|s,a)}\left[f(s')\right]\right]\right|
    \\&= H \left|\E_{(s,a) \sim \mu}\left[w^{\pi^\star}(s,a)\left(f(s,a) - R(s,a) - \E_{s' \sim P(\cdot|s,a)}\left[f(s')\right]\right)\right]\right|
    \\&\le H \left|\E_{(s,a,r,s') \sim \calD}\left[w^{\pi^\star}(s,a)\left(f(s,a) - r - f(s')\right)\right]\right| + \order\left(B_{\mathcal{W}}H^2\sqrt{\frac{\log(|\calF|/\delta)}{n}}\right) \tag{Hoeffding's Inequality}
    \\&\le H\max_{w \in \mathcal{W}}\left|\E_{(s,a,r,s') \sim \calD}\left[w(s,a)\left(f(s,a) - r - f(s')\right)\right]\right| + \order\left(B_{\mathcal{W}}H^2\sqrt{\frac{\log(|\calF|/\delta)}{n}}\right)
    \\&\le  \order\left(H\epsilon_{\rm stat}\right) \tag{Use the definition of $\epsilon_{\rm stat}$}
\end{align*}
Thus,
\begin{align*}
     \Dav^{\pi^\star}(f\| M^\star) &= \left(\E_{(s,a) \sim d^{\pi^\star}}\left[f(s,a) - R(s,a) - \E_{s' \sim P(\cdot|s,a)}\left[f(s')\right]\right]\right)^2 \le \order\left(H^2\epsilon_{\rm stat}^2\right). 
\end{align*}
\end{proof}

\begin{proof}[Proof of \pref{lem: Bellman rank case}]
   \begin{align*}
        \Dav^{\pi^\star}(f\|M^\star) 
        &= \left(\E^{\pi^\star, M^\star}\left[ f(s,a) - R(s,a) - \E_{s'\sim P(\cdot|s,a)}[f(s')] \right]\right)^2 \\
        &= \inner{X(\pi^\star), W(f)}^2   \tag{by \pref{assum: bellman rank}}\\
    &\leq \left(X(\pi^\star)^\top \Sigma_\mubar^\dagger X(\pi^\star)\right)\left(W(f)^\top \Sigma_\mubar W(f)\right)   \tag{Cauchy-Schwarz}\\
    &= \Cbar^{\pi^\star}  \E_{\pi\sim \mubar}\left[\inner{X(\pi), W(f)}^2\right] \tag{\pref{def: policy cover}}\\
    &\leq \Cbar^{\pi^\star}   \E_{\pi\sim \mubar}\left[\left( \E^{\pi, M^\star}\left[f(s,a) - R(s,a) - \E_{s'\sim P(\cdot|s,a)}[f(s')]\right]\right)^2\right] \\
    &= \Cbar^{\pi^\star} \E_{\pi\sim \mubar} [\Dav^\pi(f\|M^\star)]. \numberthis \label{eq: DDD}
   \end{align*}
   On the other hand, by the definition of $L_\brr$, we have 
   \begin{align*}
       \E[L_\brr(f)] 
       &= \E_{\pi\sim \mubar}\left[ \E^{\pi, M^\star}[f(s,a) - r - f(s')]\cdot \E^{\pi, M^\star}[f(s,a) - r - f(s')]  \right] \\
       &= \E_{\pi\sim \mubar}\left[ \Dav^\pi(f\|M^\star) \right]   \numberthis \label{eq: BLinear concen}
   \end{align*}
   and by Hoeffding's inequality, with probability at least $1-\delta$,  
   \begin{align*}
       \left|L_\brr(f) - \E_{\pi\sim \mubar}\left[ \Dav^\pi(f\|M^\star) \right]\right| \leq H^2\sqrt{\frac{\log(2|\calF|/\delta)}{2n}}.
   \end{align*}
   Under this event, 
   \begin{align*}
       L_\brr(Q^\star) \leq \E_{\pi\sim \mubar}\left[ \Dav^\pi(Q^\star\|M^\star) \right] + H^2\sqrt{\frac{\log(2|\calF|/\delta)}{2n}} = H^2\sqrt{\frac{\log(2|\calF|/\delta)}{2n}}
   \end{align*}
   and thus $Q^\star\in\calF_\conf$. Also, by \pref{eq: DDD} we have 
   \begin{align*}
       \max_{f\in\calF_\conf} \Dav^{\pi^\star}(f\|M^\star)  
       &\leq \Cbar^{\pi^\star} \max_{f\in\calF_\conf}\E_{\pi\sim \mubar} [\Dav^\pi(f\|M^\star)] \\
       &\leq \Cbar^{\pi^\star} \max_{f\in\calF_\conf}\left(L_{\brr}(f) + H^2\sqrt{\frac{\log(2|\calF|/\delta)}{2n}}\right)  \tag{by \pref{eq: BLinear concen}} \\
       &\leq \Cbar^{\pi^\star} \cdot 2H^2\sqrt{\frac{\log(2|\calF|/\delta)}{2n}}.  \tag{by the definition of $\calF_\conf$}
   \end{align*}
   
\end{proof}
\subsection{Low Bellman rank setting: lower bounds without policy feature coverage or without double policy samples}\label{app: jia lower }
For the necessity of the dependence on $\Cbar^{\pi^\star}$ and the necessity of double policy samples, we rely on an extension to the lower bound construction in \citep{jia2024offline}. The two claims are proven in \pref{lem: first claim} and \pref{lem: second claim}, respectively. 


\subsubsection{Reusing the hard instance in \cite{jia2024offline}}\label{app: construct hard}

Specifically, we build upon the construction of \cite{jia2024offline} in the admissible setting ($\mu=\frac{1}{H}d^{\pi_b}$ for some $\pi_b$) described in their Theorem~4.1. 
Their construction is illustrated in \pref{fig: jia construction}, with details described in the caption. Moreover, their construction has the following properties: 
\begin{itemize}[leftmargin=11pt, topsep=6pt]
   \setlength{\itemsep}{0pt}
  \setlength{\parskip}{0pt}
  \setlength{\parsep}{0pt}
    \item They focus on \emph{policy evaluation}. There are two actions. The policy to be evaluated is $\pi_e = (1,0)$, i.e., always choosing action~1. The learner is given a function set $\calF=\{f_1, f_2\}$ \mbox{such that $Q^{\pi_e}\in\calF$.  }
    \item The offline data distribution $\mu$ satisfies $\mu=\frac{1}{H}d^{\pi_b}$ where $\pi_b = (\frac{1}{H^2}, 1-\frac{1}{H^2})$. 
    \item Their construction satisfies $C^{\pi}=O(H^3)$ for any policy $\pi$. 
    \item All rewards and $Q^\pi$ values are in the range of $[-1,1]$, 
\end{itemize}

\begin{figure}[!htbp]
    \centering
    \hspace{-40pt}\includegraphics[width=0.8\textwidth]{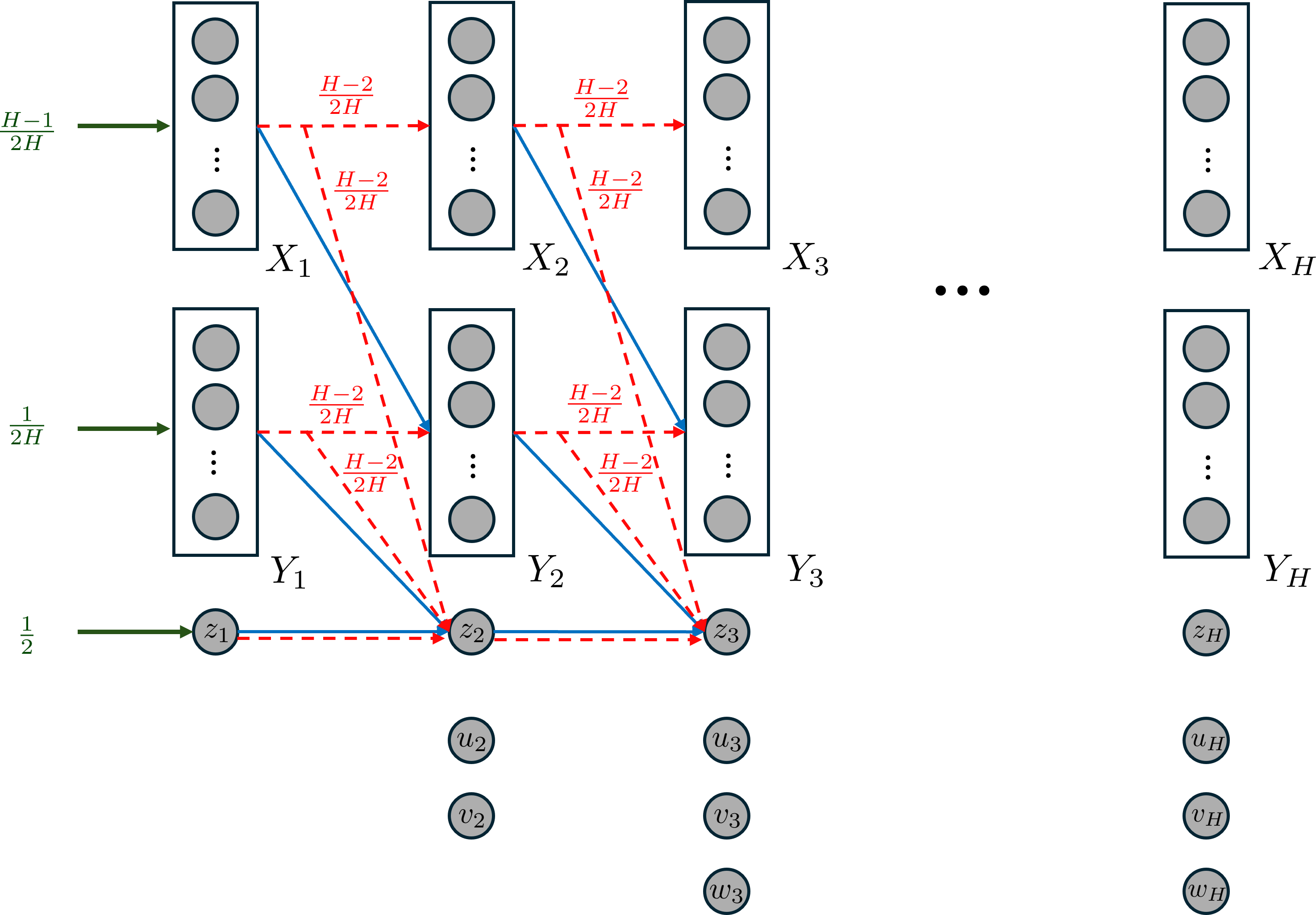}
    \caption{The construction of \cite{jia2024offline}. The numbers on the left of the green arrows specify the initial state distribution. There are two actions. Taking either action on state in $X_h\cup Y_h \cup\{z_h\}$ leads to transitions to $u_{h+1}$ and $v_{h+1}$ with probabilities $P(u_{h+1}|s,a) + P(v_{h+1}|s,a) = \frac{2}{H}$ which are not specified in the figure.  Blue arrow specifies the transition if taking action $1$ besides those to $u_{h+1}$ and $v_{h+1}$, and red arrow if taking action $2$. The numbers on the red arrows are the transition probabilities. The blue and red arrows without a number on it have a transition probability $\frac{H-2}{H}$.  Every transition to group $X_h$ or $Y_h$ results in a uniform distribution over that group. On $u_h$ and $v_h$ and $w_h$, taking any action leads to a deterministic transition to $w_{h+1}$. 
    }\label{fig: jia construction}

    \vspace{20pt}
    \centering
    \hspace{-40pt}\includegraphics[width=0.85\textwidth]{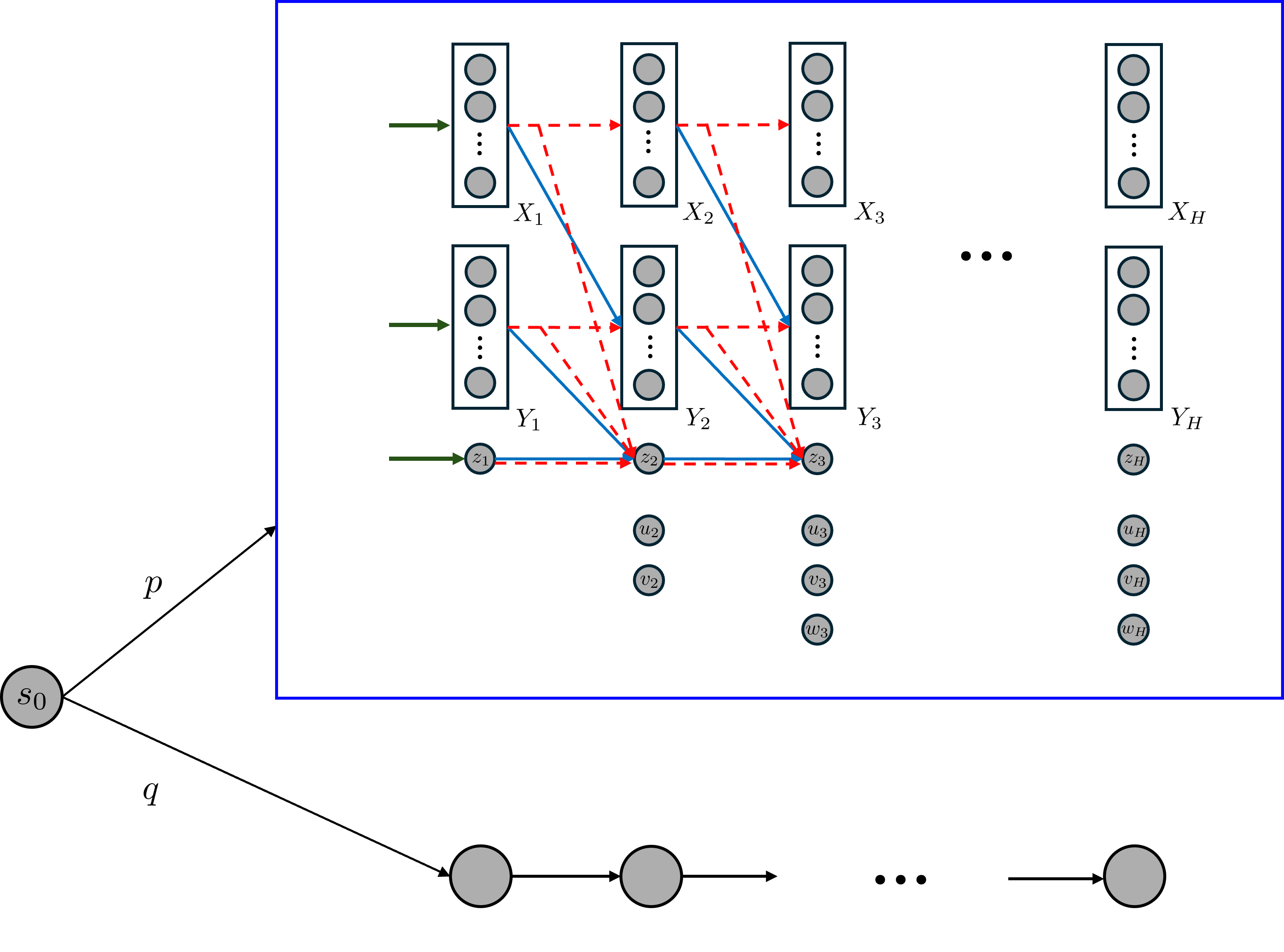}
    \caption{A hard instance for policy optimization. The blue box is \cite{jia2024offline}'s construction (\pref{fig: jia construction}) with modified reward of action 2 (see \pref{app: construct hard} for a description of the construction). 
    }\label{fig: jia 2}
\end{figure}

The construction depends on a parameter $\epsilon\in[0,\frac{1}{15}]$ which makes $|\E_{s\sim \rho}[f_1(s)]- \E_{s\sim \rho}[f_2(s)]|=2\epsilon$, where $\rho$ is the initial state distribution. They prove that to estimate $\E_{s\sim \rho}[V^{\pi_e}(s)]$ up to precision~$\epsilon$, the learner must use at least $2^H/\epsilon$ samples drawn from $\mu$.  
Below we describe how to turn this into a hard case for \emph{policy optimization}. 
\begin{itemize}[leftmargin=11pt, topsep=6pt]
   \setlength{\itemsep}{0pt}
  \setlength{\parskip}{0pt}
  \setlength{\parsep}{0pt}
    \item Subtract the instantaneous reward of action~2 by $10$ on every state.  As this makes action 2 very bad, the optimal policy on any state is action~1. 
    \item Based on their function class $\calF$ for $Q^{\pi_e}$, define the new function class $\calF' = \{f_1', f_2'\}$ such that $f_i'(s,1) = f_i(s,1)$ and $f_i'(s,2) = f_i(s,2) - 10$ for $i=1,2$. Below, we argue that $Q^\star\in \calF'$ in the modified environments.  In the modified environment, since action 1 is always optimal, $Q^\star(s,1)$ is the value of always taking action 1, which is the same as $Q^{\pi_e}(s,1)$ in the original environment. For the same reason, in the modified environment, $Q^\star(s,2)$ is the value of taking action~2 for one step and taking action~1 thereafter. Therefore, $Q^\star(s,2) = -10 + R(s,2) + \E_{s'\sim P(\cdot|s,2)}[V^{\pi_e}(s')] = Q^{\pi_e}(s,2)-10$, where $R$,  $V^{\pi_e}, Q^{\pi_e}$ are the reward and values in the original environment. Since $Q^{\pi_e}\in\calF$ in the original environment, this implies $Q^\star\in\calF'$ in the modified environment. 
    \item This mere modification of offsetting action~2's reward by a constant does not make evaluating $\pi_e$ any easier. Since $\pi^\star$ in the modified environment is the same as $\pi_e$ in the original environment, and their state values are the same, it becomes a hard instance of evaluating $\E_{s\sim \rho}[V^\star(s)]$. 
    \item Finally, we add one additional state $s_0$ with two actions. Choosing $p$ leads to zero instantaneous reward and transitions to the MDP described above; choosing $q$ leads to a deterministic reward of $\avg=\E_{s\sim \rho}\big[\frac{f_1'(s)+f_2'(s)}{2}\big]$ and a chain of state with single action and zero reward until the end of the episode (\pref{fig: jia 2}).   Augment the functions with $f_1'(s_0, p) = \E_{s\sim \rho}[f_1'(s)]$ and $f_2'(s_0, p) = \E_{s\sim \rho}[f_2'(s)]$ and $f_1'(s_0,q) = f_2'(s_0,q)=\avg$. Define behavior policy $\pi_b(\cdot|s_0)=(\frac{1}{2},\frac{1}{2})$ and keep it unchanged in other states.  
    \item Let $\epsilon=|\E_{s\sim \rho}[f_1'(s)] - \E_{s\sim \rho}[f_2'(s)]|=\frac{1}{15}$. Then the two actions on $s_0$ has a gap of $|Q^\star(s_0,p) - Q^\star(s_0,q)|=\frac{1}{2} |\E_{s\sim \rho}[f_1'(s)] - \E_{s\sim \rho}[f_2'(s)]|=\frac{1}{2}\cdot\frac{1}{15}=\frac{1}{30}$. Therefore, in order to find an $\frac{1}{60}$-optimal policy, the learner must distinguish the two worlds, i.e., estimating $Q^\star(s_0,p) = \E_{s\sim \rho}[V^\star(s)]$ up to an accuracy of $\frac{1}{30}$, which is proven to be hard by \cite{jia2024offline} (needs $\Omega(2^H)$ samples).   
\end{itemize}

\subsubsection{Relate the hard instance to our setting}
We will use the instance in  \pref{app: construct hard} to show that both \emph{policy feature coverage} and \emph{double policy samples} are necessary for polynomial sample complexity in low $Q$-Bellman rank settings. 

We first show that the hard instance in \pref{app: construct hard} is a linear $Q^\star/V^\star$ setting \citep{du2021bilinear} (\pref{lem: linear Q* V*}), which is further an example of an MDP with low $Q$-Bellman rank (\pref{lem: imply Q beallman}). 

\begin{definition}[Linear $Q^\star/V^\star$ setting]\label{def: linear Q* V*}
   The learner has known features $\phi: \calS\times \calA\to \mathbb{R}^d$ and $\xi: \calS\to \mathbb{R}^d$ such that $Q^\star(s,a) = \phi(s,a)^\top \theta^\star$ and $V^\star(s) = \xi(s)^\top w^\star$ for some unknown $\theta^\star$ and $w^\star$. 
\end{definition}

\begin{lemma}\label{lem: linear Q* V*}
   The instance described in \pref{app: construct hard} is a linear $Q^\star/V^\star$ setting (\pref{def: linear Q* V*}). 
\end{lemma}
\begin{proof}
   In the construction, the learner has a function set  $\calF=\{f_1, f_2\}$ such that $Q^\star\in\calF$. Then 
   \begin{align*}
        \phi(s,a) = \begin{bmatrix}
            f_1(s,a) \\
            f_2(s,a)
        \end{bmatrix},\qquad
        \xi(s) = \begin{bmatrix}
            f_1(s) \\
            f_2(s)
        \end{bmatrix} =  \begin{bmatrix}
            \max_a f_1(s,a) \\
            \max_a f_2(s,a)
        \end{bmatrix}
    \end{align*}
    satisfies \pref{def: linear Q* V*} with $d=2$.  
\end{proof}

\begin{lemma}\label{lem: imply Q beallman} Linear $Q^\star/V^\star$ setting has low $Q$-Bellman rank. 
\end{lemma}
\begin{proof}
    In the linear $Q^\star/V^\star$ setting, let $f$ be the function associated with weight $(\theta, w)$. Then 
    \begin{align*}
        &\E_{(s,a)\sim d^\pi}\left[f(s,a) - R(s,a) - \E_{s'\sim P(\cdot|s,a)}[f(s')]\right] \\
        &=\E_{(s,a)\sim d^\pi}\left[f(s,a) - R(s,a) - \E_{s'\sim P(\cdot|s,a)}[f(s')]\right] \\
        &\qquad \quad - \E_{(s,a)\sim d^\pi}\left[Q^\star(s,a) - R(s,a) - \E_{s'\sim P(\cdot|s,a)}[V^\star(s')]\right] \tag{subtracting zero}\\
        &= \E_{(s,a)\sim d^\pi}\left[\phi(s,a)^\top\theta - R(s,a) - \E_{s'\sim P(\cdot|s,a)}[\xi(s')^\top w]\right] \\
        &\qquad \quad - \E_{(s,a)\sim d^\pi}\left[\phi(s,a)^\top \theta^\star - R(s,a) - \E_{s'\sim P(\cdot|s,a)}[\xi(s')^\top w^\star]\right] \\
        &= \E_{(s,a)\sim d^\pi}\left[\phi(s,a)^\top (\theta-\theta^\star) - \E_{s'\sim P(\cdot|s,a)}[\xi(s')^\top (w-w^\star)]\right]\\
        &= \inner{X(\pi), W(f)}
    \end{align*}
    where $X(\pi) = \E_{(s,a)\sim d^\pi}[\phi(s,a), \E_{s'\sim P(\cdot|s,a)}[\xi(s')]]$ and $
        W(f) = [\theta-\theta^\star, w-w^\star]$. 
\end{proof}

\begin{lemma}\label{lem: first claim}
    In the low $Q$-Bellman rank setting, it is impossible to achieve $\E[J(\pi^\star) - J(\hat{\pi})]\leq \epsilon$ with only $\poly(d,H,\epsilon^{-1}, \log|\calF|, C^{\pi^\star})$ samples, even with access to double policy samples (\pref{assum: double sample}) and with constant value gap $\Delta_{Q^\star}=\Theta(1)$. 
\end{lemma}
\begin{proof}
   In the instance constructed in \pref{app: construct hard} (which has low $Q$-Bellman rank by \pref{lem: linear Q* V*} and \pref{lem: imply Q beallman}), the offline distribution is admissible, i.e., $\mu=\frac{1}{H}d^{\pi_b}$ for some fixed policy $\pi_b$. Hence, it satisfies the double policy sample assumption---any two $(s,a,r,s')$ tuples in the dataset are generated by the same policy. Furthermore, $C^{\pi^\star}=O(H^3)$ and $\log|\calF|=\log 2$ in the construction. It is also easy to see that $\Delta_{Q^\star}=\Theta(1)$ in the constructed instance as we make action~2 highly sub-optimal compared to action~1 on all states. However, as proven by \cite{jia2024offline}, polynomial sample complexity is impossible. 
\end{proof}

\begin{lemma}\label{lem: second claim}
    In the low $Q$-Bellman rank setting, it is impossible to achieve \mbox{$\E[J(\pi^\star) - J(\hat{\pi})]\leq \epsilon$} with only $\poly(d,H,\epsilon^{-1}, \log|\calF|, \Cbar^{\pi^\star})$ samples if the learner only gets access to single policy samples $(s,a,r,s')$ but not double policy samples (\pref{assum: double sample}), even with constant value gap $\Delta_{Q^\star}=\Theta(1)$.  
\end{lemma}
\begin{proof}
    Below, we construct $\mubar$ so that the process  ``sample $ \pi\sim \mubar$ and then sample $(s,a,r,s')\sim d^\pi$'' is equivalent to ``sample $(s,a,r,s')\sim d^{\pi_b}$'' for the $\pi_b$ defined in \pref{app: construct hard}. The latter has been proven to be hard in \citep{jia2024offline}. If we can further show that $\mubar$ makes $\Cbar^{\pi^\star}=X(\pi^\star)^\top \Sigma_{\mubar}^\dagger X(\pi^\star)$ small, then the impossibility result follows from that of \citep{jia2024offline}. 

    Recall that $\pi_b$ is a randomized policy that is equal to $(\frac{1}{2}, \frac{1}{2})$ on state $s_0$, and equal to $\pi_b=(\frac{1}{H^2}, 1-\frac{1}{H^2})$ in all other states. We define $\mubar$ to be a distribution over \emph{deterministic policies} in the following way: 
    \begin{align*}
        \mubar(\delta_{a_0}\circ \delta_{a_1}\circ \delta_{a_2}\circ \cdots \circ \delta_{a_H}) = \frac{1}{2}\prod_{h=1}^H \left(\frac{1}{H^2}\right)^{\ind\{a_h=1\}} \left(1-\frac{1}{H^2}\right)^{\ind\{a_h=2\}}
    \end{align*}
    where $\delta_a$ is the one-hot action distribution concentrating on action $a$,  $a_0$ is the action taken on $s_0$, and $a_h$ where $h\geq 1$ is the action taken on step $h$ in the trajectory. 
    That means, this set of deterministic policies only determine their action based on the step index and ignore the state.  Clearly, sampling $\pi\sim \mubar$ and drawing one trajectory from $\pi$ induces the same state-action distribution as executing~$\pi_b$. 

    Next, we bound $\Cbar^{\pi^\star}=X(\pi^\star)^\top \Sigma_{\mubar}^\dagger X(\pi^\star)$. One observation in the construction of \pref{fig: jia construction} is that if a policy choose action 1 in both layers $h=1$ and $h=2$, then the state-action distribution in the rest of the trajectory is policy-independent. Indeed, if a policy chooses action 1 in layers $h=1$ and $h=2$, then it must lands in states $\{z_h, u_h, v_h, w_h\}$ for $h\geq 3$, on which choosing action 1 and 2 induces the same transition. 

    Therefore, 
    \begin{align*}
    \Sigma_{\mubar}
    &= \E_{\pi\sim \mubar} [X(\pi)X(\pi)^\top] = \E_{\pi\sim \mubar} \begin{bmatrix}
        \E_{(s,a)\sim d^\pi}[\phi(s,a)]  \\
        \E_{(s,a)\sim d^\pi}\E_{s'\sim P(\cdot|s,a)} [\xi(s')]
    \end{bmatrix}\begin{bmatrix}
        \E_{(s,a)\sim d^\pi}[\phi(s,a)]  \\
        \E_{(s,a)\sim d^\pi}\E_{s'\sim P(\cdot|s,a)} [\xi(s')]
    \end{bmatrix}^\top \\
    &\succeq \sum_{\pi} \mubar(\pi)\ind\{\pi_1=1, \pi_2=1\}\begin{bmatrix}
        \E_{(s,a)\sim d^\pi}[\phi(s,a)]  \\
        \E_{(s,a)\sim d^\pi}\E_{s'\sim P(\cdot|s,a)} [\xi(s')]
    \end{bmatrix}\begin{bmatrix}
        \E_{(s,a)\sim d^\pi}[\phi(s,a)]  \\
        \E_{(s,a)\sim d^\pi}\E_{s'\sim P(\cdot|s,a)} [\xi(s')]
    \end{bmatrix}^\top \\
    &\succeq \frac{1}{2}\cdot \frac{1}{H^2}\cdot \frac{1}{H^2}\begin{bmatrix}
        \E_{(s,a)\sim d^{\pi^\star}}[\phi(s,a)]  \\
        \E_{(s,a)\sim d^{\pi^\star}}\E_{s'\sim P(\cdot|s,a)} [\xi(s')]
    \end{bmatrix}\begin{bmatrix}
        \E_{(s,a)\sim d^{\pi^\star}}[\phi(s,a)]  \\
        \E_{(s,a)\sim d^{\pi^\star}}\E_{s'\sim P(\cdot|s,a)} [\xi(s')]
    \end{bmatrix}^\top \\
    &= \frac{1}{2H^4} X(\pi^\star)X(\pi^\star)^\top, 
\end{align*}
implying that $\Cbar^{\pi^\star}=X(\pi^\star)^\top \Sigma_{\mubar}^\dagger X(\pi^\star)$ is polynomial in $H$. However, by the lower bound of \cite{jia2024offline}, polynomial sample complexity is impossible. 
\end{proof}

\section{Conservative $Q$-Learning}\label{app: CQL}
Conservative $Q$-Learning (CQL) is one of the most widely adopted baselines in offline RL \citep{kumar2020conservative}.
We consider the following form of CQL: 
\begin{align*}
   \textstyle \hat{f} = \argmin_{f\in\calF} \E_{(s,a)\sim \calD} [f(s) - f(s, a) ] + \lambda \E_{(s,a)\sim \calD} \left[ \big(f(s,a) - [\hat{\calT}f](s,a)\big)^2\right]   \numberthis \label{eq: CQL algorithm}
\end{align*}
where $[\hat{\calT}f] = \argmin_{g\in\calG}\E_{(s,a, r, s')\sim \calD} \left[(g(s,a) - r - f(s'))^2\right]$ and $\calG$ is as defined in \pref{assum: Bellman complete}. Furthermore, CQL uses regularizer $\psi(p; s) = \alpha\KL(p, \piref(\cdot|s))$ as in \pref{sec: regularized MDP}. We have: 
\begin{theorem}\label{thm: CQL thm}
    Assume $Q^\star$-realizability and Bellman completeness . Also, assume $\mu$ is admissible:  $\sum_{(s,a)\in\calS_h\times \calA}\mu(s,a)P(s'|s,a) = \mu(s')$ for $s'\in\calS_{h+1}$.  \mbox{Then CQL \pref{eq: CQL algorithm} has with probability $1-\delta$,}
    \begin{align*}
       J(\pi^\star) - J(\pi_{\hat{f}})\leq \order\left( \frac{H^4}{\alpha^2} \left(\frac{\lambda H^2\log(|\calF||\calG|/\delta)}{n} + \frac{HC^{\pi^\star}}{\lambda}\right)\right). 
    \end{align*}
\end{theorem}
CQL does not maintain a confidence set, but uses a penalty (second term in \pref{eq: CQL algorithm}) to ensure that $\hat{f}$ conforms with the Bellman optimality equation. The first term in \pref{eq: CQL algorithm} implicitly encourages value pessimism, which we elaborate below. As the second term in \pref{eq: CQL algorithm} forces $\E_{(s,a)\sim \mu}[f(s,a)-R(s,a) - \E_{s'\sim P(\cdot|s,a)}[f(s')]]\approx 0$, we have 
$$
    \E_{(s,a)\sim \mu}[f(s)-\E_{s'\sim P(\cdot|s,a)}[f(s')]] \approx  \E_{(s,a)\sim \mu}[R(s,a) + f(s) - f(s,a)]. 
$$
The left-hand side is the difference of expected value of $f$ between consecutive layers, while the right-hand side is the reward plus $\E_{(s,a)\sim \mu}[f(s)-f(s,a)]$ that reflects how much more $f$ values out-of-sample actions than in-sample actions. When this is large, the expected value of $f$ grows more rapidly when propagating over layers. To ensure pessimism (i.e., small $f(s_1)$), \pref{eq: CQL algorithm} adds $\E_{(s,a)\sim \mu}[f(s)-f(s,a)]$ as part of the loss. This pessimistic behavior is formalized in \pref{lem: approx pessimi}, and other parts of the proof are similar to those in \pref{sec: regularized MDP}. Before we prove \pref{thm: CQL thm}, we first introduce \pref{lem: hatT acc} and \pref{lem: approx pessimi} to bound the error of estimation and pessimistic decision.

\begin{lemma}\label{lem: hatT acc}
    With probability at least $1-\delta$, for any $f\in\calF$, 
    \begin{align*}
        \E_{(s,a)\sim \calD}\left[\left((\hatcalT f)(s,a)- (\calT f)(s,a)\right)^2\right] \leq \order\left(\frac{H^2\log(|\calF||\calG|/\delta)}{n}\right). 
    \end{align*}
\end{lemma}
\begin{proof}
   By the definition of $\hatcalT f$, we have 
   \begin{align*}
       0 &\geq \E_{(s,a, r, s')\sim \calD}\left[\left( (\hatcalT f)(s,a) - r  - f(s') \right)^2 - \left( (\calT f)(s,a) - r  - f(s') \right)^2 \right] \\
       &= \E_{(s,a,r,s')\sim \calD}\left[ \left((\hatcalT f)(s,a) - (\calT f)(s,a)\right)\left( (\hatcalT f)(s,a) + (\calT f)(s,a) -2r - 2f(s')\right) \right] \\
       &= \E_{(s,a)\sim \calD}\left[ \left((\hatcalT f)(s,a) - (\calT f)(s,a)\right)^2 \right]\\
       &\qquad + \E_{(s,a,r,s')\sim \calD} \left[\left((\hatcalT f)(s,a) - (\calT f)(s,a)\right)\left( 2(\calT f)(s,a) -2r - 2f(s')\right) \right]
   \end{align*}
   Note that $\E[r + f(s')|s,a] = R(s,a) + \E_{s'\sim P(\cdot|s,a)}[f(s')] = (\calT f)(s,a)$. By Freedman's inequality, with probability at least $1-\delta$, 
   \begin{align*}
       &\left|\E_{(s,a,r,s')\sim \calD} \left[\left((\hatcalT f)(s,a) - (\calT f)(s,a)\right)\left( 2(\calT f)(s,a) -2r - 2f(s')\right) \right]\right|  \\
       &\lesssim H\sqrt{\frac{\log(|\calF||\calG|/\delta)}{n}\E_{(s,a)\sim \calD}\left[ \left((\hatcalT f)(s,a) - (\calT f)(s,a)\right)^2 \right]} + \frac{H^2\log(|\calF||\calG|/\delta)}{n}
   \end{align*}
   Combining inequalities above and denoting $X = \E_{(s,a)\sim \calD}\left[ \left((\hatcalT f)(s,a) - (\calT f)(s,a)\right)^2 \right]$, we get 
   \begin{align*}
       X\lesssim H\sqrt{X\cdot \frac{\log(|\calF||\calG|/\delta)}{n}} + \frac{H^2\log(|\calF||\calG|/\delta)}{n}, 
   \end{align*}
   implying that $X \lesssim \frac{H^2\log(|\calF||\calG|/\delta)}{n}$. 
\end{proof}

\begin{lemma}\label{lem: approx pessimi}
With probability at least $1-\delta$, 
    \begin{align*}
    \hat{f}(s_1) \leq J(\pi^\star) +  \order\left(\sqrt{\frac{\iota}{n}} + \frac{\lambda \iota}{n} + \frac{1}{\lambda}\right) - \E_{(s,a)\sim \calD}\left[\frac{\lambda}{2}\err(s,a;\hat{f})^2\right]   
\end{align*}
where $\err(s,a;f) := f(s,a) - (\hatcalT f)(s,a)$ and $\iota = H^2\log(|\calF||\calG|/\delta)$
\end{lemma}
\begin{proof}
In the proof we fix a failure probability $\delta$ and denote $\iota = H^2\log(|\calF||\calG|/\delta)$. 
The objective of the algorithm can be written as 
\begin{align*}
    &\E_{(s,a)\sim \calD} [f(s) - f(s,a)] + \lambda \E_{(s,a)\sim \calD} \left[ (f(s,a) - (\hatcalT f)(s,a))^2 \right] \\
    &= \E_{(s,a)\sim \calD}\left[f(s) - (\hatcalT f)(s,a)\right] - \E_{(s,a)\sim \calD}\left[f(s,a) - (\hatcalT f)(s,a)\right] + \lambda \E_{(s,a)\sim \calD} \left[ (f(s,a) - (\hatcalT f)(s,a))^2 \right] \\
    &= \E_{(s,a)\sim \calD}\left[f(s) - (\hatcalT f)(s,a)\right] + \E_{(s,a)\sim \calD}\left[-\err(s,a; f) + \lambda \err(s,a; f)^2\right]
\end{align*}
where $\err(s,a;f) := f(s,a) - (\hatcalT f)(s,a)$. Therefore, by the optimality of $\hat{f}$, 
\begin{align}
    &\E_{(s,a)\sim \calD} \left[\hat{f}(s) - (\hatcalT\hat{f})(s,a)\right] + \E_{(s,a)\sim \calD}\left[-\err(s,a; \hat{f}) + \lambda \err(s,a; \hat{f})^2\right] \nonumber\\
    & \leq \E_{(s,a)\sim \calD} \left[V^\star(s) - (\hatcalT Q^\star)(s,a)\right] + \E_{(s,a)\sim \calD}\left[-\err(s,a; Q^\star) + \lambda \err(s,a; Q^\star)^2\right] \nonumber \\
    & \leq \E_{(s,a)\sim \calD} \left[V^\star(s) - (\hatcalT Q^\star)(s,a)\right] + \order\left(\sqrt{\frac{\iota}{n}} + \frac{\lambda \iota}{n}\right) 
    \label{eq: optimality bound}
\end{align}
where in the last inequality we use that 
\begin{align*}
   \E_{(s,a)\sim \calD}\left[ \err(s,a; Q^\star)^2\right] 
   &= \E_{(s,a)\sim \calD}\left[ \left(Q^\star(s,a) - (\hatcalT Q^\star)(s,a)\right)^2\right] \\
   &= \E_{(s,a)\sim \calD}\left[ \left((\calT Q^\star)(s,a) - (\hatcalT Q^\star)(s,a)\right)^2\right] \lesssim \frac{\iota}{n}.   \tag{by \pref{lem: hatT acc}}
\end{align*}
To further bound other terms in \pref{eq: optimality bound}, note that by \pref{lem: hatT acc} and Hoeffding's inequality, we have for all $f$,  
\begin{align}
    &\left|\E_{(s,a)\sim \calD}\Big[f(s) - (\hatcalT f)(s,a)\Big] - \E_{(s,a)\sim \mu}\Big[f(s) - (\calT f)(s,a)\Big]\right|  \nonumber \\
    &\leq \left| \E_{(s,a)\sim \calD}\Big[ (\hatcalT f)(s,a) - (\calT f)(s,a)\Big]  \right| + \left|\E_{(s,a)\sim \calD}\Big[f(s) - (\calT f)(s,a)\Big] - \E_{(s,a)\sim \mu}\Big[f(s) - (\calT f)(s,a)\Big]\right| \nonumber \\
    &\lesssim \sqrt{\frac{\iota}{n}}. \label{eq: concentrate} 
\end{align}
Using this in \pref{eq: optimality bound} and recalling $(\calT f)(s,a)=R(s,a) + \E_{s'\sim P(\cdot|s,a)}[f(s')]$ yields 
\begin{align*}
   &\E_{(s,a)\sim \mu} \left[\hat{f}(s) - R(s,a) - \E_{s'\sim P(\cdot|s,a)}[\hat{f}(s')]\right] + \E_{(s,a)\sim \calD}\left[-\err(s,a; \hat{f}) + \lambda \err(s,a; \hat{f})^2\right] \nonumber\\
   &\quad \leq \E_{(s,a)\sim \mu} \left[V^\star(s) - R(s,a) - \E_{s'\sim P(\cdot|s,a)}[V^\star(s')]\right] + \order\left(\sqrt{\frac{\iota}{n}} + \frac{\lambda \iota}{n}\right) 
\end{align*}
which rearranges to 
\begin{align*}
    \hat{f}(s_1) + \E_{(s,a)\sim \calD}\left[-\err(s,a; \hat{f}) + \lambda \err(s,a; \hat{f})^2\right] \leq V^\star(s_1) +  \order\left(\sqrt{\frac{\iota}{n}} + \frac{\lambda \iota}{n}\right) 
\end{align*}
using the assumption that $\mu$ is admissible. 
Finally, using AM-GM to bound 
\begin{align*}
    \err(s,a; \hat{f}) - \lambda \err(s,a; \hat{f})^2 
    &= \err(s,a; \hat{f}) - \frac{\lambda}{2} \err(s,a; \hat{f})^2 - \frac{\lambda}{2} \err(s,a; \hat{f})^2 \\
    &\leq \frac{1}{2\lambda}  - \frac{\lambda}{2} \err(s,a; \hat{f})^2
\end{align*}
finishes the proof.


\end{proof}

\begin{proof}[Proof of \pref{thm: CQL thm}]
First, observe that 
\begin{align*}
    &\E^{\pi^\star, M^\star}\left[\hat{f}(s,a) - R(s,a) - \E_{s'\sim P(\cdot|s,a)}[\hat{f}(s')]\right] \\
    &= \E^{\pi^\star, M^\star}\big[\hat{f}(s,a) -\psi(\pi^\star;s) - \hat{f}(s)\big] + \hat{f}(s_1) - J(\pi^\star) \tag{same calculation as \pref{eq: same calc} and \pref{eq: temptemp}} \\
    &\leq \E^{\pi^\star, M^\star}\big[\hat{f}(s,a) -\psi(\pi^\star;s) - \hat{f}(s)\big] - \E_{(s,a)\sim \calD}\left[\frac{\lambda}{2}\err(s,a;\hat{f})^2\right] + \order\left(\sqrt{\frac{\iota}{n}} + \frac{\lambda \iota}{n} + \frac{1}{\lambda}\right)  \tag{by \pref{lem: approx pessimi}}
\end{align*}
where $\iota=H^2\log(|\calF||\calG|/\delta)$. This implies 
\begin{align*}
   &\Dav^{\pi^\star}(\hat{f}\|M^\star)\\ 
   &= \left(\E^{\pi^\star, M^\star}\left[\hat{f}(s,a) - R(s,a) - \E_{s'\sim P(\cdot|s,a)}[\hat{f}(s')]\right] \right)^2 \\
   &\geq \frac{1}{2}\left(\E^{\pi^\star, M^\star}\big[\hat{f}(s,a) -\psi(\pi^\star;s) - \hat{f}(s)\big]\right)^2 + \frac{\lambda^2}{8}\left(\E_{(s,a)\sim \calD}\left[\err(s,a;\hat{f})^2\right] \right)^2 - \order\left(\frac{\iota}{n} + \frac{\lambda^2\iota^2}{n^2} +\frac{1}{\lambda^2}\right)
\end{align*}
where we use that $(X+Y+Z)^2 \geq \frac{1}{2}(X+Y)^2 - [Z]_+^2 \geq \frac{1}{2}X^2 + \frac{1}{2}Y^2 - [Z]_+^2$ if $X<0$, $Y<0$. 
Thus, 
    \begin{align*}
        &J(\pi^\star) - J(\pi_{\hat{f}}) \\
        &= J(\pi^\star) - J(\pi_{\hat{f}}) - \gamma\Dav^{\pi^\star}(\hat{f}\|M^\star) + \gamma \Dav^{\pi^\star}(\hat{f}\|M^\star)  \tag{for any $\gamma$} \\
        &\leq \underbrace{J(\pi^\star) - J(\pi_{\hat{f}}) - \frac{\gamma}{2}\left(\E^{\pi^\star, M^\star}\big[\hat{f}(s,a) - \psi(\pi^\star; s) - \hat{f}(s)\big]\right)^2}_{\term_1} +  \order\left(\gamma \left(\frac{\iota}{n}  + \frac{\lambda^2\iota^2}{n^2} +\frac{1}{\lambda^2}\right)\right) \\
        &\qquad \underbrace{- \frac{\gamma\lambda^2}{8}\left(\E_{(s,a)\sim \calD}\left[\err(s,a;\hat{f})^2\right] \right)^2 + \gamma \Dav^{\pi^\star}(\hat{f}\|M^\star)}_{\term_2}. 
    \end{align*}
    $\term_1$ is bounded as before in regularized MDP
    \begin{align*}
        \term_1 
        &\leq \frac{1}{2\gamma} \gapcomp(\hat{f};\calF_\conf)^2 \tag{by AM-GM and the definition of $\gapcomp$}\\
        &\leq \frac{1}{2\gamma}\left(3\left(1+\frac{4H}{\alpha}\right)\left(1+H^3\cdot \frac{1}{\alpha}\right)\right)^2   \tag{by Equation \pref{eq:ER-regular}}\\
        &\lesssim \frac{H^8}{\gamma\alpha^4}. 
    \end{align*}
    Below, we bound $\term_2$. By definition, 
    \begin{align*}
       \term_2 
       &= - \frac{\gamma\lambda^2}{8}\left(\E_{(s,a)\sim \calD}\left[\err(s,a;\hat{f})^2\right] \right)^2 + \gamma \Dav^{\pi^\star}(\hat{f}\|M^\star) \\
       &= - \frac{\gamma\lambda^2}{8}\left(\E_{(s,a)\sim \calD}\left[(\hat{f}(s,a) - (\hatcalT \hat{f})(s,a))^2\right] \right)^2 + \gamma \left(\E_{(s,a)\sim d^{\pi^\star}}\left[\hat{f}(s,a) - (\calT \hat{f})(s,a)\right]\right)^2  \numberthis \label{eq: term2 label}
    \end{align*}
    By \pref{lem: hatT acc} and concentration inequality (\pref{lem:freedman}),   
    \begin{align*}
        &\E_{(s,a)\sim \calD}\left[(\hat{f}(s,a) - (\hatcalT \hat{f})(s,a))^2\right] \\
        &\geq \frac{1}{2}\E_{(s,a)\sim \calD}\left[(\hat{f}(s,a) - (\calT \hat{f})(s,a))^2\right] - \order\left( \frac{\iota}{n} \right) \tag{by $X^2\geq \frac{1}{2}Y^2 - (X-Y)^2$}\\
        &\geq \frac{1}{4}\E_{(s,a)\sim \mu}\left[(\hat{f}(s,a) - (\calT \hat{f})(s,a))^2\right] - \order\left( \frac{\iota}{n} \right). \numberthis \label{eq: temmpp1}
    \end{align*}
    By Jensen's inequality and coverage definition, 
\begin{align*}
    \left(\E_{(s,a)\sim d^{\pi^\star}}\left[\hat{f}(s,a) - (\calT \hat{f})(s,a)\right]\right)^2 \leq HC^{\pi^\star}\E_{(s,a)\sim \mu}\left[\left(\hat{f}(s,a) - (\calT \hat{f})(s,a)\right)^2\right].  \numberthis \label{eq: temmpp2}
\end{align*}
Using \pref{eq: temmpp1} and \pref{eq: temmpp2} in \pref{eq: term2 label} and denoting $Z = \E_{(s,a)\sim \mu} \big[\big(\hat{f}(s,a) - (\calT \hat{f})(s,a)\big)^2\big]$, we get 
\begin{align*}
   \term_2 &\leq 
        - \frac{\gamma\lambda^2}{8\cdot4^2\cdot 2}Z^2 + \order\left(\frac{\gamma \lambda^2 \iota^2}{n^2} \right) + \gamma HC^{\pi^\star} Z \tag{by $(X-Y)_+^2\geq \frac{1}{2}X^2 - Y^2$ for $X,Y>0$}\\
        &\leq \order\left( \frac{\gamma(HC^{\pi^\star})^2}{\lambda^2}  + \frac{\gamma \lambda^2 \iota^2}{n^2}\right). 
    \end{align*}
    Collecting terms, we get 
    \begin{align*}
        J(\pi^\star) - J(\pi_{\hat{f}})
        &\lesssim \order\left(\gamma \left(\frac{\iota}{n}  + \frac{\lambda^2\iota^2}{n^2} +\frac{1}{\lambda^2}\right) + \frac{H^8}{\gamma\alpha^4} + \frac{\gamma(HC^{\pi^\star})^2}{\lambda^2}\right) \\
        &= \order\left( \frac{H^4}{\alpha^2} \left(\sqrt{\frac{\iota}{n}}+\frac{\lambda \iota}{n} + \frac{HC^{\pi^\star}}{\lambda}\right)\right)   \tag{choosing optimal $\gamma$} \\
        &= \order\left( \frac{H^4}{\alpha^2} \left(\frac{\lambda \iota}{n} + \frac{HC^{\pi^\star}}{\lambda}\right)\right). \tag{AM-GM and $1\leq C^{\pi^\star}$}  
    \end{align*}
\end{proof}

\section{Concentration Inequalities and Performance Difference Lemma}
\begin{lemma}[Lemma A.3 in \cite{foster2021statistical}]
Let $(X_t)_{t \le T}$ be a sequence of random variables adapted to a filtration $(\mathcal{F}_t)_{t \le T}$. If $0 \le X_t \le R$ almost surely, then with probability at least $1 - \delta$,
\[
\sum_{t=1}^T X_t \le \frac{3}{2} \sum_{t=1}^T \mathbb{E}_{t-1}[X_t] + 4R \log(2\delta^{-1}),
\]
and
\[
\sum_{t=1}^T \mathbb{E}_{t-1}[X_t] \le 2 \sum_{t=1}^T X_t + 8R \log(2\delta^{-1}).
\] 
\label{lem:freedman}
\end{lemma}

\begin{lemma}[Performance Difference Lemma]
For any two policies $\pi_1, \pi_2$, we have
\begin{align*}
   V^{\pi_1}(s_1) - V^{\pi_2}(s_1) = \E_{s\sim d^{\pi_1}}\left[\sum_{a \in \calA}\left(\pi_1(a|s) - \pi_2(a|s)\right)Q^{\pi_2}(s,a)\right],
\end{align*}
and 
\begin{align*}
   V^{\pi_1}(s_1) - V^{\pi_2}(s_1) = \E_{s\sim d^{\pi_2}}\left[\sum_{a \in \calA}\left(\pi_1(a|s) - \pi_2(a|s)\right)Q^{\pi_1}(s,a)\right].
\end{align*}
\label{lem:pdl}
\end{lemma}

\section{Auxiliary Lemmas for Regularized MDPs and Proof of \pref{lem:psi-example}}
\label{app:aux-lemma}
\begin{lemma}
   Let $p, p'\in\Delta_\calA$ and $q\in\mathbb{R}^A$. Suppose that $\eta >0$ is such that $\eta q(a)\leq 1$ for all $a$. Then
    \begin{align*}
        \inner{p'-p,q} \leq \frac{1}{\eta} \KL(p',p) + \eta \sum_a  p(a)q(a)^2.
    \end{align*}
\label{lem:KL-stability}
\end{lemma}
\begin{proof}
    It suffices to show that
   \begin{align*}
       \max_{p'\in\Delta_A}\left\{\inner{p',q} - \frac{1}{\eta} \KL(p',p) \right\}\leq  \sum_a p(a)q(a) + \eta \sum_a p(a)q(a)^2.
   \end{align*}
   The left-hand side is the Fenchel conjugate of $\KL$ and has a closed form $ \frac{1}{\eta} \log\left( \sum_{a} p(a) \exp(\eta q(a))\right)$.
   Using the condition $\eta q(a)\leq 1$, we bound this by
   \begin{align*}
       &\frac{1}{\eta} \log\left( \sum_{a} p(a) \left(1 + \eta q(a) + \eta^2 q(a)^2\right)\right)   \tag{$e^x\leq 1+x+x^2$ for $x\leq 1$}\\
       &= \frac{1}{\eta} \log\left(1 + \eta \sum_a p(a)q(a) + \eta^2 \sum_a p(a)q(a)^2 \right) \\
       &\leq \sum_a p(a)q(a) + \eta \sum_a p(a)q(a)^2. \tag{$\log(1+x) \leq x$}
   \end{align*}
\end{proof}

\begin{lemma}
     Let $\Delta_{\calA}$ be the simplex over $\calA$, $\psi: \Delta_{\calA}\to \mathbb{R}$ be a convex function and let $q\in\mathbb{R}^{\calA}$. Define $G(p) = \inner{p,q} - \psi(p)$ for $p\in\Delta_{\calA}$.
   Let $p^\star = \argmax_{p\in\Delta_A} G(p)$.
   If $p^\star(a) > 0$ for all $a$, then for any $p\in\Delta_A$,   \begin{align*}
      G(p^\star) - G(p) = \breg_\psi(p,p^\star)
   \end{align*}
   where $\breg_\psi(x,y) = \psi(x) - \psi(y) - \inner{\nabla \psi(y), x-y}$ is the Bregman divergence w.r.t. $\psi$.
\label{lem:breg}
\end{lemma}
\begin{proof}
We maximize the concave objective $G(p)=\langle p,q\rangle-\psi(p)$ over the simplex $\Delta_A$ such that $\sum_{a\in \calA}p(a)=1$ and $p(a)\ge 0,\, \forall a \in \calA$. Introduce  $\lambda\in\mathbb{R}$ and $\mu=(\mu_a)_{a\in A}\in\mathbb{R}_+^A$ as Lagrange multipliers. The Lagrangian is
\begin{align*}
\mathcal{L}(p,\lambda,\mu)
= \langle p,q\rangle-\psi(p) + \lambda(1-\mathbf{1}^\top p) + \sum_{a\in A}\mu_a\,p(a),
\end{align*}
Since $G$ is concave and the feasible set is convex, the KKT conditions are necessary and sufficient for optimality.
In particular, at an optimal solution $(p^\star,\lambda^\star,\mu^\star)$ we have:
\begin{align}
&\nabla_p \mathcal{L}(p^\star,\lambda^\star,\mu^\star) \label{eq:stationary}
= q-\nabla\psi(p^\star) - \lambda^\star \mathbf{1} + \mu^\star = 0
\\& \mu_a^\star\,p^\star(a)=0, \quad \forall a \in \calA \label{eq:slack}
\\& p^\star \in \Delta_{\calA}, \mu^\star \ge 0
\end{align}
With the assumption $p^\star(a)>0$ for all $a$, and Equation \pref{eq:slack}, we have
\begin{align*}
    \mu_a^\star=0 \quad \forall a \in \calA
\end{align*}
Plugging into Equation \pref{eq:stationary} yields
\begin{equation}
q-\nabla\psi(p^\star)=\lambda^\star \mathbf{1}.
\label{eq:q_1}
\end{equation}
For any $p \in \Delta_{\calA}$, we have
\begin{align*}
G(p^\star)-G(p)
&= \langle p^\star,q\rangle-\psi(p^\star) - \left(\langle p,q\rangle-\psi(p)\right) \\
&= \langle p^\star-p,q\rangle - \psi(p^\star) + \psi(p)
\\&= \left(\psi(p)-\psi(p^\star)-\langle \nabla\psi(p^\star),\,p-p^\star\rangle\right)
   + \langle p^\star-p,\,q-\nabla\psi(p^\star)\rangle 
\\&= \breg_\psi(p,p^\star) + \langle p^\star-p,\,q-\nabla\psi(p^\star)\rangle
\\&= \breg_\psi(p,p^\star) + \langle p^\star-p,\,\lambda^\star \mathbf{1}\rangle \tag{Equation \pref{eq:q_1}}
\\&= \breg_\psi(p,p^\star) \tag{$p^\star,p\in\Delta_{\calA}$}
\end{align*}
\end{proof}

\begin{definition}[Legendre Functions (Section 26.4 of \cite{lattimore2020bandit})]
    Let $f$ be a convex function, and let $C = \ri(\dom f) \neq \emptyset$, where $\dom f$ is the domain of $f$ and $\ri$ is relative interior. We call $f$ \emph{Legendre} if
\begin{enumerate}[label=(\roman*)]
    \item $f$ is differentiable and strictly convex on $C$.
    \item $\|\nabla f(x_n)\|\to\infty$ for any sequence $(x_n)_n$ with $x_n\in C$ for all $n$ and $\lim_{n \rightarrow \infty}x_n\to x\in \partial C$, where $\partial C$ is the subgradient of $C$.
\end{enumerate}
\label{def:legendre}
\end{definition}

\begin{lemma}
For any $s$, $\psi(p;s) =  \alpha \breg_{\df}(p, \pi_{\rm ref}(\cdot|s))$ is a Legendre function of $p$ when $\df(x) =  \sum_{a \in \calA}x(a)\log(x(a))$ (negative entropy), or $\df(x) = -\sum_{a \in \calA}\log\left(x(a)\right)$ (log barrier), or $\df(x) =   \frac{\alpha}{1-q}\left(1-\sum_{a \in \calA} x(a)^q\right)$ (Tsallis entropy) with $q \in (0,1)$.
\label{lem:legendre-example}
\end{lemma}

\begin{proof}
When $\df$ is negative entropy, then on $\ri(\Delta_{\calA})$, for any $s$, $\left[\nabla\psi(p;s)\right]_a=1+\log p(a) - \log\pi_{\rm ref}(a|s)$ and
$\nabla^2\psi(p;s)=\mathrm{diag}(1/p(a))\succ 0$, so $\psi$ is differentiable and strictly convex on $\ri(\Delta_{\calA})$.
If $p_n\to\partial\Delta_{\calA}$, some coordinate $p_n(a_0)\downarrow 0$ and
$|1+\log p_n(a_0)|\to\infty$, hence $\|\nabla\psi(p_n)\|\to\infty$. Thus, $\psi$ is Legendre from \pref{def:legendre}.

When $\df$ is log barrier, then on $\ri(\Delta_{\calA})$, $[\nabla\psi(p;s)]_a= - \frac{1}{p(a)} + \frac{1}{\pi_{\rm ref}(a|s)}$ and
$\nabla^2\psi(p;s)=\mathrm{diag}(1/p(a)^2)\succ0$, so $\psi$ is differentiable and strictly convex.
If $p_n\to\partial\Delta_{\calA}$ then some $p_n(a_0)\downarrow 0$ and $|1/p_n(a_0)|\to\infty$,
hence $\|\nabla\psi(p_n)\|\to\infty$. Thus, $\psi$ is Legendre from \pref{def:legendre}.

When $\df$ is Tsallis entropy with $q \in (0,1)$, then on $\ri(\Delta_{\calA})$, we have
\[
[\nabla\psi(p;s)_a]=-\frac{\alpha q}{1-q}\,p(a)^{q-1} + \frac{\alpha q}{1-q}\,\pi_{\rm ref}(a)^{q-1},
\qquad
\nabla^2\psi(p;s)=\alpha q\,\mathrm{diag}\left(p(a)^{q-2}\right)\succ 0,
\]
so $\psi$ is differentiable and strictly convex.
If $p_n\to\partial\Delta_{\calA}$ then some $p_n(a_0)\downarrow 0$ and since $q-1<0$,
$|p_n(a_0)^{q-1}|\to\infty$, hence $\|\nabla\psi(p_n)\|\to\infty$.
Thus, $\psi$ is Legendre from \pref{def:legendre}.
\end{proof}

\begin{lemma}
Define $G(p) = \inner{p,q} - \psi(p)$ where $p,q \in \Delta_{\calA}$ and $\calA$ is an arbitrary finite set. Let
$p^\star = \argmax_{p\in\Delta_A} G(p)$. If $\psi$ is a Legendre mirror map, then $p^\star(a) > 0, \forall a \in \calA$.
\label{lem:legendre-unique}
\end{lemma}
\begin{proof}
Since $|\calA|<\infty$ and $\psi$ is Legendre on $C=\ri(\Delta_{\calA})$.
The maximizer of $G(p)=\langle p,q\rangle-\psi(p)$ over $\Delta_{\calA}$ is equivalently the minimizer of $
f(p)=\psi(p)-\langle p,q\rangle$ over $\Delta_{\calA}$. Since adding a linear function preserves differentiability, strict convexity on $C$, and the blow-up condition
$\|\nabla f(p_n)\|=\|\nabla\psi(p_n)-q\|\to\infty$ as $p_n\to \partial \Delta_{\calA}$, the function $f$ is also Legendre on $C$ from the definition in \pref{def:legendre}.
By Corollary 26.8 in \cite{lattimore2020bandit} (applied on the affine hull of $\Delta_{\calA}$ with $\mathrm{int}$ interpreted as $\ri$),
any minimizer $p^\star$ of $f$ over $\Delta_{\calA}$ must lie in $\ri(\Delta_{\calA})$.
Therefore $p^\star(a)>0$ for all $a\in\calA$.
\end{proof}

\begin{lemma}[Boundedness of $Q^\star$]
If $\psi(p;s) = \alpha \breg_{\df}(p, \pi_{\rm ref}(\cdot|s))$ for any convex function $\df$, then $0 \le Q^\star(s,a) \le H, \forall s,a$. Moreover, for any $f \in \calF$ and any $s,a$, we have $Q^\star(s,a) - Q^{\pi_f}(s,a) \le H^2$.


\label{lem:bound-of-Q*}
\end{lemma}

\begin{proof}
We have $\psi(p;s) \ge 0$ and the minimizer is achieved with $p = \pi_{\rm ref}(\cdot|s)$. We can show that for any $s \in \calS_h$,
\begin{align}
0\ \le\ V^\star(s)\ \le\ H-h+1,
\qquad
0\ \le\ Q^\star(s,a)\ \le\ H-h+1.
\label{eq:Q*-bounds}
\end{align}
We will also prove this by induction. For $s \in \calS_{H}$, we have $V^\star(s) = 0$ and $Q^\star(s,a)=R(s,a)\in[0,1]$.
If $0\le V^\star(s')\le H-h$ for $s' \in \calS_{h+1}$, then for any $s \in \calS_h$, we have 
\[
0\le  Q^\star(s,a)=R(s,a)+\E_{s' \sim P(\cdot|s,a)}[V^\star(s')]\le 1+(H-h)=H-h+1.
\]
Moreover, 
\[ V^\star(s)=\max_{p \in \Delta_{\calA}}\{\E_{a \sim p}[ Q^\star(s,a)]-\psi(p;s)\}
\le \max_{p \in \Delta_{\calA}} \E_{a \sim p}[Q^\star(s,a)] \le H-h+1,
\]
and $V^\star(s)\ge 0$ because
\begin{align*}
    V^\star(s) \ge \E_{a \sim \pi_{\rm ref}(\cdot|s)}[ Q^\star(s,a)] - \psi(\pi_{\rm ref}(\cdot|s);s) =  \E_{a \sim \pi_{\rm ref}(\cdot|s)}[Q^\star(s,a)] \ge 0
\end{align*}
For any $f \in \calF$, we have
\begin{align*}
\E_{a\sim\pi_f(\cdot|s)}[f(s,a)]-\psi(\pi_f(\cdot|s);s) \ge \E_{a\sim \pi_{\rm ref}(\cdot|s)}[f(s,a)]-\psi(\pi_{\rm ref}(\cdot|s);s) =\E_{a\sim \pi_{\rm ref}(\cdot|s)}[f(s,a)]
\end{align*}
Thus, for any $s$, we have  
\begin{align*}
    \psi(\pi_{f}(\cdot|s);s) \le \E_{a\sim\pi_f(\cdot|s)}[f(s,a)] - \E_{a\sim \pi_{\rm ref}(\cdot|s)}[f(s,a)] \le H
\end{align*}
where we use $f \in [0,H]$ from the bound of $Q^\star$ in \pref{eq:Q*-bounds}. $f \in \calF$ that is not in $[0,H]$ can be removed because it is not $Q^\star$. We can show that
\begin{align}
V^{\pi_f}(s)\ \ge\ -H(H-h+1),
\qquad
Q^{\pi_f}(s,a)\ \ge\ -H(H-h).
\label{eq:bar-lower-pif}
\end{align}
For $h=H$, since $Q^{\pi_f}(s,a)=R(s,a)\ge 0$, the second inequality holds.
Moreover,
\[
V^{\pi_f}(s)=\E_{a\sim\pi_f(\cdot|s)}[R(s,a)]-\psi(\pi_f(\cdot|s);s)\ge 0-H=-(H-H+1)H,
\]
proving the base case. Assume \eqref{eq:bar-lower-pif} holds for stage $h+1$. Then for any $s'\in\calS_{h+1}$,
\[
V^{\pi_f}(s')\ge -(H-(h+1)+1)H=-(H-h)H.
\]
Hence for any $s\in\calS_h,a\in\calA$,
\[
Q^{\pi_f}(s,a)
=R(s,a)+\E_{s' \sim P(\cdot|s,a)}[ V^{\pi_f}(s')]
\ge 0-(H-h)H,
\]
which proves the second inequality at stage $h$.
Finally,
\[
V^{\pi_f}(s)
=\E_{a\sim\pi_f(\cdot|s)}[Q^{\pi_f}(s,a)]-\psi(\pi_f(\cdot|s);s)
\ge -(H-h)H - H = -(H-h+1)H,
\]
proving the first inequality at stage $h$. This completes the induction. Thus,  for any $s \in \calS_h$ and any $a \in \calA$,  using the upper bound in \pref{eq:Q*-bounds} and lower bound in \pref{eq:bar-lower-pif}, we have
\begin{align*}
    Q^\star(s,a) - Q^{\pi_f}(s,a) \le H - h + 1 + H(H-h) \le H^2
\end{align*}

\end{proof}

\begin{lemma}[Solution for Log-barrier Regularizer]
\label{lem:lb-ratio}
For any $\alpha >0$ and any $f \in \calF$, if
\begin{equation*}
\label{eq:lb-greedy-log-barrier}
\pi_f(\cdot
|s) \in \argmax_{p\in\Delta_{\calA}}
\Big\{\sum_{a \in \calA} p(a) f(s,a) - \alpha\breg_{\df}(p, \pi_{\rm ref}(\cdot|s))\Big\},
\end{equation*}
where $\df(p) = -\sum_{a\in\calA}\log p(a)$. Then we have
\begin{align*}
    \frac{\alpha}{\alpha + 2H}\ \le\ \frac{\pi_{f_1}(a|s)}{\pi_{f_2}(a|s)}\ \le\ \frac{\alpha + 2H}{\alpha}, \quad \forall f_1, f_2 \in \calF.
\end{align*}

\end{lemma}

\begin{proof}
From \pref{lem:legendre-example} and \pref{lem:legendre-unique}, we known $\pi_f(\cdot|s)$ is unique for any $s$ with $\pi_f(a|s)>0$ for all $a$. since the maximizer is interior, only the equality constraint $\sum_a p(a)=1$ is active. Consider the Lagrangian
\[
\mathcal{L}_{f,s}(p,\lambda)=\sum_a p(a)f(s,a) - \alpha \breg_{\df}(p, \pi_{\rm ref}(\cdot|s)) + \lambda_f(s)\left(\sum_a p(a)-1\right).
\]
By KKT condition, the stationarity condition at $p=\pi_f(\cdot|s)$ gives, for every $a$,
\[
\frac{\partial \mathcal{L}_{f,s}}{\partial p(a)}
=f(s,a) + \frac{\alpha}{\pi_f(a|s)} - \frac{\alpha}{\pi_{\rm ref}(a|s)} + \lambda_f(s)=0,
\]
hence
\[
\pi_f(a|s)=\frac{\pi_{\rm ref}(a|s)}{1 - \frac{\pi_{\rm  ref}(a|s)}{\alpha}\left(\lambda_{f}(s) + f(s,a)\right)}.
\]
We have $\pi_f(a|s)>0$ and normalization implies
\begin{align}
    \sum_{a \in \calA}\pi_f(a|s) = \sum_{a \in \calA}\frac{\pi_{\rm ref}(a|s)}{1 - \frac{\pi_{\rm  ref}(a|s)}{\alpha}\left(\lambda_{f}(s) + f(s,a)\right)} = 1
\label{eq:log-normal}
\end{align}

Define function $g_{f}(\lambda;s,a) = 1 - \frac{\pi_{\rm  ref}(a|s)}{\alpha}\left(\lambda + f(s,a)\right)$.  For any $f \in \calF$ and any $s,a$, we have $g_f(0;s,a) =  1 - \frac{\pi_{\rm  ref}(a|s)}{\alpha}\left(f(s,a)\right)  \le 1$ and $g_f(-H;s,a) =  1 - \frac{\pi_{\rm  ref}(a|s)}{\alpha}\left(-H + f(s,a)\right) \ge 1$
where we use $f \in [0,H]$. Since $g_{f}(\lambda; s,a)$ is decreasing for $\lambda$ and $\sum_{a \in \calA}\pi_{\rm ref}(a|s)=1$, from \pref{eq:log-normal}, we have $\lambda_f(s) \in [-H, 0]$. Thus, $\lambda_f(s) + f(s,a) \in [-H, H]$, and we have
\begin{align*}
    \frac{\pi_{f_1}(a|s)}{\pi_{f_2}(a|s)} &= \pi_{f_1}(a|s) \left(\frac{1}{\pi_{f_2}(a|s)} - \frac{1}{\pi_{f_1}(a|s)}\right) + 1 
    \\&=  \frac{\pi_{f_1}(a|s)}{\alpha} \left(\lambda_{f_1}(s) + f_1(s,a) -\left(\lambda_{f_2}(s) + f_2(s,a)\right)\right) + 1 
    \\&\le 1 + \frac{2H}{\alpha}.
\end{align*}

\end{proof}

\begin{lemma}[Solution for Tsallis-entropy Regularizer]
\label{lem:tsallis-ratio}
For any $\alpha >0$ and any $f \in \calF$, if
\begin{equation*}
\label{eq:lb-greedy-tsallis}
\pi_f(\cdot
|s) \in \argmax_{p\in\Delta_{\calA}}
\Big\{\sum_{a \in \calA} p(a) f(s,a) - \alpha\breg_{\df}(p, \pi_{\rm ref}(\cdot|s))\Big\},
\end{equation*}
where $\df(p) = \frac{1}{1-q}\left(1-\sum_{a\in\calA}p(a)^q\right)$. Then we have
\begin{align*}
     \left(1 + \frac{2H(1-q)}{\alpha q}\right)^{\frac{1}{q-1}}\ \le\ \frac{\pi_{f_1}(a|s)}{\pi_{f_2}(a|s)}\ \le \left(1 + \frac{2H(1-q)}{\alpha q}\right)^{\frac{1}{1-q}}, \quad \forall f_1, f_2 \in \calF.
\end{align*}

\end{lemma}

\begin{proof}
From \pref{lem:legendre-example} and \pref{lem:legendre-unique}, we known $\pi_f(\cdot|s)$ is unique for any $s$ with $\pi_f(a|s)>0$ for all $a$. Since the maximizer is interior, only the equality constraint $\sum_{a \in \calA} p(a)=1$ is active. The Lagrangian is
\[
\mathcal{L}_{f,s}(p,\lambda)
=\sum_{a\in\calA}p(a)f(s,a) - \alpha \text{Breg}_{\df}(p, \pi_{\rm ref}(\cdot|s)) + \lambda_f(s)\left(\sum_{a\in\calA}p(a)-1\right).
\]
By KKT condition, the stationarity condition at $p=\pi_f(\cdot|s)$, for every $a$,
\[
\frac{\partial \mathcal{L}_{f,s}}{\partial p(a)}
=f(s,a) + \frac{\alpha}{1-q}\cdot q\,\pi_f(a|s)^{q-1} - \frac{\alpha}{1-q}\cdot q\,\pi_{\rm ref}(a|s)^{q-1} + \lambda_f(s)=0.
\]
This implies
\[
\pi_f(a|s) = \pi_{\rm ref}(a|s)\left(1 - \frac{1-q}{\alpha q}\cdot\frac{\lambda_f(s) + f(s,a)}{\pi_{\rm ref}(a|s)^{q-1}}\right)^{\frac{1}{q-1}}
\]
We have $\pi_f(a|s) > 0$ and normalization ensures
\begin{align}
    \sum_{a \in \calA}\pi_f(a|s) = \sum_{a \in \calA} \pi_{\rm ref}(a|s)\left(1 - \frac{1-q}{\alpha q}\cdot\frac{\lambda_f(s) + f(s,a)}{\pi_{\rm ref}(a|s)^{q-1}}\right)^{\frac{1}{q-1}} = 1
\label{eq:tsallis-normal}
\end{align}
Define function $g_f(\lambda; s,a) = \left(1 - \frac{1-q}{\alpha q}\frac{\lambda + f(s,a)}{\pi_{\rm ref}(a|s)^{q-1}}\right)^{\frac{1}{q-1}}$.  We have  $g_f(0; s,a) \ge 1$ and $g_f(-H; s,a) \le 1$ where we use $f \in [0,H]$, $q \in (0,1)$. Since $g_f(\lambda; s,a)$ is increasing in $\lambda$, from $\sum_{a \in \calA}\pi_{\rm ref}(a|s) = 1$ and \pref{eq:tsallis-normal}, we have $\lambda_{f}(s) \in [-H, 0]$. Thus, $\lambda_{f}(s) + f(s,a) \in [-H, H]$. Thus
\begin{align*}
    \frac{\pi_{f_1}(a|s)}{\pi_{f_2}(a|s)} &= \left(\frac{1}{\pi_{f_2}(a|s)^{q-1}}\left(\pi_{f_1}(a|s)^{q-1} - \pi_{f_2}(a|s)^{q-1}\right) + 1\right)^{\frac{1}{q-1}}
    \\&= \left(\frac{1}{\pi_{f_2}(a|s)^{q-1}}\cdot \frac{1-q}{\alpha q} \cdot \left(f_2(s,a) + \lambda_{f_2}(s,a) - \left(f_1(s,a) + \lambda_{f_1}(s,a)\right)\right) + 1\right)^{\frac{1}{q-1}}
    \\&\ge \left(1 + \frac{2H(1-q)}{\alpha q}\right)^{\frac{1}{q-1}}.
\end{align*}
\end{proof}


\begin{lemma}\label{lem:psi-example}
Let $\psi(p;s) = \alpha \breg_{\df}(p, \pi_{\rm ref}(\cdot|s))$ for some reference policy $\pi_\reff$. The following choices of $\df$ make \pref{assum:psi_proper} hold with respective $(C^\psi_1, C^\psi_2)$: 
\begin{itemize}[leftmargin=11pt, topsep=0pt]
   \setlength{\itemsep}{0pt}
  \setlength{\parskip}{0pt}
  \setlength{\parsep}{0pt}
  \item $\df(p) = \sum_{a \in \calA} p(a)\log(p(a))$ \emph{(Shannon entropy):} $(C^\psi_1, C^\psi_2)=(1 + \frac{4H}{\alpha}, \frac{1}{\alpha})$. 
  \item $\df(p) = \frac{1}{1-q}\left(1-\sum_{a} p(a)^q\right)$ with $q \in (0,1)$ \emph{(Tsallis entropy):} $(C^{\psi}_1, C^\psi_2) = \big(\big( \scalebox{0.9}{$1 + \frac{2H(1-q)}{\alpha q}\big)^{\frac{2-q}{1-q}}, \frac{1}{\alpha q}$} \big)$. 
  \item $\df(p) = -\sum_{a \in \calA} \log\left(p(a)\right)$ \emph{(log-barrier):} $(C^{\psi}_1, C^\psi_2) = (1+\frac{2H}{\alpha}, \frac{2}{\alpha})$. 
\end{itemize}
\end{lemma}

\begin{proof}[Proof of \pref{lem:psi-example}]
From \pref{lem:legendre-example}, $\psi(p;s) = \alpha \breg_{\df}(p, \pi_{\rm ref}(\cdot|s))$ is Legendre when $\df$ is Shannon entropy, log-barrier or Tsallis entropy with $q \in (0,1)$.  We will prove the concrete
$C_\psi^1$ and $C_\psi^2$ factors for these three regularizers now.

\paragraph{When $\df$ is (negative) Shannon entropy. }
Since $\df$ is entropy, $\psi(p;s) = \alpha \KL\left(p, \pi_{\rm ref}(\cdot|s)\right)$. For any $x,y \in \Delta_{\calA}$ and $s \in \calS$, $\breg_{\psi}\left(x, y;s\right) = \alpha\KL\left(x\|y\right)$ and define $r(a)=\frac{y(a)}{x(a)}$ for all $a\in\calA$. Define the nonnegative convex functions $
\phi(t) = t\log t-t+1, \xi(t) = t-1-\log t$ for $t>0$.  We have $\KL(y\|x)= \sum_a y(a)\log\frac{y(a)}{x(a)}=\E_{a\sim x}\big[r(a)\log r(a)\big] = \E_{x}\big[\phi(r)\big]$ and 
$\KL(x\|y)= \E_x[-\log r]=\E_x[-\log r+r-1]=\E_x[\xi(r)]$.   From \pref{lem:KL-function-mono}, we have
\begin{equation}
\label{eq:phi-psi-pointwise}
\phi(t)\le \xi(t)\quad \text{for }t\in(0,1],
\qquad
\phi(t)\le \bigl(1+\log t\bigr)\xi(t)\quad \text{for }t\ge 1.
\end{equation}
Assume there exists $R\ge 1$ such that $R^{-1}\le r(a)\le R, \forall a\in\calA$. , we have
\[ \frac{\phi(r(a))}{\xi(r(a))}
\le
\sup_{t\in[R^{-1},R]}\frac{\phi(t)}{\xi(t)}
\le 1+\log R
\quad\text{by }\pref{eq:phi-psi-pointwise}.
\]
Multiplying by $\xi(r(a))\ge 0$ and taking $\E_{a\sim x}$ gives
\begin{align}
\KL(y\|x)=\E_x[\phi(r)]\le (1+\log R)\E_x[\xi(r)]=(1+\log R)\KL(x\|y).
\label{eq:KL-trans}
\end{align}
Now we calculate the exact value of $R$. For any $f_1, f_2 \in \calF$, we have
\[
\pi_{f_1}(a|s) = \frac{\pi_{\rm ref}(a|s)e^{\frac{1}{\alpha}f_1(s,a)}}{\sum_{a' \in \calA} \pi_{\rm ref}(a'|s)e^{\frac{1}{\alpha}f_1(s, a')}},
\qquad
\pi_{f_2}(a|s) = \frac{\pi_{\rm ref}(a|s)e^{\frac{1}{\alpha}f_2(s,a)}}{\sum_{a' \in \calA} \pi_{\rm ref}(a'|s)e^{\frac{1}{\alpha}f_2(s, a')}},
\]
with $|f_1(s, a)|\le H$ and $|f_2(s, a)|\le H$ for all $s, a$ from \pref{lem:bound-of-Q*}. Let $Z_1 =\sum_{a' \in \calA} \pi_{\rm ref}(a'|s)e^{\frac{1}{\alpha}f_1(s, a')}$ and $Z_2 = \sum_{a' \in \calA} \pi_{\rm ref}(a'|s)e^{\frac{1}{\alpha}f_2(s, a')}$.
We have $e^{-\frac{H}{\alpha}}\le Z_1\le e^{\frac{H}{\alpha}}$, and similarly
$e^{-\frac{H}{\alpha}}\le Z_2\le e^{\frac{H}{\alpha}}$. Hence $e^{-\frac{2H}{\alpha}}\le Z_1/Z_2\le e^{\frac{2H}{\alpha}}$.
Moreover $|f_1(s,a)-f_2(s,a)|\le 2H$, so $e^{-\frac{2H}{\alpha}}\le e^{\frac{1}{\alpha}\left(f_1(s,a)-f_2(s,a)\right)}\le e^{\frac{2H}{\alpha}}$.
This implies
\[ \frac{\pi_{f_1}(a|s)}{\pi_{f_2}(a|s)}=\frac{Z_{f_2}}{Z_{f_1}}\,e^{\frac{1}{\alpha}\left(f_1(s,a)-f_2(s,a)\right)}\in[e^{\frac{-4H}{\alpha}},e^{\frac{4H}{\alpha}}],
\]
for all $s, a$. Thus, $R = e^{\frac{4H}{\alpha}}$ and from \pref{eq:KL-trans}, we have
\begin{align*}
\breg_{\psi}(\pi_{f_1}, \pi_{f_2}; s) &= \alpha\KL\left(\pi_{f_1}(\cdot|s)\| \pi_{f_2}(\cdot|s)\right)
\\ &\ge \frac{\alpha}{1 + \log(e^{4H/\alpha})}\KL\left(\pi_{f_2}(\cdot|s)\| \pi_{f_1}(\cdot|s)\right)
\\&= \frac{\alpha}{\alpha +4H}\breg_{\psi}(\pi_{f_2}, \pi_{f_1}; s). 
\end{align*}
Thus, we have $C^{\psi}_1 = 1 + \frac{4H}{\alpha}$ and $C^{\psi}_2= \frac{1}{\alpha}$ because $\breg_{\psi}\left(x, y;s\right) = \alpha\KL\left(x\|y\right)$ for any $s$.

\paragraph{When $\df$ is log-barrier. }  We first compute the value of $C_{\psi}^1$. To start, define $$g(t) = \frac{t - 1 - \log(t)}{\frac{1}{t} - 1 + \log(t)}. $$ 
\pref{lem:log-barrier-mono} shows that $g(t)$ is increasing on $t > 0$. Thus, for any $R > 1$ the minimum of $g(t)$ is attained at $t=1/R$ on $[1/R,R]$,  giving
\begin{align}
    g(t) \ge \min_{t \in [R^{-1}, R]} g(t) =  \frac{\log R + \frac{1}{R}-1}{R-1-\log R} > 0, \,\, \forall t \in [R^{-1}, R].
\label{eq:g-log}
\end{align}
For any $s$ and any $f_1, f_2 \in \calF$, let $r(a|s) = \frac{\pi_{f_1}(a|s)}{\pi_{f_2}(a|s)}$, from \pref{lem:lb-ratio}, we have $r(a|s) \in [R_{\text{log}}^{-1}, R_{\text{log}}]$ where $R_{\text{log}} = \frac{\alpha + 2H}{\alpha}$. When $\psi(p;s) = \alpha \breg_{\df}(p, \pi_{\rm ref}(\cdot|s))$ and $\df$ is log-barrier, we have
\begin{align*}
    \breg_{\psi}\left(\pi_{f_1}, \pi_{f_2}; s\right) &= \alpha \sum_{a \in \calA}\left(r(a|s) - 1 - \log\left(r(a|s)\right)\right) 
    \\&= \alpha \sum_{a \in \calA}g(r(a|s))\left(\frac{1}{r(a|s)} - 1 + \log\left(r(a|s)\right)\right)
    \\&\ge c_{\text{log}}\alpha \sum_{a \in \calA}\left(\frac{1}{r(a|s)} - 1 + \log\left(r(a|s)\right)\right)
    \\&= c_{\text{log}}  \breg_{\psi}\left(\pi_{f_2}, \pi_{f_1}; s\right)
\end{align*}
where by \pref{eq:g-log} we define $c_{\text{log}} =  \frac{\log\left(R_{\text{log}}\right) + \frac{1}{R_{\text{log}}}-1}{R_{\text{log}}-1-\log\left(R_{\text{log}}\right)} \ge \frac{1}{R_{\text{log}}}$. Thus, $C^{\psi}_1 = \frac{\alpha + 2H}{\alpha}$.

From \pref{lem:breg-log-kl}, we have $C_2^{\psi} = \frac{2}{\alpha}$.

\paragraph{When $\df$ is (negative) Tsallis entropy with $q\in(0,1)$.} We first compute the value of $C_{\psi}^1$.
Define the scalar function (for $t>0$)
\[
\varphi_q(t)=\frac{1+q(t-1)-t^q}{1-q}\ \ge\ 0,
\qquad
g(t)=\frac{\varphi_q(t)}{t^q\,\varphi_q(1/t)}.
\]
From \pref{lem:tsallis-mono}, we have
\[
g(t)
\ \ge\
\begin{cases}
t^{q-2}, & t\ge 1,\\
t^{2-q}, & t\le 1.
\end{cases}
\]
Therefore, for any $R>1$ and any $t\in[R^{-1},R]$,
\begin{align}
    g(t)\ \ge\ \min_{t\in[R^{-1},R]} g(t)\ \ge\ R^{-(2-q)}\;>\;0, \qquad \forall t\in[R^{-1},R].
\label{eq:g-tsallis}
\end{align}
For any $s$ and any $f_1,f_2\in\calF$, let $r(a|s)=\frac{\pi_{f_1}(a|s)}{\pi_{f_2}(a|s)}$.
From  \pref{lem:tsallis-ratio},
we have $r(a|s)\in[R_{\text{tsa}}^{-1},R_{\text{tsa}}]$, where
$R_{\text{tsa}}  = \left(1 + \frac{2H(1-q)}{\alpha q}\right)^{\frac{1}{1-q}}$. When $\psi(p;s) = \alpha \breg_{\df}(p, \pi_{\rm ref}(\cdot|s))$ and $\df$ is Tsallis entropy, we have
\[
\breg_{\psi}(p_1, p_2; s)
=\alpha\sum_{a\in\calA} p_2(a)^q\,\varphi_q\!\left(\frac{p_1(a)}{p_2(a)}\right),\qquad p_1, p_2\in\ri(\Delta_{\calA}), \forall s \in \calS
\]
Hence,
\begin{align*}
\breg_{\psi}\!\left(\pi_{f_1}, \pi_{f_2}; s\right)
&=\alpha\sum_{a\in\calA}\pi_{f_2}(a|s)^q\,\varphi_q\!\left(r(a|s)\right)\\
&\ge \alpha\,c_{\text{tsa}}\sum_{a\in\calA}\pi_{f_2}(a|s)^q\,r(a|s)^q\,\varphi_q\!\left(1/r(a|s)\right)\\
&=c_{\text{tsa}}\ \breg_{\psi}\!\left(\pi_{f_2}, \pi_{f_1} ; s\right),
\end{align*}
where we used \pref{eq:g-tsallis} with $R=R_{\text{tsa}}$ and $c_{\text{tsa}}:=R_{\text{tsa}}^{-(2-q)}$.
Therefore, $C^{\psi}_1 = \left(1 + \frac{2H(1-q)}{\alpha q}\right)^{\frac{2-q}{1-q}}$.

Now we calculate $C_{\psi}^2$. We first define $\psi_{\mathrm{ent}}(p) = \sum_{a \in \calA} p(a)\log(p(a))$, we have $\Breg_{\psi_{\mathrm{ent}}}(p_1, p_2) = \KL\left(p_1\|p_2\right)$ for any $p_1, p_2 \in \Delta_{\calA}$. For $a\in\ri(\Delta_{\calA})$, we have the Hessians
\[
\nabla^2\psi(p)=\alpha q\,\mathrm{diag}\big(p(a)^{q-2}\big),
\qquad
\nabla^2\psi_{\mathrm{ent}}(p)=\mathrm{diag}\big(1/p(a)\big).
\]
Since $q\in(0,1)$ and $p(a)\in(0,1]$, we have $2-q>1$ and thus $p(a)^{q-2} \ge\ \frac{1}{p(a)}, \forall a\in\calA$.
Therefore, $\nabla^2\psi(p)\ \succeq\ \alpha q\;\nabla^2\psi_{\mathrm{ent}}(p), \forall p\in\ri(\Delta_{\calA})$. Now use the integral representation of Bregman divergence for twice-differentiable convex functions:
for any $p_1, p_2\in\ri(\Delta_{\calA})$,
\begin{equation}
\label{eq:breg-integral}
\Breg_{\psi}(p_1\|p_2)
=\int_0^1 (1-t)\,(p_1-p_2)^\top \nabla^2\psi\big(p_2+t(p_1-p_2)\big)\,(p_1-p_2)\,dt,
\end{equation}
and the same identity holds for $\psi_{\mathrm{ent}}$.
Since $\ri(\Delta_{\calA})$ is convex, the segment $p_1+t(p_1-p_2)\in\ri(\Delta_{\calA})$ for all $t\in[0,1]$. Thus, 
\begin{align*}
\Breg_{\psi}(p_1\|p_2)
\ &\ge\ 
\alpha q\int_0^1 (1-t)\,(p_1-p_2)^\top \nabla^2\psi_{\mathrm{ent}}\big(p_2+t(p_1-p_2)\big)\,(p_1-p_2)\,dt
\\&=\alpha q\;\Breg_{\psi_{\mathrm{ent}}}(p_1\|p_2) = \alpha q \KL\left(p_1\|p_2\right).
\end{align*}
This implies $C^{\psi}_2= \frac{1}{\alpha q}$.
\end{proof}

\begin{lemma}
Define $
\phi(t) = t\log t-t+1$, $\xi(t) = t-1-\log t$ for $t>0$, we have
\begin{equation*}
\phi(t)\le \xi(t)\quad \text{for } t\in(0,1],
\qquad
\phi(t)\le \bigl(1+\log t\bigr)\xi(t)\quad \text{for }t\ge 1.
\end{equation*}
\label{lem:KL-function-mono}
\end{lemma}

\begin{proof}
$\phi(t)$ and $\xi(t)$ are both nonnegative and convex functions. For $t\ge 1$, define $g(t):=(1+\log t)\xi(t)-\phi(t)=2t-(\log t)^2-2\log t-2$.
Let $u=\log t\ge 0$, so $t=e^u$ and
$g(u)=2e^u-u^2-2u-2$.
Then $g''(u)=2e^u-2\ge 0$ for $u\ge 0$, hence $g'(u)$ is increasing.
Moreover $g'(0)=2-0-2=0$, so $g'(u)\ge 0$ for $u\ge 0$, implying $g(u)\ge g(0)=0$.
Thus $\phi(t)\le (1+\log t)\xi(t)$ for $t\ge 1$. For $t\in(0,1]$, define
$h(t):=\xi(t)-\phi(t)=2t-(t+1)\log t-2$.
Let $u=-\log t\ge 0$, so $t=e^{-u}$ and
$h(u)=2e^{-u}+(e^{-u}+1)u-2$.
Then $h''(u)=u e^{-u}\ge 0$, hence $h'(u)$ is increasing, and $h'(0)=0$,
so $h(u)\ge h(0)=0$, proving $\phi(t)\le \xi(t)$ for $t\in(0,1]$.
This establishes the results.
\end{proof}

\begin{lemma}
Function $g(t) = \frac{t - 1 - \log(t)}{\frac{1}{t} - 1 + \log(t)}$ is increasing on $t > 0$.
\label{lem:log-barrier-mono}
\end{lemma}

\begin{proof}
We write $t=e^{u}$ with $u\in \mathbb{R}$, and
\[
A(u):=\varphi(e^{u})=e^{u}-1-u,\qquad
B(u):=\varphi(e^{-u})=e^{-u}-1+u.
\]
Then $g(e^{u})=A(u)/B(u)$ and note $B(u)=A(-u)>0$ for $u\neq 0$.
Differentiate with respect to $u$:
\[
\frac{d}{du}\left(\frac{A(u)}{B(u)}\right)=\frac{A'(u)B(u)-A(u)B'(u)}{B(u)^2}.
\]
Define 
\[
N(u) = A'(u)B(u)-A(u)B'(u)=(u-2)e^{u}-(u+2)e^{-u}+4.
\]
We now show $N(u)\ge 0$ for all $u\ge 0$.
Indeed,
\[
N(0)=0,\qquad
N'(u)=(u-1)e^{u}+(u+1)e^{-u},\qquad
N''(u)=u(e^{u}+e^{-u})\ge 0 \ \ \forall u\ge 0.
\]
Thus $N'$ is increasing on $[0,\infty)$ and $N'(0)=0$, so $N'(u)\ge 0$ for all $u\ge 0$.
Therefore $N$ is increasing on $[0,\infty)$ and $N(u)\ge N(0)=0$ for all $u\ge 0$.
This implies $\frac{d}{du}\left(\frac{A(u)}{B(u)}\right)\ge 0$ for $u\ge 0$. Thus,
$g(t)$ is increasing for $t=e^{u}\ge 1$.
By the identity $g(1/t)=1/g(t)$, it follows that $g$ is increasing on $(0,1]$.
\end{proof}

\begin{lemma}
Define the scalar function (for $t>0$)
\[
\varphi_q(t)=\frac{1+q(t-1)-t^q}{1-q}\ \ge\ 0,
\qquad
g(t)=\frac{\varphi_q(t)}{t^q\,\varphi_q(1/t)}.
\]
We have 
\[g(t)
\ \ge\
\begin{cases}
t^{q-2}, & t\ge 1,\\
t^{2-q}, & t\le 1.
\end{cases}\]
\label{lem:tsallis-mono}
\end{lemma}

\begin{proof}
We first lower bound $g(t)$ on $[R^{-1},R]$.
By Taylor's theorem applied to $t^q$ around $t=1$, there exists $\xi$ between $1$ and $t$ such that
\[
t^q = 1+q(t-1) + \frac{q(q-1)}{2}\,\xi^{q-2}(t-1)^2
\quad\Longrightarrow\quad
\varphi_q(t)=\frac{q}{2}\,\xi^{q-2}(t-1)^2 .
\]
Since $q-2<0$, $x^{q-2}$ is decreasing; hence for any $t>0$,
\[
\frac{q}{2}\max\{t,1\}^{\,q-2}(t-1)^2
\ \le\ \varphi_q(t)\ \le\
\frac{q}{2}\min\{t,1\}^{\,q-2}(t-1)^2 .
\]
Using these two-sided bounds for $\varphi_q(t)$ and $\varphi_q(1/t)$, one obtains
\[
g(t)=\frac{\varphi_q(t)}{t^q\varphi_q(1/t)}
\ \ge\
\begin{cases}
t^{q-2}, & t\ge 1,\\
t^{2-q}, & t\le 1.
\end{cases}
\]
\end{proof}

\begin{lemma}
If $\df(p) = -\sum_{a \in \calA}\log(p(a))$, then $\breg_{\df}\left(p_1, p_2\right) \ge \frac{1}{2}\KL\left(p_1\|p_2\right)$.
\label{lem:breg-log-kl}
\end{lemma}
\begin{proof}
 Since $t-1-\log t=\int_1^t\frac{u-1}{u}\,du$,
if $t\ge 1$, then $\frac{u-1}{u}\ge \frac{u-1}{t}$ on $u\in[1,t]$, hence
$t-1-\log t\ge \int_1^t \frac{u-1}{t}\,du=\frac{(t-1)^2}{2t}$.
If $t\le 1$, then $t-1-\log t=\int_t^1(\frac1u-1)\,du\ge \int_t^1(1-u)\,du=\frac{(1-t)^2}{2}$. Thus, for any $t >0$, $t-1-\log t \ \ge\ \frac{(t-1)^2}{2\max\{t,1\}}$. Set $t = \frac{p_1(a)}{p_2(a)}$, we have
\begin{align*}
    \breg_{\df}\left(p_1, p_2\right) &= \sum_{a \in \calA}\left(\frac{p_1(a)}{p_2(a)} - 1 - \log\left(\frac{p_1(a)}{p_2(a)} \right)\right) 
    \\&\ge \frac 1 2\sum_a \frac{(p_1(a)-p_2(a))^2}{p_2(a)\max\{p_1(a),p_2(a)\}}
\\&\ge \frac 1 2\sum_a \frac{(p_1(a)-p_2(a))^2}{p_2(a)}=\frac 1 2 \chi^2(p_1\|p_2) \ge \frac 1 2\KL(p_1\|p_2).
\end{align*}
\end{proof}

\clearpage
\section{Limitations and Future Work}\label{app:limitation}
Our work provides a new unified theory for offline RL under $Q^\star$-realizability and partial coverage. Our complexity measure currently serves mainly to identify new learnable regimes; we do not yet provide a matching lower bound showing that it gives a tight characterization. Another limitation is that exactly solving our minimax algorithm is computationally difficult, though it may be approximated through an actor-critic algorithm similar to \cite{cheng2022adversarially}.

In addition, we have been focusing on $Q^\star$-approximation and average Bellman error as the discrepancy measure; however, the framework may be generalized to different function approximation schemes and estimation methods, which could unify offline RL methods further.


\end{document}